\documentclass{article} % For LaTeX2e
\usepackage{iclr2026_conference,times}

\def\rmH{{\mathbf{H}}}

\def\rmP{{\mathbf{P}}}

\def\rmT{{\mathbf{T}}}

\def\rmX{{\mathbf{X}}}

\def\gD{{\mathcal{D}}}

\def\gF{{\mathcal{F}}}

\def\gN{{\mathcal{N}}}

\def\gT{{\mathcal{T}}}

\def\sP{{\mathbb{P}}}

\def\bA{{\mathbf A}}

\def\bH{{\mathbf H}}

\def\bK{{\mathbf K}}

\def\bQ{{\mathbf Q}}

\def\bV{{\mathbf V}}
\def\bW{{\mathbf W}}

\def\bsigma{{\boldsymbol \sigma}}
\def\btheta{{\boldsymbol \theta}}

\def\bmu{{\boldsymbol \mu}}

\def\be{{\mathbf e}}

\def\bh{{\mathbf h}}

\def\bp{{\mathbf p}}
\def\bt{{\mathbf t}}

\def\bw{{\mathbf w}}

\usepackage[utf8]{inputenc} % allow utf-8 input
\usepackage[T1]{fontenc}    % use 8-bit T1 fonts    
\usepackage{booktabs}       % professional-quality tables
\usepackage{amsfonts}       % blackboard math symbols
\usepackage{nicefrac}       % compact symbols for 1/2, etc.
\usepackage{microtype}      % microtypography
\usepackage{xcolor}         % colors
\definecolor{myfandango}{rgb}{0.71, 0.2, 0.54} % Beautiful Fandango
\definecolor{myroyalblue}{rgb}{0.0, 0.14, 0.4}  % Deep Blue for footnotes
\definecolor{mydarkblue}{rgb}{0.0, 0.0, 0.2}    % Professional Dark for URLs
\usepackage{hyperref}
\hypersetup{
    colorlinks=true,
    linkcolor=myroyalblue,  % Color for internal links (footnotes, sections)
    citecolor=myfandango,   % Color for citations (references)
    urlcolor=mydarkblue     % Color for external hyperlinks
}

\usepackage{amsthm}
\usepackage{algorithm}
\usepackage{algpseudocode}
\usepackage{bbm}
\usepackage{graphicx}
\usepackage{enumitem}

\newtheorem{theorem}{Theorem}
\newtheorem{lemma}{Lemma}
\newtheorem{corollary}{Corollary}
\newtheorem{definition}{Definition}
\newtheorem{remark}{Remark}
\usepackage{subcaption}
\usepackage[toc,page,header]{appendix}
\usepackage{minitoc}

\usepackage[textsize=tiny, color=lightgray]{todonotes}

\usepackage{amsmath}

\newcommand{\bq}{{\bf q}}
\newcommand{\br}{{\bf r}}
\newcommand{\bTheta}{{\boldsymbol{\Theta}}}

\newcommand{\1}{\mathbbm{1}}
\newcommand{\0}{\mathbf{0}}

\newcommand{\R}{\mathbb{R}}

\newcommand{\N}{\mathbb{N}}

\newcommand{\E}{\mathbb{E}}

\newcommand{\bma}{\begin{matrix*}[r]}
\newcommand{\ema}{\end{matrix*}}

\def\bA{{\mathbf A}}

\def\bH{{\mathbf H}}

\def\bK{{\mathbf K}}

\def\bQ{{\mathbf Q}}

\def\bV{{\mathbf V}}
\def\bW{{\mathbf W}}

\def\bsigma{{\boldsymbol \sigma}}
\def\bpi{{\boldsymbol \pi}}

\def\bmu{{\boldsymbol \mu}}

\def\be{{\mathbf e}}

\def\bh{{\mathbf h}}

\def\bp{{\mathbf p}}
\def\bt{{\mathbf t}}

\def\bw{{\mathbf w}}

\newcommand{\nrmp}[1]{{\left|\!\left|\!\left|{#1}\right|\!\right|\!\right|}}

\newcommand{\MLP}{\mathrm{MLP}}

\newcommand{\norm}[1]{\left\|{#1}\right\|} %
\newcommand{\bAtt}{\bTheta_{\tt attn}}
\newcommand{\bTF}{{\bTheta_{\tt TF}}}

\newcommand{\bthetamlp}{\bTheta_{\tt mlp}}

\newcommand{\lth}{{(\ell)}}

\newcommand{\Attn}{{\rm Attn}}

\newcommand{\TF}{{\rm TF}}
\newcommand{\TGMM}{{\rm TGMM}}

\newcommand{\paren}[1]{{\left( #1 \right)}}
\newcommand{\brac}[1]{{\left[ #1 \right]}}
\newcommand{\set}[1]{{\left\{ #1 \right\}}}
\newcommand{\sets}[1]{{\{ #1 \}}}

\newenvironment{talign}
 {\align}
 {\endalign}
\newenvironment{talign*}
 {\csname align*\endcsname}
 {\endalign}

\newcommand{\defeq}{\mathrel{\mathop:}=}
\newcommand{\wt}{\widetilde}
\newcommand{\barsig}{\sigma}

\usepackage[capitalise]{cleveref}
\crefformat{equation}{#2(#1)#3}
\Crefformat{equation}{#2(#1)#3}
\crefformat{figure}{Figure #2#1#3}
\Crefformat{figure}{Figure #2#1#3}

\newcommand{\bzero}{{\mathbf 0}}

\title{Transformers as Unsupervised Learning Algorithms: \\A study on Gaussian Mixtures}

\author{
Zhiheng Chen$^{1}$\thanks{Equal contribution.} , Ruofan Wu$^{2*}$, Guanhua Fang$^2$\thanks{Corresponding to: \texttt{fanggh@fudan.edu.cn}}\\
$^1$Shanghai Center for Mathematical Sciences, Fudan University,\\
$^2$Department of Statistics and Data Science, School of Management, Fudan University\\
\footnotesize{\texttt{zhchen22@m.fudan.edu.cn}, \texttt{wuruofan1989@gmail.com}, \texttt{fanggh@fudan.edu.cn}}
}

\iclrfinalcopy % Uncomment for camera-ready version, but NOT for submission.
\begin{document}

\maketitle

\begin{abstract}
    The transformer architecture has demonstrated remarkable capabilities in modern artificial intelligence, among which the capability of implicitly learning an internal model during inference time is widely believed to play a key role in the understanding of pre-trained large language models. However, most recent works have been focusing on studying supervised learning topics such as in-context learning, leaving the field of unsupervised learning largely unexplored.
    This paper investigates the capabilities of transformers in solving Gaussian Mixture Models (GMMs), a fundamental unsupervised learning problem through the lens of statistical estimation.
    We propose a transformer-based learning framework called Transformer for Gaussian Mixture Models (TGMM) that simultaneously learns to solve multiple GMM tasks using a shared transformer backbone. The learned models are empirically demonstrated to effectively mitigate the limitations of classical methods such as Expectation-Maximization (EM) or spectral algorithms, at the same time exhibit reasonable robustness to distribution shifts.
    Theoretically, we prove that transformers can efficiently approximate both the Expectation-Maximization (EM) algorithm and a core component of spectral methods—namely, cubic tensor power iterations. These results not only improve upon prior work on approximating the EM algorithm,
    % \cite{he2025transformersversusemalgorithm}
    but also provide, to our knowledge, the first theoretical guarantee that transformers can approximate high-order tensor operations.
    Our study bridges the gap between practical success and theoretical understanding, positioning transformers as versatile tools for unsupervised learning. 
    \footnote{Code available at \href{https://github.com/Rorschach1989/transformer-for-gmm}{\texttt{https://github.com/Rorschach1989/transformer-for-gmm}}}
\end{abstract}

\doparttoc % Tell to minitoc to generate a toc for the parts
\faketableofcontents % Run a fake tableofcontents command for the partocs

\section{Introduction}
Large Language Models (LLMs) have achieved remarkable success across various tasks in recent years. 
Transformers\citep{attention-is-all-you-need}, the dominant architecture in modern LLMs\citep{BrownGPT3}, outperform many other neural network models in efficiency and scalability. 
Beyond language tasks, transformers have also demonstrated strong performance in other domains, such as computer vision\citep{survey-vision-TF, TF-in-vision} and reinforcement learning\citep{li2023TF-RL}. 
Given their practical success, understanding the mechanisms behind transformers has attracted growing research interest.
Existing studies often treat transformers as algorithmic toolboxes, investigating their ability to implement diverse algorithms\citep{von-oswald23a, bai2023tfstats, lin2024transformers, giannou2025how, teh2025empiricalbayestransformers}--a perspective linked to meta-learning\citep{hospedales2021meta}. 

However, most research has focused on supervised learning settings, such as regression\citep{bai2023tfstats} and classification\citep{giannou2025how}, leaving the unsupervised learning paradigm relatively unexplored.
Since transformer models are typically trained in a supervised manner, unsupervised learning poses inherent challenges for transformers due to the absence of labeled data. 
Moreover, given the abundance of unlabeled data in real-world scenarios, investigating the mechanisms of transformers in unsupervised learning holds significant implications for practical applications.
The Gaussian mixture model (GMM) represents one of the most fundamental unsupervised learning tasks in statistics, with a rich historical background\citep{GMM1969,Aitkin01081980} and ongoing research interest\citep{gmm2021ACM-TOMM,dcgmm-nips-2021,loffler2021optimality,ndaoud2022sharp,gribonval2021statistical,yu2021scgmai}. 
Two primary algorithmic approaches are existing for solving GMM problems: (1) likelihood-based methods employing the Expectation-Maximization (EM) algorithm\citep{EM1977,BinYu-EM-2017}, and (2) moment-based methods utilizing spectral algorithms\citep{10.1145LearningGMM,JMLR:v15:anandkumar14b}.
However, both algorithms have inherent limitations. The EM algorithm is prone to convergence at local optima and is highly sensitive to initialization\citep{moitra2018algorithmic,chijinEM}. In contrast, while the spectral method is independent of initialization, it requires the number of components to be smaller than the data's dimensionality—an assumption that restricts its applicability to problems involving many components in low-dimensional GMMs\citep{10.1145LearningGMM}.
% However, both algorithms have limitations. The EM algorithm suffers from local optima convergence, with its effectiveness being highly sensitive to initialization. In contrast, the spectral method does not depend on initialization, but requires the number of components to be smaller than the data dimension—a restrictive assumption that hinders its application to many-component, low-dimensional GMM problems.

In this work, we explore transformers for GMM parameter estimation to address two questions.
(i) Can Transformers \textit{provably} work for GMM in-context?
(ii) Can Transformers \textit{empirically} overcome the drawbacks of both EM algorithm and the spectral method?
Our answers are affirmative.
We find that meta-trained transformers exhibit strong performance on GMM tasks without the aforementioned limitations.
Notably, we construct transformer-based solvers that efficiently solve GMMs with varying component counts simultaneously. 
The experimental phenomena are further backed up by novel theoretical establishments: We prove that transformers can effectively learn GMMs with different components by approximating both the EM algorithm and a key component of spectral methods on GMM tasks.
% In a theoretical perspective, we provide results showing that transformers can effectively learn GMMs with different components simultaneously via approximating both the EM algorithm and a key component of spectral methods on GMM tasks.

% Our contributions are mainly summarized as follows.

\textbf{Main Contributions}.
\vspace{-3mm}
\begin{itemize}[leftmargin=*]
    % \item Through extensive experimentation, we systematically investigate transformer behavior on Gaussian mixture models. 
    % Our results reveal that: (1) Transformers can overcome the inherent limitations of both EM and spectral algorithms, achieving outstanding performance on GMM tasks. 
    % (2)Transformers can learn GMM tasks with different components simultaneously.
    % Our results reveal that transformers simultaneously exhibit characteristics of both EM algorithms (e.g., effectiveness with few layers/steps) and spectral methods (e.g., comparable error patterns). 
    % (3)Transformers possess remarkable sequence generalization and out-of-domain (OOD) generalization capabilities on GMM tasks.
    \item We propose the TGMM framework that utilizes transformers to solve multiple GMM tasks with varying numbers of components simultaneously during inference time. Through extensive experimentation, the learned TGMM model is demonstrated to achieve competitive and robust performance over synthetic GMM tasks. Notably, TGMM outperforms the popular EM algorithm in terms of estimation quality, and approximately matches the strong performance of spectral methods while enjoying better flexibility.
    \item We establish theoretical foundations by proving that transformers can approximate both the EM algorithm and a key component of spectral methods. 
    Our approximation of the EM algorithm fundamentally leverages the weighted averaging property inherent in softmax attention, enabling simultaneous approximation of both the E and M steps. 
    Notably, our approximation results also hold across varying dimensions and mixture components in GMM.
    
    \item We proved that transformers (with RELU activation) can implement cubic tensor power iterations- a crucial component of spectral algorithms for GMM. The proof is highly dependent on the multi-head structure of transformers. To the best of our knowledge, this is the first theoretical demonstration of transformers’ capacity for high-order tensor calculations.
    % We prove that ReLU-activated transformers can exactly approximate cubic tensor power iterations -- a crucial component of spectral algorithms for GMM. This result stems from a novel insight about attention heads: they can perform computations along dimensions beyond the query-key-value framework. To the best of our knowledge, this is the first theoretical demonstration of transformers' capacity for high-order tensor calculations.
    
    % Our findings enhance the theoretical understanding of transformers and offer valuable insights for future research.
\end{itemize}

\vspace{-3mm}

{\bf Related works.} 
Recent research has explored the mechanisms by which transformers can implement various supervised learning algorithms. 
% Notably, a growing body of work has focused on the In-context learning capabilities of transformers. 
For instance, \cite{akyrek2023what}, \cite{von-oswald23a}, and \cite{bai2023tfstats} demonstrate that transformers can perform gradient descent for linear regression problems in-context. \cite{lin2024transformers} shows that transformers are capable of implementing Upper Confidence Bound (UCB) algorithms, as well as other classical algorithms in reinforcement learning tasks. \cite{giannou2025how} reveals that transformers can execute in-context Newton's method for logistic regression problems. \cite{teh2025empiricalbayestransformers} illustrates that transformers can approximate Robbins' estimator and solve Naive Bayes problems. \cite{kim2024transformers} studies the minimax optimality of transformers on nonparametric regression. Some literature on density estimation using LLMs is discussed in \cref{secapp:additional-review}.
% \cite{he2025learningspectralmethodstransformers} demonstrates that transformers can implement Principal Component Analysis (PCA).

\textbf{Comparison with prior theoretical works in unsupervised learning setting.} 
Several recent studies have investigated the mechanisms of transformer-based models in mixture model settings\citep{he2025learningspectralmethodstransformers,jin2025incontextlearningmixturelinear,he2025transformersversusemalgorithm}. Among these, \cite{he2025learningspectralmethodstransformers} establishes that transformers can implement Principal Component Analysis (PCA) and leverages this to GMM clustering. However, their analysis is limited to the two-component case, restricting its broader applicability.

The paper \cite{jin2025incontextlearningmixturelinear} investigates the in-context learning capabilities of transformers for mixture linear models, a setting that differs from ours. 
Furthermore, their approximation construction of the transformer is limited to two-component GMMs, leaving the general case unaddressed. 
While they assume ReLU as the activation function--contrary to the conventional choice of softmax--their theoretical proofs rely on a key lemma from prior work \cite{pathak2024transformers} that assumes softmax activation, thereby introducing an inconsistency in their assumptions.
The paper \cite{he2025transformersversusemalgorithm} studies the performance of transformers on multi-class GMM clustering, a setting closely related to ours. However, our work focuses on \textit{parameter estimation} rather than \textit{clustering}. 
We give a discussion of our theoretical improvements over their work in detail in the following paragraph.
% From a technical standpoint, their approximation bound scales exponentially with the dimension $d$, and their analysis requires the number of attention heads $M$ to grow to infinity to achieve a meaningful bound.
% In contrast, our results are significantly stronger: our bounds depend polynomially on $d$, and we only require $M=O(1)$.
From an empirical perspective, their experiments are conducted on a small-scale transformer, which fails to validate their theoretical claims.

\textbf{Sharpness of our results}. 
Our theoretical analysis fully leverages key architectural components of Transformers: the query-key-value mechanism, multi-head attention, and the properties of the activation function.
It is worth pointing out that our result improves the prior work for EM approximation in several points: First, Our analysis shows that Transformers can approximate L-step EM algorithms with just O(L) layers, a significant improvement over prior work \citep{he2025transformersversusemalgorithm} , which requires O(KL) layers (dependent on the number of components K). Second, unlike \cite{he2025transformersversusemalgorithm}, which needs number of attention heads $M\rightarrow +\infty$ to get valid bounds, our results hold with $M = O(1)$, aligning better with real-world designs. Third, our approximation bounds scale polynomially in dimension $d$, unlike \cite{he2025transformersversusemalgorithm}’s exponential dependence--a crucial improvement for high-dimensional settings. 
We believe our results and proofs can offer profound insights for subsequent theoretical research on transformers.

\textbf{Organization}. The rest of paper is organized as follows. In \cref{sec:pre}, some background knowledge is introduced. In \cref{sec:experiment}, we present the experimental details and findings. The theoretical results are proposed in \cref{sec:theory}, and some discussions are given in \cref{sec:discussion}. The proofs and additional experimental results are given in the appendix.

{\bf Notations.} We introduce the following notations. Let $[n] \defeq \sets{1,2,\cdots,n}$.
All vectors are represented as column vectors unless otherwise specified. For a vector $v \in \R^d$, we denote $\norm{v}$ as its Euclidean norm. 
% All matrices are written in boldfaced uppercase characters. 
For two sequences $a_n$ and $b_n$ indexed by $n$, we denote $a_n = O(b_n)$ if there exists a universal constant $C$ such that $a_n \leq C b_n$ for sufficiently large $n$. 

\vspace{-2mm}

\section{Methodology}
\label{sec:pre}

\vspace{-2mm}

\subsection{Preliminaries}

\vspace{-2mm}

{The Gaussian mixture model (GMM) is a cornerstone of unsupervised learning in statistics, with deep historical roots and enduring relevance in modern research. 
Since its early formalizations\citep{GMM1969,Aitkin01081980}, GMM has remained a fundamental tool for clustering and density estimation, widely applied across diverse domains. 
Recent advances have further explored the theoretical foundations of Gaussian Mixture Models (GMMs)\citep{loffler2021optimality,ndaoud2022sharp,gribonval2021statistical}, extended their applications in incomplete data settings\citep{gmm2021ACM-TOMM}, and integrated them with deep learning frameworks\citep{dcgmm-nips-2021,yu2021scgmai}. 
Due to their versatility and interpretability, GMMs remain indispensable in unsupervised learning, effectively bridging classical statistical principles with modern machine learning paradigms.}
% The probability density function of the classical Gaussian Mixture Model with $K$ components is given by
We consider the (unit-variance) isotropic Gaussian Mixture Model with $K$ components, with its probability density function as
\begin{align}
    p(x|\btheta) = \sum_{k = 1}^{K} \pi_k \phi(x;\mu_k) ~,
    \label{eqn: GMM density}
\end{align}
where $\phi(x;\mu)$ is the standard Gaussian kernel, i.e. $\phi(x;\mu) = \frac{1}{(2\pi)^{d/2}} \exp\left(-\frac{1}{2}(x - \mu)^{\top}(x - \mu)\right)$. 
The parameter $\btheta$ is defined as $\btheta = \bpi \cup \bmu$, where $\bpi \defeq \sets{\pi_1, \pi_2, \cdots, \pi_K}$, $\pi_k \in \R$ and  $\bmu = \sets{\mu_1, \mu_2, \cdots,\mu_K},\mu_k\in \R^d$, $k \in [K]$.
We take $N$ samples $\rmX = \sets{X_i}_{i\in [N]}$ from model \cref{eqn: GMM density}. $\sets{X_i}_{i\in [N]}$ can be also rewritten as
\begin{align*}
    X_i = \mu_{y_i} + Z_i,
\end{align*}
where $\sets{y_i}_{i\in[N]}$ are i.i.d. discrete random variables with $\sP\paren{y=k} = \pi_k$ for $k\in[K]$ and $\sets{Z_i}_{i\in[N]}$ are i.i.d. standard Gaussian random vector in $\R^d$.

The EM algorithm\citep{EM1977} remains the most widely used approach for GMM parameter estimation.
Due to space constraints, we propose the algorithm in \cref{secapp:alg}.
Alternatively, the spectral algorithm\citep{10.1145LearningGMM} offers an efficient moment-based approach that estimates parameters through low-order observable moments. 
A key component of this method is cubic tensor decomposition\citep{JMLR:v15:anandkumar14b}. 
For brevity, we defer the algorithmic details to \cref{secapp:alg}.
% Another efficient and widely used algorithm for solving GMM is the spectral algorithm\citep{10.1145LearningGMM}, which leverages low-order observable moments to estimate parameters.
% The cubic tensor decomposition is the key component of the spectral algorithm\cite{JMLR:v15:anandkumar14b}.
% Also, we propose the details of this algorithm in \cref{secapp:alg}.

% \subsection{Transformer Structure}
Next, we give a rigorous definition of the transformer model.
To maintain consistency with existing literature, we adopt the notational conventions presented in \cite{bai2023tfstats}, with modifications tailored to our specific context. 
We consider a sequence of $N$ input vectors $\{h_i\}_{i=1}^N \subset \mathbb{R}^D$, which can be compactly represented as an input matrix $\bH = [h_1, \dots, h_N] \in \mathbb{R}^{D \times N}$, where each $h_i$ corresponds to a column of $\bH$ (also referred to as a token).

Here we introduce several useful definitions and their full notations are given in Appendix \ref{sec:full_notation}.

\begin{definition}[Attention layer]
% \label{def:attention}
A (self-)attention layer with $M$ heads is denoted as $\Attn_{\bAtt}(\cdot)$ with parameters $\bAtt=\sets{ (\bV_m,\bQ_m,\bK_m)}_{m\in[M]}\subset \R^{D\times D}$. 
%On any input sequence $\bH\in\R^{D\times N}$,
% \begin{talign}
% % \label{eqn:attention}
%     \wt{\bH} = \Attn_{\bAtt}(\bH)\defeq \bH + \sum_{m=1}^M (\bV_m \bH) \text{SoftMax}\paren{ (\bK_m\bH)^\top (\bQ_m\bH) } \in \R^{D\times N}, \nonumber
% \end{talign}
% % where $\sursf: \R^N \to \R^N$ is the softmax function. 
% In vector form,
% \begin{talign*}
%     \wt{\bh}_i = \brac{\Attn_{\bAtt}(\bH)}_i = \bh_i + \sum_{m=1}^M \sum_{j=1}^{N} \brac{\text{SoftMax}\paren{ \paren{\paren{\bQ_m\bh_i}^\top \paren{\bK_m\bh_j}}_{j=1}^{N} }}_j  \bV_m\bh_j.
% \end{talign*}
% Here $\sfm$ is the activation function defined by
% \(
%     \sfm\paren{v} = \paren{\frac{\exp(v_1)}{\sum_{i=1}^d \exp(v_i)},\cdots,\frac{\exp(v_d)}{\sum_{i=1}^d \exp(v_i)}}
% \)
% for $v \in \R^d$.
\end{definition}

% The Multilayer Perceptron(MLP) layer is defined as follows.
\begin{definition}[MLP layer]
% \label{def:mlp}
A (token-wise) MLP layer with hidden dimension $D'$ is denoted as $\MLP_{\bthetamlp}(\cdot)$ with parameters $\bthetamlp=(\bW_1,\bW_2)\in\R^{D'\times D}\times\R^{D\times D'}$. 
%On any input sequence $\bH\in\R^{D\times N}$, 
% \begin{talign*}
%     \wt{\bH} = \MLP_{\bthetamlp}(\bH) \defeq \bH + \bW_2\barsig(\bW_1\bH),
% \end{talign*}
% where $\barsig: \R \to \R$ is the ReLU function. In vector form, we have $\wt{\bh}_i=\bh_i+\bW_2\sigma(\bW_1\bh_i)$. 
\end{definition}
% Then we can use the above definitions to define the transformer model.
% We consider a transformer architecture with $L\ge 1$ transformer layers, each consisting of a self-attention layer followed by an MLP layer. 
\begin{definition}[Transformer]
% \label{def:tf}
An $L$-layer transformer, denoted as $\TF_\bTF(\cdot)$, is a composition of $L$ self-attention layers each followed by an MLP layer:
\begin{align*}
    \TF_\bTF(\bH) = \MLP_{\bthetamlp^{(L)}}\paren{ \Attn_{\bAtt^{(L)}}\paren{\cdots \MLP_{\bthetamlp^{(1)}}\paren{ \Attn_{\bAtt^{(1)}}\paren{\bH}} }}.
\end{align*}
% Here the parameter $\bTF=(\bAtt^{(1:L)},\bthetamlp^{(1:L)})$ consists of the attention layers $\bAtt^{(\ell)}=\sets{ (\bV^{(\ell)}_m,\bQ^{(\ell)}_m,\bK^{(\ell)}_m)}_{m\in[M^{(\ell)}]}\subset \R^{D\times D}$, the MLP layers $\bthetamlp^{(\ell)}=(\bW^{(\ell)}_1,\bW^{(\ell)}_2)\in\R^{D^{(\ell)}\times D}\times \R^{D\times D^{(\ell)}}$.
\end{definition}

\vspace{-4mm}

\subsection{The TGMM architecture}
\label{sec:TGMM-archi}
\vspace{-2mm}

A recent line of work\citep{xie2021explanation, garg2022can, bai2023tfstats, akyrek2023what, li2023transformers} has been studying the capability of transformer that functions as a data-driven algorithm under the context of in-context learning (ICL). However, in contrast to the setups therein where inputs consist of both features and labels, under the unsupervised GMM setup, there is no explicitly provided label information. Therefore, we formulate the learning problem as learning an \emph{estimation} algorithm instead of learning a \emph{prediction} algorithm as in the case of ICL. A notable property of GMM is that the structure of the estimand depends on an unknown parameter $K$, which is often treated as a hyper-parameter in GMM estimation\citep{titterington_85, mclachlan2000finite}. For clarity of representation, we define an isotropic Gaussian mixture task as $\mathcal{T} = (\btheta, \rmX, K)$, where $\rmX$ is a i.i.d. sample generated according to ground truth $\btheta$ according to the isotropic GMM law and $K$ is the configuration used during estimation which we assume to be the same as the number of components of the ground truth $\btheta$. The GMM task is solved via applying some algorithm $\mathcal{A}$ that takes $\rmX$ and $K$ as inputs and outputs an estimate of the ground truth $\widehat{\btheta} = \mathcal{A}(\rmX; K)$. 

In this paper, we propose a transformer-based architecture, transformers-for-Gaussian-mixtures (TGMM), as a GMM task solver that allows flexibility in its outputs, while at the same time being parameter-efficient, as illustrated in \cref{fig:task_fig}: A TGMM model supports solving $s$ different GMM tasks with $K \in \mathcal{K} := \sets{K_1, \ldots, K_s}$. Given inputs $N$ data points $\rmX \in \R^{d \times N}$ and a structure configuration of the estimand $K$. TGMM first augments the inputs with auxiliary configurations about $K$ via concatenating it with a task embedding $\rmP = \text{embed}(K)$, i.e., $\rmH = \left[\rmX || \rmP\right]$, and use a linear $\operatorname{Readin}$ layer to project the augmented inputs onto a shared hidden representation space for several estimand structures $\sets{K_1, \ldots, K_s}$, which is then manipulated by a shared transformer backbone that produces task-aware hidden representations. The TGMM estimates are then decoded by task-specific $\operatorname{Readout}$ modules. More precisely, with target decoding parameters of $K$ components, the $\operatorname{Readout}$ module first performs an attentive-pooling operation\citep{lee2019set}:
\begin{align}
    \mathbf{O} = (\bV_o \bH) \text{SoftMax}\paren{ (\bK_o\bH)^\top \bQ_o } \in \R^{(d+K)\times K}, \nonumber
\end{align}
where $\bV_o,\bK_o \in \R^{(d+K)\times D}$, $\bQ_o \in \R^{(d+K)\times K}$. The estimates for mixture probability are then extracted by a row-wise mean-pooling of the first $K$ rows of $\mathbf{O}$, and the estimates for mean vectors are the last $d$ rows of $\mathbf{O}$. We wrap the above procedure as $\sets{\widehat{\pi}_k, \widehat{\mu}_k}_{i\in[K]} = \operatorname{Readout}_{\bTheta_{\tt out}}\paren{\rmH}$.
% The purpose of this function is to aggregate the results to match the required parameter dimensions. The mixture probability $\bpi$ can be computed by averaging the first $K$ rows and the last $d$ rows are the mixture Gaussian means.\par 
TGMM is parameter-efficient in the sense that it only introduces extra parameter complexities of the order $O(sdD)$ in addition to the backbone. We give a more detailed explanation of the parameter efficiency of TGMM in appendix \cref{sec: parameter_efficiency}. 
We wrap the TGMM model into the following form:
\begin{align*}
    \TGMM_\bTheta(\rmX; K) = \operatorname{Readout}_{\bTheta_{\tt out}}\paren{\TF_\bTF \paren{\operatorname{Readin}_{\bTheta_{\tt in}}\paren{\left[\rmX || \text{embed}(K) \right]}}}.
    % \TF_\bTheta(\bH) = \operatorname{Readout}_{\bTheta_{\tt out}}\paren{\MLP_{\bthetamlp^{(L)}}\paren{ \Attn_{\bAtt^{(L)}}\paren{\cdots \MLP_{\bthetamlp^{(1)}}\paren{ \Attn_{\bAtt^{(1)}}\paren{\operatorname{Readin}_{\bTheta_{\tt in}}\paren{\bH}} }}}}
\end{align*}
% In our setting, we apply $\operatorname{Readin}$ and $\operatorname{Readout}$ functions to encode and decode latent tokens:
% \begin{align*}
%     \TF_\bTheta(\bH) = \operatorname{Readout}_{\bTheta_{\tt out}}\paren{\widetilde{\TF}\paren{\operatorname{Readin}_{\bTheta_{\tt in}}\paren{\bH}}}.
%     % \TF_\bTheta(\bH) = \operatorname{Readout}_{\bTheta_{\tt out}}\paren{\MLP_{\bthetamlp^{(L)}}\paren{ \Attn_{\bAtt^{(L)}}\paren{\cdots \MLP_{\bthetamlp^{(1)}}\paren{ \Attn_{\bAtt^{(1)}}\paren{\operatorname{Readin}_{\bTheta_{\tt in}}\paren{\bH}} }}}}
% \end{align*}
% $\bH^{(L)}=\TF_{\bTheta}(\bH^{(0)})$, where $\bH^{(0)}\in\R^{D\times N}$ is the input sequence, and
% \begin{talign*}
% \bH^{(\ell)} = \MLP_{\bthetamlp^{(\ell)}}\paren{ \Attn_{\bAtt^{(\ell)}}\paren{\bH^{(\ell-1)}} },~~~\ell\in\set{1,\dots,L}.
% \end{talign*}
Above, the parameter $\bTheta=(\bTF, \bTheta_{\tt in}, \bTheta_{\tt out})$ consists of the parameters in the transformer $\bTF$ and the parameters in the $\operatorname{Readin}$ and the $\operatorname{Readout}$ functions $\bTheta_{\tt in}$, $\bTheta_{\tt out}$.
% We also consider ``\emph{attention-only}'' transformers with $\bW_1^\lth,\bW_2^\lth=\bzero$, which we denote as $\TFz_\bTheta(\cdot)$ for shorthand, with $\bTheta=\bTheta^{(1:L)}\defeq \bAtt^{(1:L)}$.
% \begin{remark}
% In practice, we  implement the $\operatorname{Readin}$ function as an affine transformation and the $\operatorname{Readout}$ functions as an attentive pooling layer \cite{lee2019set}. 
% \end{remark}
% To learn GMM tasks with different components simultaneously, we adapt the same transformer backbone with different $\operatorname{Readout}$ functions.
% For convenience, we also denote the parameters of transformers with different Readouts as $\bTheta$.
\begin{figure}[htbp]
    \centering
    \includegraphics[width=0.8\linewidth]{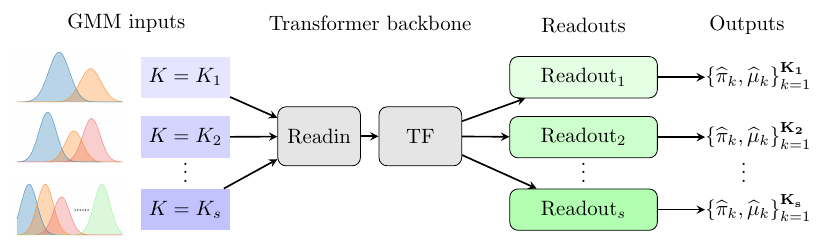}
    \caption{Illustration of the proposed TGMM architecture: TGMM utilizes a shared transformer backbone that supports solving $s$ different kind of GMM tasks via a task-specific $\operatorname{Readout}$ strategies.}
    % \caption{Learning different GMM tasks with different $\operatorname{Readout}$ functions.}
    \label{fig:task_fig}
    \vspace{-5mm}
\end{figure}

\subsection{Meta training procedure}\label{sec: meta_training_procedure}

\vspace{-2mm}

We adopt the meta-training framework as in \cite{garg2022can, bai2023tfstats} and utilize diverse synthetic tasks to learn the TGMM model. In particular, during each step of the learning process, we first use a {TaskSampler} routine (described in \cref{alg: task_sampler}) to generate a batch of $n$ tasks, with each task having a probably distinct sample size. The TGMM model outputs estimates for each task, i.e., $\sets{\widehat{\mu}_k, \widehat{\pi}_k}_{k\in [K]} = \TGMM_{\bTheta}\paren{\rmX; K}$. 
{Define $\widehat{\bpi} \defeq \sets{\widehat{\pi}_k}_{k\in[K]}$ and $\widehat{\bmu} \defeq \sets{\widehat{\mu}_k}_{k\in[K]}$.} For a batch of tasks $\sets{\mathcal{T}_i}_{i \in [n]} = \sets{\rmX_i, \btheta_i, K_i}_{i \in [n]}$, denote by $\btheta_i = \bmu_i \cup \bpi_i$ and $\widehat{\btheta}_i = \widehat{\bmu}_i \cup \widehat{\bpi}_i = \TGMM_{\bTheta}\paren{\rmX_i; K_i}$, $i\in [n]$. Then the learning objective is thus:
\vspace{-2mm}
\begin{align}
\label{eqn:loss}
    % \widehat{\bTheta} = \underset{\TF_\bTheta \in \gF}{\arg\min}
    \widehat{L}_n\paren{\bTheta} = \frac{1}{n} \sum_{i=1}^n \ell_\mu(\widehat{\bmu}_i,\bmu_i) + \ell_\pi(\widehat{\bpi}_i,\bpi_i).
    % \widehat{L}_n\paren{\bTheta} = \frac{1}{n} \sum_{i=1}^n \ell_\mu(\widehat{\mu}_i,\mu_i) + \ell_\pi(\widehat{\pi}_i,\pi_i).
\end{align}
where $\ell_\mu$ and $\ell_\pi$ are loss functions for estimation of $\mu$ and $\pi$, respectively. We will by default use square loss for $\ell_\mu$ and cross entropy loss for $\ell_\pi$. Note that the task sampling procedure relies on several sampling distributions $p_{\mu}, p_{\pi}, p_N, p_K$, which are themselves dependent upon some global configurations such as the dimension $d$ as well as the task types $\mathcal{K}$. We will omit those dependencies on global configurations when they are clear from context. The (meta) training procedure is detailed in \cref{alg: meta_training}.
% The GMM problem is unsupervised, where no labels are given. However, transformers are usually used in the supervised setting. To make GMM learnable for transformers, we need to design supervised pretraining. To describe this process, we consider the setting where we have a total number of $n$ training data instances $\set{\rmX^{i}}_{i\in[n]}$ and the corresponding parameter $\set{\btheta^i}_{i\in[n]}$. We encode the instances as $\set{\rmH^{i}}_{i\in[n]}$, where the formula of $\rmH$ are given by
% \begin{align}
% \label{eqn:token_formula}
%     \rmH…^i = \begin{bmatrix}
%         \rmX^i\\\rmP^i
%     \end{bmatrix}
%     \in \R^{D\times N^i},
% \end{align}
% where the matrix $\rmP^i$ is the auxiliary matrix which collects all other useful information. In our experiments, $\rmP$ contains the information on the number of mixture components. For theoretical analysis, we need to design $\rmP$ to store the initial parameters, see \cref{sec:theory} for details.
% Given the training data, the pretraining procedure can be described by minimizing the following objective for some loss function $\ell$,
% \begin{align}
% \label{eqn:loss}
%     % \widehat{\bTheta} = \underset{\TF_\bTheta \in \gF}{\arg\min}
%     \widehat{L}_n\paren{\bTheta} = 
%     \sum_{i=1}^n \ell(\TF_\bTheta\paren{\bH^i},\btheta^i).
% \end{align}
% In practice, the loss function $\ell$ comprises two distinct terms: a cross-entropy loss for the weight $\set{\pi_k}_{k\in[K]}$ and a mean square loss for the mean $\set{\mu_k}_{k\in[K]}$. Our meta-training algorithm is summarized as follows.
% To learn the TGMM model, we 
\begin{figure}
\begin{minipage}{0.48\textwidth}
    \begin{algorithm}[H]
        \caption{TaskSampler}
        \label{alg: task_sampler}
        \begin{algorithmic}[1]
            \Require sampling distributions $p_{\mu}, p_{\pi}, p_N, p_K$.
            \State Sample the type of task (i.e., number of mixture components) $K \sim p_K$.
            \State Sample a GMM task according to the type of task
            \begin{align*}
                \begin{aligned}
                    &\btheta = (\bmu, \bpi),\\
                    &\bmu \sim p_\mu, \bpi \sim p_\pi,
                \end{aligned}
            \end{align*}
            where $\bmu = \sets{\mu_1, \cdots,\mu_K}$, $\bpi = \sets{\pi_1, \cdots, \pi_K}$.
            \State Sample the size of inputs $N \sim p_N$.
            \State Sample the data points $\rmX = (X_1, \ldots, X_N) \overset{\text{i.i.d.}}{\sim} p(\cdot | \btheta)$.
            \State \Return An (isotropic) GMM task $\mathcal{T} = (\rmX, \btheta, K)$.
        \end{algorithmic}
    \end{algorithm}
\end{minipage}
\hfill
\begin{minipage}{0.48\textwidth}
    \begin{algorithm}[H]
        \caption{(Meta) Training procedure for TGMM}
        \label{alg: meta_training}
        \begin{algorithmic}[1]
            \Require task dimension $d$, task types $\mathcal{K} = \{K_1, \ldots, K_s\}$, number of tasks $n$ per step, number of steps $T$.
            \State Initialize a TGMM model $\TGMM_{\bTheta^{(0)}}$.
            \For{$t = 1:T$}
                \State Sample $n$ tasks
                % \footnote{Hereafter we will from time to time refer to $n$ as the \emph{batch size} hyper-parameter.}
                $\sets{\mathcal{T}_i}_{i \in [n]}$ independently using the {TaskSampler} from Algorithm \ref{alg: task_sampler}.
                \State Compute the training objective $\widehat{L}_n\paren{\bTheta^{(t-1)}}$ as in \eqref{eqn:loss}.
                \State Update $\bTheta^{(t-1)}$ into $\bTheta^{(t)}$ using any gradient based training algorithm like AdamW.
            \EndFor
            \State \Return Trained model $\TGMM_{\bTheta^{(T)}}$.
        \end{algorithmic}
    \end{algorithm}
\end{minipage}
\vspace{-4mm}
\end{figure}

\vspace{-3mm}

\section{Experiments}
\label{sec:experiment}

\vspace{-3mm}

% We present our experimental results in this section. 
% The objective of experiments is to investigate the expressive capabilities of transformer models through the lens of the Gaussian Mixture Model (GMM) problem. 
% While space limitations permit only key findings to be shown here, additional results are provided in Appendix B. 
In this section, we empirically investigate TGMM's capability of learning to solve GMMs. We focus on the following research questions (RQ):\\
\textbf{RQ1 Effectiveness}: How well do TGMM solve GMM problems, compared to classical algorithms?\\
\textbf{RQ2 Robustness}: How well does TGMM perform over test tasks unseen during training? \\
\textbf{RQ3 Flexibility}: Can we extend the current formulation by adopting alternative backbone architectures or relaxing the isotropic setting to more sophisticated models like anisotropic GMM?  
% {\bf Data Generation.} 
% We generate GMM data with the dimension $d \in \sets{2, 8, 32, 128}$ and $K \in \sets{2, 3, 4, 5}$.  The mixture weights $\sets{\pi_k}_{k\in[K]}$ are sampled uniformly from $[0.2,0.8]$ and then normalized to a probability vector.  The Gaussian mean vectors $\sets{\mu_k}_{k\in[K]}$ are sampled uniformly from $[-5,5]^d$.

\subsection{Experimental Setup}\label{sec: exp_setup}
We pick the default backbone of TGMM similar to that in \cite{garg2022can, bai2023tfstats}, with a GPT-2 type transformer encoder\citep{radford2019language} of $12$ layers, $4$ heads, and $128$-dimensional hidden state size. The task embedding dimension is fixed at $128$. Across all the experiments, we use AdamW\citep{loshchilov2017fixing} as the optimizer and use both learning rate and weight decay coefficient set to $10^{-4}$ without further tuning. During each meta-training step, we fix the batch size to be $64$ and train $10^{6}$ steps. For the construction of {TaskSampler}, the sampling distributions are defined as follows: For $p_K$, We sample $K$ uniformly from $\sets{2, 3, 4, 5}$; For $p_\mu$, given dimension $d$ and number of components $K$, we sample each component uniformly from $[-5,5]^d$. Additionally, to prevent collapsed component means\citep{ndaoud2022sharp}, we filter the generated mean vectors with a maximum pairwise cosine similarity threshold of $0.8$. For $p_\pi$, given $K$, we sample each $\pi_k$ uniformly from $[0.2,0.8]$ and normalize them to be a probability vector; For $p_N$, Given a maximum sample size $N_0$, we sample $N$ uniformly from $[N_0/2, N_0]$. The default choice of $N_0$ is $128$.
During evaluation, we separately evaluate $4$ tasks with $2,3,4,5$ components, respectively. With a sample size of $128$ and averaging over $1280$ randomly sampled tasks.
% We train a transformer with 12
% layers, 4 heads, and 128-dimensional embedding size. 
% We train each setting $10000$ iterations, and evaluate model performance every $1000$ iterations. 
% In each iteration, we use $32$ sample instances each iteration (i.e., the batch-size is $32$), with the number of in-context samples uniformly drawn from $[N/2, N]$ in each instance.
% Currently, we choose $N = 64$.
% For evaluation, the batch-size is $128$ and the number of in-context samples is $64$.
% We report the mean of the resulting metrics.
% The model is trained using the Adam optimizer with a learning rate of $0.001$ and a weight decay rate of $0.0001$. 
% It should be noted that we perform both training and testing on sample instances with different values of $K$ \textit{simultaneously}.

{\bf Metrics. }
We use $\ell_2${\texttt{-error}} as evaluation metrics in the experiments. We denote the output of the TGMM as $\widehat{\btheta} \defeq \sets{\widehat{\pi}_1, \widehat{\mu}_1, \widehat{\pi}_2,\widehat{\mu}_2, \cdots \widehat{\pi}_K,\widehat{\mu}_K}$. The rigorous definition is 
\begin{align*}
    % \min_{\sigma} 
    \frac{1}{K }\sum_{k\in[K]} \paren{\frac{1}{d}\norm{\widehat{\mu}_{\tilde{\sigma}(i)} - \mu_i}^2 + \paren{\widehat{\pi}_{\tilde{\sigma}(i)} - \pi_i}^2},
\end{align*}
where $\tilde{\sigma}$ is the permutation such that  $\tilde{\sigma} = \arg \min_{\sigma}\sum_{k\in[K]} \norm{\widehat{\mu}_{{\sigma}(i)} - \mu_i}^2$. We obtain the permutation via solving a linear assignment program using the Jonker-Volgenant algorithm\citep{crouse2016implementing}. We also report all the experimental results under two alternative metrics: cluster-classification accuracy and log-likelihood in \cref{sec: complete_exp_metrics}.

\vspace{-2mm}
\subsection{Results and findings}
\label{sec:results-and-findings}
\textbf{RQ1: Effectiveness}
\vspace{-2mm}

% We compare the effect of transformer on the GMM problem (denoted as TGMM) with the classical EM algorithm and spectral algorithm. The results are presented in \cref{fig:l2 comparison}.
% As shown in the results, TGMM performs comparably to classical algorithms across task settings, outperforming EM in most cases and closely matching the spectral algorithm.
We compare the performance of a learned TGMM with the classical EM algorithm and spectral algorithm under $4$ scenarios where the problem dimension ranges over $\sets{2, 8, 32, 128}$. The results are reported in \cref{fig:l2 comparison}. We observe that all three algorithms perform competitively (reaching almost zero estimation error) when $K=2$. However, as the estimation problem gets more challenging as $K$ increases, the EM algorithm gets trapped in local minima and underperforms both spectral and TGMM. Moreover, while the spectral algorithm performs comparably with TGMM, it cannot handle cases when $K > d$, which is effectively mitigated by TGMM, with corresponding performances surpassing those of the EM algorithm. This demonstrates the effectiveness of TGMM for learning an estimation algorithm that efficiently solves GMM problems.

\begin{figure}[htbp]
    \centering
    \includegraphics[width=0.9\linewidth]{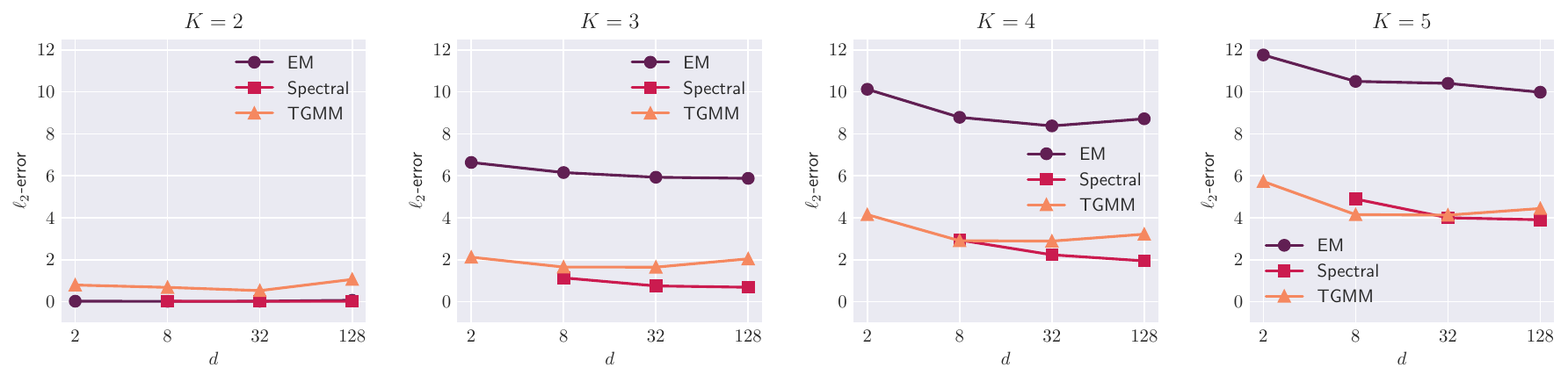}
    \caption{Performance comparison between TGMM and two classical algorithms, reported in $\ell_2$-error.}
    \label{fig:l2 comparison}
\end{figure}

\textbf{RQ2: Robustness}
% In this experiment, we evaluate the generalization ability of transformers to different sequence lengths.
% For training, we uniformly sample the in-context length from $[N/2,N]$, where $N\in\sets{32,64,128}$, but fix it to $128$ during testing.
% Other settings are same.
% The results are shown in \cref{fig:l2 in-context length generalization}, indicating a strong decreasing in $\ell_2$-error as $N$ grows.
To assess the robustness of the learned TGMM, we consider two types of test-time distribution shifts:\\
\textbf{1. Shifts in sample size $N$} Under this scenario, we evaluate the learned TGMM model on tasks with sample size $N^\text{test}$ that are unseen during training.\\
\textbf{2. Shifts in sampling distributions} Under this scenario, we test the learned TGMM model on tasks that are sampled from different sampling distributions that are used during training. Specifically, we use the same training sampling configuration as stated in \cref{sec: exp_setup} and test on the following perturbed sampling scheme, with $\Tilde{\mu}_k = \mu_k + \sigma_p \varepsilon_k$, where $\mu_k \overset{i.i.d.}{\sim} \operatorname{Unif}\paren{[-5,5]^d}$, $\varepsilon_k \overset{i.i.d.}{\sim} \gN\paren{0,I_d}$, $k \in [K]$ and $\sets{\varepsilon_k}_{k\in[K]}$ is independent with $\sets{\mu_k}_{k\in[K]}$.

\begin{figure}[htbp]
    \centering
    \includegraphics[width=0.9\linewidth]{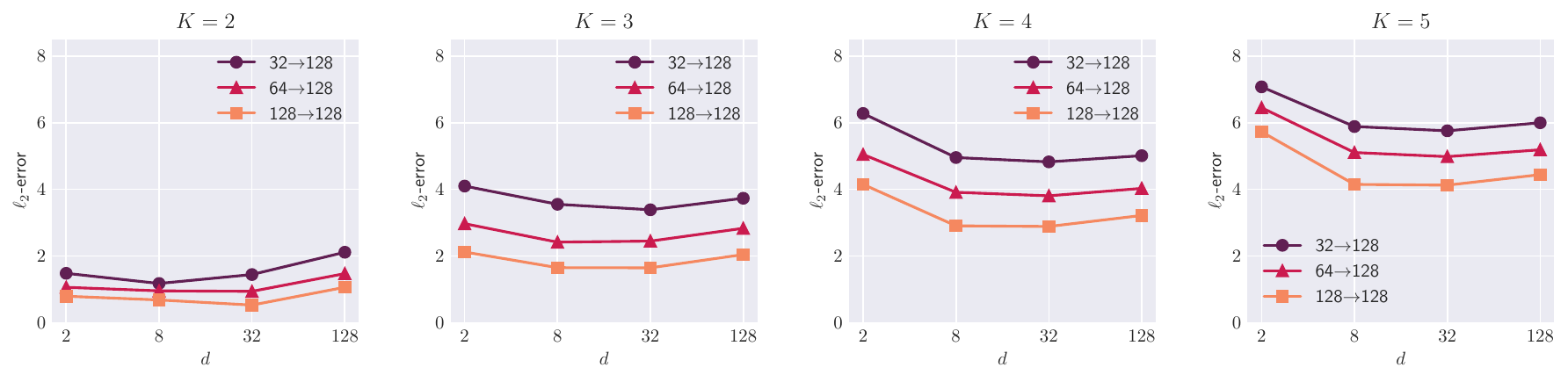}
    \caption{Assessments of TGMM under test-time task distribution shifts I: A line with $N_0^\text{train} \rightarrow N^\text{test}$ draws the performance of a TGMM model trained over tasks with sample size randomly sampled in $[N_0^\text{train} / 2, N_0^\text{train}]$ and evaluated over tasks with sample size $N^\text{test}$. We can view the configuration $128\rightarrow 128$ as an in-distribution test and the rest as out-of-distribution tests.}
    \label{fig:l2 in-context length generalization}
    \vspace{-2mm}
\end{figure}

In \cref{fig:l2 in-context length generalization}, we report the assessments regarding shifts in sample size, where we set $N_\text{test}$ to be $128$ and vary the training configuration $N_0$ to range over $\set{32, 64, 128}$, respectively. The results demonstrate graceful performance degradation of out-of-domain testing performance in comparison to the in-domain performance. To measure performance over shifted test-time sampling distributions, we vary the perturbation scale $\sigma_p \in \{0, 1, \ldots, 10\}$ with problem dimension fixed at $d=8$. The results are illustrated in \cref{fig:l2 ood} along with comparisons to EM and spectral baselines. As shown in the results, with the increase of the perturbation scale, the estimation problem gets much harder. Nevertheless, the learned TGMM can still outperform the EM algorithm when $K > 2$. Both pieces of evidence suggest that our meta-training procedure indeed learns an algorithm instead of overfitting to some training distribution.

% {\bf Out-of-Domain Generalization Capabilities. } 
% We study the out-of-domain(OOD) generalization ability in this experiment. 
% To do this, we use the same data generation setting above in the training, but perturb the mean vector with Gaussian noise in the testing. In more details, for the testing data, the mean vectors are presented by $\Tilde{\mu}_k = \mu_k + \sigma_p \varepsilon_k$, where $\mu_k \overset{i.i.d.}{\sim} \operatorname{Unif}\paren{[-5,5]^d}$, $\varepsilon_k \overset{i.i.d.}{\sim} \gN\paren{0,I_d}$, $k \in [K]$ and $\sets{\varepsilon_k}_{k\in[K]}$ is independent with $\sets{\mu_k}_{k\in[K]}$.
% All experimental settings are the same as the setup section. We compare the TGMM with the classical algorithms on the OOD data. 
% As shown in the results (\cref{fig:l2 ood}), TGMM demonstrates robust generalization to OOD data.

% \begin{figure}[htbp]
%     \centering
%     \includegraphics[width=1.0\linewidth]{figures/v2/ood_v1_cluster_acc.pdf}
%     \caption{The cluster accuracy for the OOD data.}
%     \label{fig:acc ood}
% \end{figure}

\begin{figure}[htbp]
    \centering
    \includegraphics[width=0.9\linewidth]{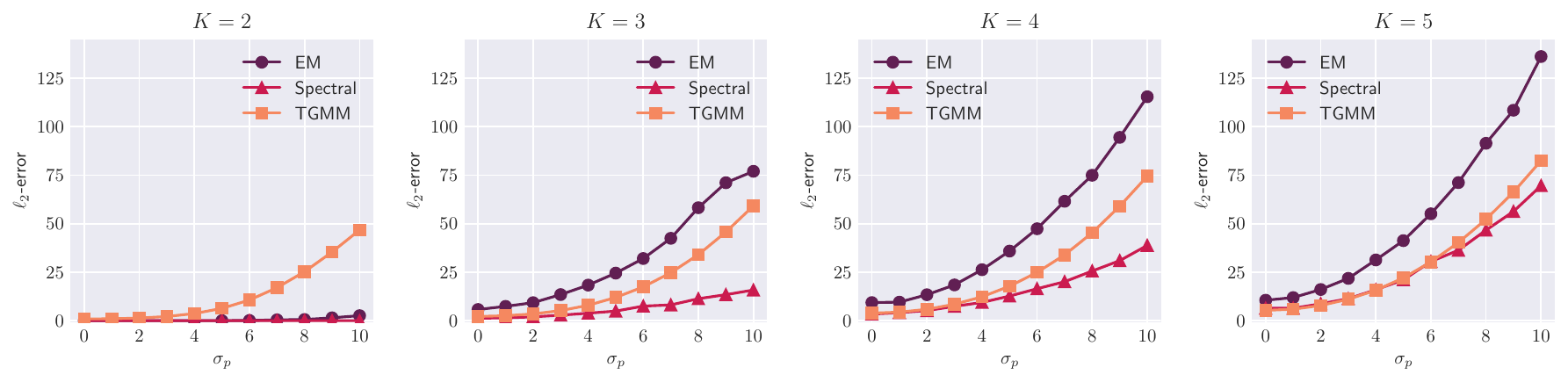}
    \caption{Assessments of TGMM under test-time task distribution shifts II: $\ell_2$-error of estimation when the test-time tasks $\mathcal{T}^\text{test}$ are sampled using a mean vector sampling distribution $p_\mu^\text{test}$ different from the one used during training.}
    \label{fig:l2 ood}
    \vspace{-2mm}
\end{figure}

% {\bf Summary.} 
% Our experiments show that TGMM achieves efficiency competitive with classical algorithms on GMM tasks, closely matching the spectral method in performance.
% % Our experiments demonstrate that TGMM achieves competitive efficiency on GMM tasks compared to classical algorithms. 
% Furthermore, TGMM exhibits increasing expressive power with larger $L$ and $N$, while maintaining strong generalization capabilities for both in-context length adaptation and OOD scenarios.
% It is worth noting that TGMM achieves strong performance with just a few layers, consistent with the intrinsic property of the EM algorithm, which exhibits rapid convergence in early iterations.
% Our results also demonstrate that TGMM achieves strong performance when simultaneously trained and tested on multiple GMM tasks across varying mixture components $K$.

\textbf{RQ3: Flexibility}
Finally, we initiate two studies that extend both the TGMM framework and the (meta) learning problem of solving isotropic GMMs. In our first study, we investigated alternative architectures for the TGMM backbone. Motivated by previous studies\citep{park2024can} that demonstrate the in-context learning capability of linear attention models such as Mamba series\citep{gu2023mamba, dao2024transformers}. We test replacing the backbone of TGMM with a Mamba2\citep{dao2024transformers} model with its detailed specifications and experimental setups listed in \cref{sec: more_exp_setups}. The results are reported in \cref{fig: tf_mamba_comparison}, suggesting that while utilizing Mamba2 as the TGMM backbone still yields non-trivial estimation efficacy, it is in general inferior to transformer backbone under comparable model complexity.

In our second study, we adapted TGMM to be compatible with more sophisticated GMM tasks via relaxing the isotropic assumption. Specifically, we construct anisotropic GMM tasks via equipping it with another scale sampling mechanism $p_\sigma$, where for each task we sample $\sigma \sim \text{softplus}(\tilde{\sigma})$ with $\tilde{\sigma}$ being sampled uniformly from $[-1, 1]^d$. We adjust the output structure of TGMM accordingly so that its outputs can be decoded into both estimates of both mean vectors, mixture probabilities, and scales, which are detailed in \cref{sec: more_exp_setups}. Note that the spectral algorithm does not directly apply to anisotropic setups, limiting its flexibility. Consequently, we compare TGMM with the EM approach and plot results in \cref{fig: anisotropic_baseline_comparison} with the $\ell_2$-error metric accommodating errors from scale estimation. The results demonstrate a similar trend as in evaluations in the isotropic case, showcasing TGMM as a versatile tool in GMM learning problems.

\textbf{Additional experiments} We postpone some further evaluations to \cref{sec: more_exps}, where we present a complete report consisting of more metrics and conduct several ablations on the effects of backbone scales and sample sizes.

\begin{figure}[htbp]
    \centering
    \includegraphics[width=0.9\linewidth]{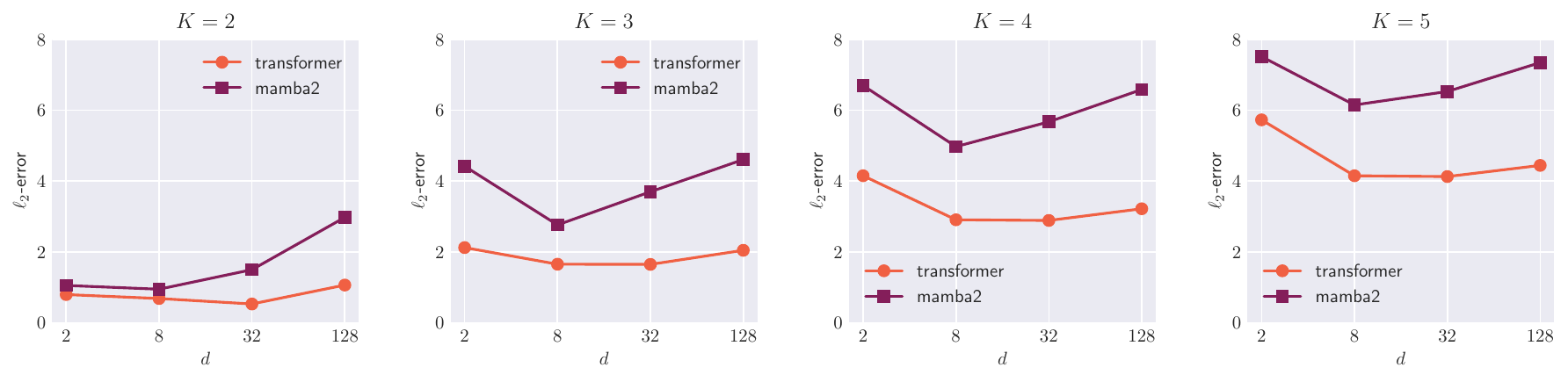}
    \caption{Performance comparisons between TGMM using transformer and Mamba2 as backbone, reported in $\ell_2$-error.}
    \label{fig: tf_mamba_comparison}
\end{figure}
\vspace{-3mm}
\begin{figure}[htbp]
    \centering
    \includegraphics[width=0.9\linewidth]{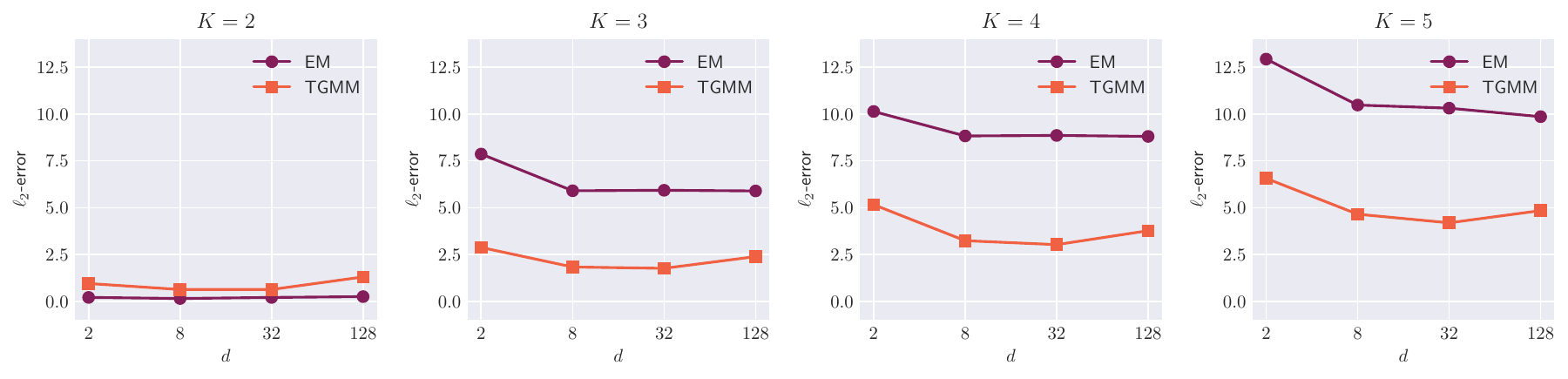}
    \caption{Performance comparison between TGMM and the EM algorithm on anisotropic GMM tasks, reported in $\ell_2$-error.}
    \label{fig: anisotropic_baseline_comparison}
\end{figure}

\vspace{-2mm}

\begin{remark}\label{rmk:concern:fairness}
One might be concerned with the fairness of comparisons between TGMM pre-training and EM/spectral method. 
We would like to point out that the only additional information that TGMM receives during meta-training is the (implicitly provided) distributional information.
The empirical results show that TGMM can generalize beyond the meta-training distribution.
\end{remark}

\vspace{-4mm}

\section{Theoretical understandings}
\label{sec:theory}

\vspace{-4mm}

In this section, we provide some theoretical understandings for the experiments.

\vspace{-3mm}

\subsection{Understanding TGMM}
\label{sec:thm statement}
\vspace{-3mm}
We investigate the expressive power of transformers-for-Gaussian-mixtures(TGMM) as demonstrated in \cref{sec:experiment}. 
% For analytical tractability, we implement Readin as an identity transformation and define Readout to extract targeted matrix elements. 
Our analysis presents two key findings that elucidate the transformer's effectiveness for GMM estimation: 1. Transformer can approximate the EM algorithm; 2. Transformer can approximate the power iteration of cubic tensor. 
% We try to understand the expressibility of TGMM founded in \cref{sec:experiment}. 
% For simplicity, we define the $\operatorname{Readin}$ function as the identity transformation, while the $\operatorname{Readout}$ function extracts specific elements from the given matrix.
% We give two parts of explanations to explain that transformer does well in the GMM problem.

\vspace{-2mm}

{\bf Transformer can approximate the EM algorithm. } 
We show that transformer can efficiently approximate the EM algorithm (\cref{alg:EM1}; see \cref{secapp:alg}) and estimate the parameters of GMM. Moreover, we show that transformer with one backbone can handle tasks with different dimensions and components simultaneously.  The formal statement appears in \cref{secapp:EM} due to space limitations.

\begin{theorem} [Informal]
\label{thm:EM-approx-informal}
There exists a $2L$-layer transformer $\TF_\bTheta$ such that for any $d\leq d_0$, $K\leq K_0$ and task $\gT = \paren{\rmX, \btheta, K}$ satisfying some regular conditions, given suitable embeddings, $\TF_\bTheta$ approximates EM algorithm $L$ steps and estimates $\btheta$ efficiently.
\end{theorem}

{\bf Transformer can approximate power iteration of cubic tensor.} 
Since directly implementing the spectral algorithm with transformers proves prohibitively complex, we instead demonstrate that transformers can effectively approximate its core computational step--the power iteration for cubic tensors (Algorithm 1 in \cite{JMLR:v15:anandkumar14b}; see \cref{secapp:alg}). 
Specifically, we prove that a single-layer transformer can approximate the iteration step:
\begin{equation}
\label{eqn:tensor iteration}
v^{(j+1)} = T\paren{I, v^{(j)}, v^{(j)}},~ j\in \N,
\end{equation}
where $I$ denotes the identity matrix and $T$ represents the given cubic tensor. 
For technical tractability, we assume the attention layer employs a $\textit{ReLU}$ activation function. The formal statement appears in \cref{secapp:tensor} due to space limitations.

\vspace{-2mm}
\begin{theorem}[Informal]
\label{thm:tensor approx}
    There exists a $2L$-layer transformer $\TF_\bTheta$ with ReLU activation such that for any $d\leq d_0$, $T\in \R^{d\times d\times d}$ and $v^{(0)} \in \R^d$,  given suitable embeddings, $\TF_\bTheta$ implements $L$ steps of \cref{eqn:tensor iteration} exactly.
\end{theorem}
\vspace{-2mm}
We give some discussion of the theorems in the following remarks.
\begin{remark}
    (1) \cref{thm:EM-approx-informal} demonstrates that a transformer architecture can approximate the EM algorithm for GMM tasks with varying numbers of components using a single shared set of parameters (i.e., one backbone $\bTheta$). This finding supports the empirical effectiveness of TGMM ({\bf RQ1} in \cref{sec:results-and-findings}). Additionally, \cref{thm:tensor approx} establishes that transformers can approximate power iterations for third-order tensors across different dimensions, further corroborating the model’s ability to generalize across GMMs with varying component counts.

    \vspace{-1mm}
    (2) \cref{thm:EM-approx-informal} holds uniformly over sample sizes $N$ and sampling distributions under mild regularity conditions, aligning with the observed robustness of TGMM ({\bf RQ2} in \cref{sec:results-and-findings}).
    % (1) \cref{thm:EM-approx-informal} shows that transformer can approximate the EM algorithm for GMM tasks with various components using one backbone(i.e., one parameter $\bTheta$), which is consistent with the  effectiveness of TGMM in \cref{sec:results-and-findings}.
    % Also, \cref{thm:tensor approx} shows that transformer can approximate power iteration of cubic tensors with different dimensions, which aligns that TGMM can learn tasks with different components.

    % (2) \cref{thm:EM-approx-informal} is free with sample sample size $N$ and the sampling distributions (only need some regular conditions), which matches the robustness of TGMM in \cref{sec:results-and-findings}. 
\end{remark}
\begin{remark}
    Different "readout" functions are also required to extract task-specific parameters in our theoretical analysis, aligning with the architectural design described in \cref{sec:TGMM-archi}. For further discussion, refer to \cref{rmk:readout} in \cref{sec:TF_EM}.
    % We also need different "Readout" functions to get parameters with different tasks in our theoretical results, which is consistent with our settings in \cref{sec:TGMM-archi}. For details, see \cref{rmk:readout} in \cref{sec:TF_EM}.
\end{remark}
% Add remarks???

% Since the spectral algorithm is too complex to construct with transformer, we give a mediate result that transformer can approximate the key step in \cref{alg:spectral1}-- the power iteration of the cubic tensor (Algorithm 1 in \cite{JMLR:v15:anandkumar14b}, see Appendix A). 
% In detail, we show that one-layer transformer can approximate the iteration step $\bv^{(j+1)} = T\paren{I, \bv^{(j)}, \bv^{(j)}}$, where $I$ is the identity matrix and $T$ is the given cubic tensor.
% For technical convenience, here we assume that the activation function in the attention layer is the $\textit{ReLU}$ function.
% Due to the space constraints, we give the statement and full construction in Appendix C.

\vspace{-4mm}

\subsection{Proof Ideas}
\label{sec:proof idea}
\vspace{-2mm}

{\bf Proof Idea of \cref{thm:EM-approx-informal}.}
We present a brief overview of the proof strategy for \cref{thm:EM-approx-informal}. 
Our approach combines three key components: 
(1) the convergence properties of the population-EM algorithm\citep{pmlr-v125-kwon20a}, 
(2) concentration bounds between population and sample quantities (established via classical empirical process theory), 
and (3) a novel transformer architecture construction. 
The transformer design is specifically motivated by the weighting properties of the \textit{softmax} activation function, which naturally aligns with the EM algorithm's update structure. 
For intuitive understanding, \cref{fig:thm_fig} provides a graphical illustration of this construction.
The full proof is in \cref{secapp:EM}.

\begin{figure}[htbp]
    \centering
    \includegraphics[width=0.9\linewidth]{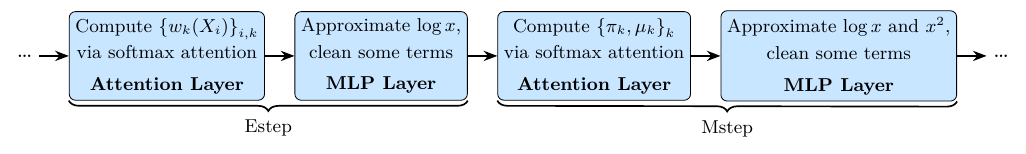}
    \caption{ (Informal version)Transformer Construction for Approximating EM Algorithm Iterations.
    The word "clean" means setting all positions of the corresponding vector to zero.}
    % The {\it pink box} represents the state of tokens, while the {\it blue box} represents the structure of different parts of the network. The term "{\bf clean}" means setting all positions of the corresponding vector to zero.}
    \label{fig:thm_fig}
\end{figure}

\vspace{-2mm}

{\bf Proof Idea of \cref{thm:tensor approx}.} 
To approximate \cref{eqn:tensor iteration}, we perform a two-dimensional computation within a single-layer transformer. The key idea is to leverage the number of attention heads $ M $ to handle one dimension while utilizing the $ Q, K, V $ structure in the attention layer. Specifically, let $ T = (T_{i,j,m})_{i, j, m \in [d]} $ and $ v^{(j)} = (v_i^{(j)})_{i \in [d]} $. Then, \cref{eqn:tensor iteration} can be rewritten as
% \begin{align}  
% \label{eqn:tensor ope}  
%     v^{(j+1)} = \sum_{j,m \in [d]} v_j v_l T_{:,j,m}, \nonumber 
% \end{align}  
$v^{(j+1)} = \sum_{j,m \in [d]} v_j v_l T_{:,j,m}$,
where $T_{:,j,m} = (T_{i,j,m})_{i\in[d]}\in\R^d$.
This operation can be implemented using $ d $ attention heads, where each head processes a dimension of size $ d $ (\cref{fig:thm2_fig}). The complete construction and proof are provided in \cref{secapp:tensor}.
\begin{figure}[htbp]
    \centering
    \includegraphics[width=0.75\linewidth]{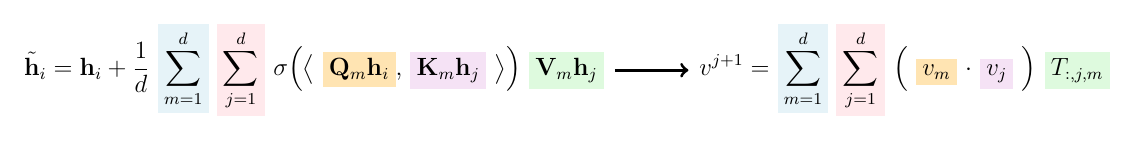}
    \caption{Illustration of implementing \cref{eqn:tensor iteration} via a multi-head attention structure, where colored boxes denote corresponding implementation components. Here $\sigma$
    % $\sigma(x) = \max\sets{x,0}$ 
    denotes the ReLU function.}
    \label{fig:thm2_fig}
\end{figure}

\vspace{-8mm}

\section{Conclusion and discussions}
\label{sec:discussion}

\vspace{-3mm}

In this paper, we investigate the capabilities of transformers in GMM tasks from both theoretical and empirical perspectives. Our work is among the earliest studies to investigate the mechanism of transformers in unsupervised learning settings. 
Our results establish fundamental theoretical guarantees that Transformers can efficiently implement classical algorithms—such as the EM algorithm and spectral methods.
This is consistent with our empirical finding that the performance of our meta-training algorithm can interpolate between EM and the spectral method. It also opens a room for future improvement of attention-based meta-training algorithms in a broader class of unsupervised learning problems.
% Our work suggests several potential extensions for future research. 
We discuss the limitations and potential future research directions in \cref{secapp:limitation}.

% In this paper, we investigate the ability of transformers on the GMM tasks in both theoretical and empirical perspectives. There are some potential extensions of our work. First, our theory focuses on the approximation ability of transformers, leaving the optimization dynamics unexplored. Second, the approximation of the full spectral algorithm (\cref{alg:spectral1}) is challenging and leaves future work. Third, we only study the expressibility of transformers on classical GMM tasks, exploring the ability of transformers on other unsupervised learning tasks is an interesting topic and may need more study.
% We discuss some limitations and future working directions in this section. 

\bibliography{references}

@InProceedings{pmlr-v125-kwon20a,
  title = 	 {The EM Algorithm gives Sample-Optimality for Learning Mixtures of Well-Separated Gaussians},
  author =       {Kwon, Jeongyeol and Caramanis, Constantine},
  booktitle = 	 {Proceedings of Thirty Third Conference on Learning Theory},
  pages = 	 {2425--2487},
  year = 	 {2020},
  volume = 	 {125},
  series = 	 {Proceedings of Machine Learning Research},
  month = 	 {09--12 Jul},
  publisher =    {PMLR},
}

@article{10.1214/21-EJS1905,
author = {Nimrod Segol and Boaz Nadler},
title = {{Improved convergence guarantees for learning Gaussian mixture models by EM and gradient EM}},
volume = {15},
journal = {Electronic Journal of Statistics},
number = {2},
publisher = {Institute of Mathematical Statistics and Bernoulli Society},
pages = {4510 -- 4544},
year = {2021},
}

@article{mei2024unets,
    title={U-Nets as Belief Propagation: Efficient Classification, Denoising, and Diffusion in Generative Hierarchical Models},
    author={Song Mei},
    journal={ArXiv},
    year={2024},
    volume={abs/2404.18444},
    url={https://arxiv.org/abs/2404.18444}, 
}

@inproceedings{
lin2024transformers,
title={Transformers as Decision Makers: Provable In-Context Reinforcement Learning via Supervised Pretraining},
author={Licong Lin and Yu Bai and Song Mei},
booktitle={The Twelfth International Conference on Learning Representations},
year={2024},
url={https://openreview.net/forum?id=yN4Wv17ss3}
}

@inproceedings{bai2023tfstats,
 author = {Bai, Yu and Chen, Fan and Wang, Huan and Xiong, Caiming and Mei, Song},
 booktitle = {Advances in Neural Information Processing Systems},
 pages = {57125--57211},
 publisher = {Curran Associates, Inc.},
 title = {Transformers as Statisticians: Provable In-Context Learning with In-Context Algorithm Selection},
 url = {https://proceedings.neurips.cc/paper_files/paper/2023/file/b2e63e36c57e153b9015fece2352a9f9-Paper-Conference.pdf},
 volume = {36},
 year = {2023}
}

@article{10.1214/19-EJS1660,
author = {Ruofei Zhao and Yuanzhi Li and Yuekai Sun},
title = {{Statistical convergence of the EM algorithm on Gaussian mixture models}},
volume = {14},
journal = {Electronic Journal of Statistics},
number = {1},
publisher = {Institute of Mathematical Statistics and Bernoulli Society},
pages = {632 -- 660},
year = {2020},
doi = {10.1214/19-EJS1660}
}

@article{JMLR:v15:anandkumar14b,
  author  = {Animashree Anandkumar and Rong Ge and Daniel Hsu and Sham M. Kakade and Matus Telgarsky},
  title   = {Tensor Decompositions for Learning Latent Variable Models},
  journal = {Journal of Machine Learning Research},
  year    = {2014},
  volume  = {15},
  number  = {80},
  pages   = {2773--2832},
  url     = {http://jmlr.org/papers/v15/anandkumar14b.html}
}

@inproceedings{10.1145LearningGMM,
author = {Hsu, Daniel and Kakade, Sham M.},
title = {Learning mixtures of spherical gaussians: moment methods and spectral decompositions},
year = {2013},
isbn = {9781450318594},
publisher = {Association for Computing Machinery},
address = {New York, NY, USA},
url = {https://doi.org/10.1145/2422436.2422439},
doi = {10.1145/2422436.2422439},
booktitle = {Proceedings of the 4th Conference on Innovations in Theoretical Computer Science},
pages = {11–20},
numpages = {10},
location = {Berkeley, California, USA},
series = {ITCS '13}
}

@inproceedings{BrownGPT3,
author = {Brown, Tom B. and Mann, Benjamin and Ryder, Nick and Subbiah, Melanie and Kaplan, Jared and Dhariwal, Prafulla and Neelakantan, Arvind and Shyam, Pranav and Sastry, Girish and Askell, Amanda and Agarwal, Sandhini and Herbert-Voss, Ariel and Krueger, Gretchen and Henighan, Tom and Child, Rewon and Ramesh, Aditya and Ziegler, Daniel M. and Wu, Jeffrey and Winter, Clemens and Hesse, Christopher and Chen, Mark and Sigler, Eric and Litwin, Mateusz and Gray, Scott and Chess, Benjamin and Clark, Jack and Berner, Christopher and McCandlish, Sam and Radford, Alec and Sutskever, Ilya and Amodei, Dario},
title = {Language models are few-shot learners},
year = {2020},
isbn = {9781713829546},
publisher = {Curran Associates Inc.},
address = {Red Hook, NY, USA},
booktitle = {Proceedings of the 34th International Conference on Neural Information Processing Systems},
articleno = {159},
numpages = {25},
location = {Vancouver, BC, Canada},
series = {NIPS '20}
}

@InProceedings{von-oswald23a,
  title = 	 {Transformers Learn In-Context by Gradient Descent},
  author =       {Von Oswald, Johannes and Niklasson, Eyvind and Randazzo, Ettore and Sacramento, Joao and Mordvintsev, Alexander and Zhmoginov, Andrey and Vladymyrov, Max},
  booktitle = 	 {Proceedings of the 40th International Conference on Machine Learning},
  pages = 	 {35151--35174},
  year = 	 {2023},
  editor = 	 {Krause, Andreas and Brunskill, Emma and Cho, Kyunghyun and Engelhardt, Barbara and Sabato, Sivan and Scarlett, Jonathan},
  volume = 	 {202},
  series = 	 {Proceedings of Machine Learning Research},
  month = 	 {23--29 Jul},
  publisher =    {PMLR},
  url = 	 {https://proceedings.mlr.press/v202/von-oswald23a.html},
}

@article{he2025transformersversusemalgorithm,
  title={Transformers versus the EM Algorithm in Multi-class Clustering},
  author={He, Yihan and Chen, Hong-Yu and Cao, Yuan and Fan, Jianqing and Liu, Han},
  journal={arXiv preprint arXiv:2502.06007},
  year={2025}
}

@article{jin2025incontextlearningmixturelinear,
  title={Provable In-context Learning for Mixture of Linear Regressions using Transformers},
  author={Jin, Yanhao and Balasubramanian, Krishnakumar and Lai, Lifeng},
  journal={arXiv preprint arXiv:2410.14183},
  year={2024}
}

@inproceedings{giannou2025how,
title={How Well Can Transformers Emulate In-Context Newton's Method?},
author={Angeliki Giannou and Liu Yang and Tianhao Wang and Dimitris Papailiopoulos and Jason D. Lee},
booktitle={The 28th International Conference on Artificial Intelligence and Statistics},
year={2025},
url={https://openreview.net/forum?id=cj5L29VWol}
}

@inproceedings{attention-is-all-you-need,
author = {Vaswani, Ashish and Shazeer, Noam and Parmar, Niki and Uszkoreit, Jakob and Jones, Llion and Gomez, Aidan N. and Kaiser, \L{}ukasz and Polosukhin, Illia},
title = {Attention is all you need},
year = {2017},
isbn = {9781510860964},
publisher = {Curran Associates Inc.},
address = {Red Hook, NY, USA},
booktitle = {Proceedings of the 31st International Conference on Neural Information Processing Systems},
pages = {6000–6010},
numpages = {11},
location = {Long Beach, California, USA},
series = {NIPS'17}
}

@article{EM1977,
 ISSN = {00359246},
 URL = {http://www.jstor.org/stable/2984875},
 author = {A. P. Dempster and N. M. Laird and D. B. Rubin},
 journal = {Journal of the Royal Statistical Society. Series B (Methodological)},
 number = {1},
 pages = {1--38},
 publisher = {[Royal Statistical Society, Oxford University Press]},
 title = {Maximum Likelihood from Incomplete Data via the EM Algorithm},
 urldate = {2025-04-27},
 volume = {39},
 year = {1977}
}

@article{Aitkin01081980,
  title={Mixture models, outliers, and the EM algorithm},
  author={Aitkin, Murray and Wilson, Granville Tunnicliffe},
  journal={Technometrics},
  volume={22},
  number={3},
  pages={325--331},
  year={1980},
  publisher={Taylor \& Francis},
}

@article{GMM1969,
    author = {DAY, N. E.},
    title = {Estimating the components of a mixture of normal distributions},
    journal = {Biometrika},
    volume = {56},
    number = {3},
    pages = {463-474},
    year = {1969},
    month = {12},
    issn = {0006-3444},
    doi = {10.1093/biomet/56.3.463},
    url = {https://doi.org/10.1093/biomet/56.3.463},
    eprint = {https://academic.oup.com/biomet/article-pdf/56/3/463/635460/56-3-463.pdf},
}

@article{gmm2021ACM-TOMM,
author = {Zhang, Yi and Li, Miaomiao and Wang, Siwei and Dai, Sisi and Luo, Lei and Zhu, En and Xu, Huiying and Zhu, Xinzhong and Yao, Chaoyun and Zhou, Haoran},
title = {Gaussian Mixture Model Clustering with Incomplete Data},
year = {2021},
issue_date = {January 2021},
publisher = {Association for Computing Machinery},
address = {New York, NY, USA},
volume = {17},
number = {1s},
issn = {1551-6857},
url = {https://doi.org/10.1145/3408318},
doi = {10.1145/3408318},
journal = {ACM Trans. Multimedia Comput. Commun. Appl.},
month = mar,
articleno = {6},
numpages = {14}
}

@article{ndaoud2022sharp,
  title={Sharp optimal recovery in the two component Gaussian mixture model},
  author={Ndaoud, Mohamed},
  journal={The Annals of Statistics},
  volume={50},
  number={4},
  pages={2096--2126},
  year={2022},
  publisher={Institute of Mathematical Statistics}
}

@article{gribonval2021statistical,
  title={Statistical learning guarantees for compressive clustering and compressive mixture modeling},
  author={Gribonval, R{\'e}mi and Blanchard, Gilles and Keriven, Nicolas and Traonmilin, Yann},
  journal={Mathematical Statistics and Learning},
  volume={3},
  number={2},
  pages={165--257},
  year={2021}
}

@article{yu2021scgmai,
  title={scGMAI: a Gaussian mixture model for clustering single-cell RNA-Seq data based on deep autoencoder},
  author={Yu, Bin and Chen, Chen and Qi, Ren and Zheng, Ruiqing and Skillman-Lawrence, Patrick J and Wang, Xiaolin and Ma, Anjun and Gu, Haiming},
  journal={Briefings in bioinformatics},
  volume={22},
  number={4},
  pages={bbaa316},
  year={2021},
  publisher={Oxford University Press}
}

@article{loffler2021optimality,
  title={Optimality of spectral clustering in the Gaussian mixture model},
  author={L{\"o}ffler, Matthias and Zhang, Anderson Y and Zhou, Harrison H},
  journal={The Annals of Statistics},
  volume={49},
  number={5},
  pages={2506--2530},
  year={2021},
  publisher={Institute of Mathematical Statistics}
}

@inproceedings{dcgmm-nips-2021,
author = {Manduchi, Laura and Chin-Cheong, Kieran and Michel, Holger and Wellmann, Sven and Vogt, Julia E.},
title = {Deep conditional gaussian mixture model for constrained clustering},
year = {2021},
isbn = {9781713845393},
publisher = {Curran Associates Inc.},
address = {Red Hook, NY, USA},
booktitle = {Proceedings of the 35th International Conference on Neural Information Processing Systems},
articleno = {864},
numpages = {12},
series = {NIPS '21}
}

@article{BinYu-EM-2017,
author = {Sivaraman Balakrishnan and Martin J. Wainwright and Bin Yu},
title = {{Statistical guarantees for the EM algorithm: From population to sample-based analysis}},
volume = {45},
journal = {The Annals of Statistics},
number = {1},
publisher = {Institute of Mathematical Statistics},
pages = {77 -- 120},
year = {2017},
doi = {10.1214/16-AOS1435},
URL = {https://doi.org/10.1214/16-AOS1435}
}

@inproceedings{pathak2024transformers,
title={Transformers can optimally learn regression mixture models},
author={Reese Pathak and Rajat Sen and Weihao Kong and Abhimanyu Das},
booktitle={The Twelfth International Conference on Learning Representations},
year={2024},
url={https://openreview.net/forum?id=sLkj91HIZU}
}

@ARTICLE{survey-vision-TF,
  author={Han, Kai and Wang, Yunhe and Chen, Hanting and Chen, Xinghao and Guo, Jianyuan and Liu, Zhenhua and Tang, Yehui and Xiao, An and Xu, Chunjing and Xu, Yixing and Yang, Zhaohui and Zhang, Yiman and Tao, Dacheng},
  journal={IEEE Transactions on Pattern Analysis and Machine Intelligence}, 
  title={A Survey on Vision Transformer}, 
  year={2023},
  volume={45},
  number={1},
  pages={87-110},
  doi={10.1109/TPAMI.2022.3152247}}

@article{TF-in-vision,
author = {Khan, Salman and Naseer, Muzammal and Hayat, Munawar and Zamir, Syed Waqas and Khan, Fahad Shahbaz and Shah, Mubarak},
title = {Transformers in Vision: A Survey},
year = {2022},
issue_date = {January 2022},
publisher = {Association for Computing Machinery},
address = {New York, NY, USA},
volume = {54},
number = {10s},
issn = {0360-0300},
url = {https://doi.org/10.1145/3505244},
doi = {10.1145/3505244},
journal = {ACM Comput. Surv.},
month = sep,
articleno = {200},
numpages = {41}
}

@article{
li2023TF-RL,
title={A Survey on Transformers in Reinforcement Learning},
author={Wenzhe Li and Hao Luo and Zichuan Lin and Chongjie Zhang and Zongqing Lu and Deheng Ye},
journal={Transactions on Machine Learning Research},
issn={2835-8856},
year={2023},
url={https://openreview.net/forum?id=r30yuDPvf2},
note={Survey Certification}
}

@article{hospedales2021meta,
  title={Meta-learning in neural networks: A survey},
  author={Hospedales, Timothy and Antoniou, Antreas and Micaelli, Paul and Storkey, Amos},
  journal={IEEE transactions on pattern analysis and machine intelligence},
  volume={44},
  number={9},
  pages={5149--5169},
  year={2021},
  publisher={IEEE}
}

@article{teh2025empiricalbayestransformers,
  title={Solving Empirical Bayes via Transformers},
  author={Teh, Anzo and Jabbour, Mark and Polyanskiy, Yury},
  journal={arXiv preprint arXiv:2502.09844},
  year={2025},
}

@article{he2025learningspectralmethodstransformers,
  title={Learning spectral methods by transformers},
  author={He, Yihan and Cao, Yuan and Chen, Hong-Yu and Wu, Dennis and Fan, Jianqing and Liu, Han},
  journal={arXiv preprint arXiv:2501.01312},
  year={2025},
}

@inproceedings{akyrek2023what,
title={What learning algorithm is in-context learning? Investigations with linear models},
author={Ekin Aky{\"u}rek and Dale Schuurmans and Jacob Andreas and Tengyu Ma and Denny Zhou},
booktitle={The Eleventh International Conference on Learning Representations },
year={2023},
url={https://openreview.net/forum?id=0g0X4H8yN4I}
}

@book{moitra2018algorithmic,
  title={Algorithmic aspects of machine learning},
  author={Moitra, Ankur},
  year={2018},
  publisher={Cambridge University Press}
}

@inproceedings{chijinEM,
author = {Jin, Chi and Zhang, Yuchen and Balakrishnan, Sivaraman and Wainwright, Martin J. and Jordan, Michael I.},
title = {Local maxima in the likelihood of Gaussian mixture models: structural results and algorithmic consequences},
year = {2016},
isbn = {9781510838819},
publisher = {Curran Associates Inc.},
address = {Red Hook, NY, USA},
booktitle = {Proceedings of the 30th International Conference on Neural Information Processing Systems},
pages = {4123–4131},
numpages = {9},
location = {Barcelona, Spain},
series = {NIPS'16}
}

@inproceedings{lee2019set,
  title={Set transformer: A framework for attention-based permutation-invariant neural networks},
  author={Lee, Juho and Lee, Yoonho and Kim, Jungtaek and Kosiorek, Adam and Choi, Seungjin and Teh, Yee Whye},
  booktitle={International conference on machine learning},
  pages={3744--3753},
  year={2019},
  organization={PMLR}
}

@article{garg2022can,
  title={What can transformers learn in-context? a case study of simple function classes},
  author={Garg, Shivam and Tsipras, Dimitris and Liang, Percy S and Valiant, Gregory},
  journal={Advances in Neural Information Processing Systems},
  volume={35},
  pages={30583--30598},
  year={2022}
}

@article{radford2019language,
  title={Language models are unsupervised multitask learners},
  author={Radford, Alec and Wu, Jeffrey and Child, Rewon and Luan, David and Amodei, Dario and Sutskever, Ilya and others},
  journal={OpenAI blog},
  volume={1},
  number={8},
  pages={9},
  year={2019}
}

@article{loshchilov2017fixing,
  title={Fixing weight decay regularization in adam},
  author={Loshchilov, Ilya and Hutter, Frank and others},
  journal={arXiv preprint arXiv:1711.05101},
  volume={5},
  pages={5},
  year={2017}
}

@inproceedings{dao2024transformers,
  title={Transformers are SSMs: generalized models and efficient algorithms through structured state space duality},
  author={Dao, Tri and Gu, Albert},
  booktitle={Proceedings of the 41st International Conference on Machine Learning},
  pages={10041--10071},
  year={2024}
}

@inproceedings{park2024can,
  title={Can mamba learn how to learn? a comparative study on in-context learning tasks},
  author={Park, Jongho and Park, Jaeseung and Xiong, Zheyang and Lee, Nayoung and Cho, Jaewoong and Oymak, Samet and Lee, Kangwook and Papailiopoulos, Dimitris},
  booktitle={Proceedings of the 41st International Conference on Machine Learning},
  pages={39793--39812},
  year={2024}
}

@article{gu2023mamba,
  title={Mamba: Linear-time sequence modeling with selective state spaces},
  author={Gu, Albert and Dao, Tri},
  journal={arXiv preprint arXiv:2312.00752},
  year={2023}
}

@article{shen2023study,
  title={A study on relu and softmax in transformer},
  author={Shen, Kai and Guo, Junliang and Tan, Xu and Tang, Siliang and Wang, Rui and Bian, Jiang},
  journal={arXiv preprint arXiv:2302.06461},
  year={2023}
}

@book{mclachlan2000finite,
  title={Finite mixture models},
  author={McLachlan, Geoffrey J and Peel, David},
  year={2000},
  publisher={John Wiley \& Sons}
}

@book{titterington_85,
  author = {Titterington, D.M. and Smith, A.F.M. and Makov, U.E.},
  publisher = {Wiley, New York},
  title = {Statistical Analysis of Finite Mixture Distributions},
  year = 1985
}

@article{crouse2016implementing,
  title={On implementing 2D rectangular assignment algorithms},
  author={Crouse, David F},
  journal={IEEE Transactions on Aerospace and Electronic Systems},
  volume={52},
  number={4},
  pages={1679--1696},
  year={2016},
  publisher={IEEE}
}

@article{xie2021explanation,
  title={An explanation of in-context learning as implicit bayesian inference},
  author={Xie, Sang Michael and Raghunathan, Aditi and Liang, Percy and Ma, Tengyu},
  journal={arXiv preprint arXiv:2111.02080},
  year={2021}
}

@inproceedings{li2023transformers,
  title={Transformers as algorithms: Generalization and stability in in-context learning},
  author={Li, Yingcong and Ildiz, Muhammed Emrullah and Papailiopoulos, Dimitris and Oymak, Samet},
  booktitle={International conference on machine learning},
  pages={19565--19594},
  year={2023},
  organization={PMLR}
}

@article{kim2024transformers,
  title={Transformers are minimax optimal nonparametric in-context learners},
  author={Kim, Juno and Nakamaki, Tai and Suzuki, Taiji},
  journal={Advances in Neural Information Processing Systems},
  volume={37},
  pages={106667--106713},
  year={2024}
}

@inproceedings{pytorch,
  author    = {Adam Paszke and
               Sam Gross and
               Francisco Massa and
               Adam Lerer and
               James Bradbury and
               Gregory Chanan and
               Trevor Killeen and
               Zeming Lin and
               Natalia Gimelshein and
               Luca Antiga and
               Alban Desmaison and
               Andreas K{\"{o}}pf and
               Edward Z. Yang and
               Zachary DeVito and
               Martin Raison and
               Alykhan Tejani and
               Sasank Chilamkurthy and
               Benoit Steiner and
               Lu Fang and
               Junjie Bai and
               Soumith Chintala},
  title     = {PyTorch: An Imperative Style, High-Performance Deep Learning Library},
  booktitle = {NeurIPS},
  pages     = {8024--8035},
  year      = {2019}
}

@inproceedings{wolf-etal-2020-transformers,
    title = "Transformers: State-of-the-Art Natural Language Processing",
    author = "Thomas Wolf and Lysandre Debut and Victor Sanh and Julien Chaumond and Clement Delangue and Anthony Moi and Pierric Cistac and Tim Rault and Rémi Louf and Morgan Funtowicz and Joe Davison and Sam Shleifer and Patrick von Platen and Clara Ma and Yacine Jernite and Julien Plu and Canwen Xu and Teven Le Scao and Sylvain Gugger and Mariama Drame and Quentin Lhoest and Alexander M. Rush",
    booktitle = "Proceedings of the 2020 Conference on Empirical Methods in Natural Language Processing: System Demonstrations",
    month = oct,
    year = "2020",
    address = "Online",
    publisher = "Association for Computational Linguistics",
    url = "https://www.aclweb.org/anthology/2020.emnlp-demos.6",
    pages = "38--45"
}

@inproceedings{ICLR2025-densityestLLM,
 author = {Liu, Toni and Boulle, Nicolas and Sarfati, Rapha\"{e}l and Earls, Christopher},
 booktitle = {International Conference on Representation Learning},
 editor = {Y. Yue and A. Garg and N. Peng and F. Sha and R. Yu},
 pages = {22163--22197},
 title = {Density estimation with LLMs: a geometric investigation of in-context learning trajectories},
 url = {https://proceedings.iclr.cc/paper_files/paper/2025/file/380afe1a245a3b2134010620eae88865-Paper-Conference.pdf},
 volume = {2025},
 year = {2025}
}

@inproceedings{
schaeffer2024incontext,
title={In-Context Learning of Energy Functions},
author={Rylan Schaeffer and Mikail Khona and Sanmi Koyejo},
booktitle={ICML 2024 Workshop on In-Context Learning},
year={2024},
url={https://openreview.net/forum?id=9QI3E2iaSD}
}

@article{fakoor2020trade,
  title={Trade: Transformers for density estimation},
  author={Fakoor, Rasool and Chaudhari, Pratik and Mueller, Jonas and Smola, Alexander J},
  journal={arXiv preprint arXiv:2004.02441},
  year={2020}
}
\bibliographystyle{iclr2026_conference}
\newpage
\appendix
\appendix
\addcontentsline{toc}{section}{Appendix} % Add the appendix text to the document TOC
\part{Appendix} % Start the appendix part
\parttoc

\setcounter{algorithm}{0} 
\renewcommand{\thealgorithm}{\thesection.\arabic{algorithm}}
\setcounter{theorem}{0} 
\renewcommand{\thetheorem}{\thesection.\arabic{theorem}}
\numberwithin{lemma}{section}
\setcounter{lemma}{0}
\renewcommand{\thelemma}{\thesection.\arabic{lemma}}
\numberwithin{remark}{section}
\setcounter{remark}{0}
\numberwithin{corollary}{section}
\setcounter{corollary}{0}
% \begin{center}
%     {\large\textbf{Appendix}}
% \end{center}

{\bf Disclosure of LLM usage.} We used LLMs solely for grammatical correction, language polishing, and drafting preliminary code snippets.

{\bf Organization of the Appendix.} 
The appendix is organized as follows. \Cref{secapp:additional-review} provides additional literature review. \Cref{secapp:alg} formally presents the GMM algorithms referenced in \cref{sec:pre}, and \cref{sec:full_notation} details the complete notation for the network architecture. We analyze the parameter efficiency of TGMM in \cref{sec: parameter_efficiency}. \Cref{secapp:limitation} discusses limitations and outlines directions for future work. Rigorous statements and proofs of \cref{thm:EM-approx-informal} and \cref{thm:tensor approx} are provided in \cref{secapp:EM} and \cref{secapp:tensor}, respectively. Additional experimental details are included in \cref{secapp:experiment-app}.
% In \cref{secapp:additional-review}, we give some additional literature review.
% In \cref{secapp:alg}, we formally present the GMM algorithms referenced in \cref{sec:pre}. 
% The full notation of network structure are given in \cref{sec:full_notation}.
% We discuss the parameter efficiency of TGMM in \cref{sec: parameter_efficiency}.
% The discussion and future work directions are presented in \cref{secapp:limitation}.
% Rigorous statements and proofs of \cref{thm:EM-approx-informal} and \cref{thm:tensor approx} are provided in \cref{secapp:EM} and \cref{secapp:tensor}, respectively. 
% Additional experimental details are included in \cref{secapp:experiment-app}.

{\bf Additional notations in the Appendix.} 
% Let $[i:j] \defeq \sets{i,i+1,\cdots,j}$ for $i < j$.
The maximum between two scalars $a, b$ is denoted as $a \vee b$. 
For a vector $v \in \R^d$, let $\norm{v}_\infty \defeq \max_{i\in[d]}|v_i|$ be its infinity norm.
We use $\bzero_d$ to denote the zero vector and $\be_i \in \R^d$ to denote the $i$-th standard unit vector in $\R^d$. 
For a matrix $\bA \in \R^{d_1\times d_2}$, we denote $\norm{\bA}_2\defeq \sup_{\norm{x}_2 = 1} \norm{\bA x}$ as its operator norm.
We use $\widetilde{O}\paren{\cdot}$ to denote $O\paren{\cdot}$ with hidden $\log$ factors.
For clarify, we denote the ground-truth parameters of GMM with a superscript $ ^\ast$, i.e. $\sets{\pi_k^\ast, \mu_k^\ast}_{k\in[K]}$, throughout this appendix.

\section{Literature on Density Estimation using LLMs}
\label{secapp:additional-review}
Recent studies have explored the capabilities of large language models (LLMs) for in-context probability density estimation. For instance, \cite{ICLR2025-densityestLLM} interprets LLM learning as an adaptive form of Kernel Density Estimation, revealing divergent learning trajectories compared to traditional methods. \cite{schaeffer2024incontext} introduces a more general framework for in-context learning by modeling unconstrained energy functions, enabling effective learning even when input and output spaces are mismatched. Meanwhile, \cite{fakoor2020trade} leverages self-attention mechanisms to perform empirical density estimation across heterogeneous data types.
Whereas these efforts prioritize empirical performance in distribution estimation, our paper focuses on the theoretical expressive power of transformers, specifically in the context of GMM estimation. 

\section{Algorithm Details}
\label{secapp:alg}
We state the classical algorithms of GMM mention in \cref{sec:pre} in this section.
    \begin{algorithm}[htbp]
    \caption{EM algorithm for GMM}\label{alg:EM1}
    \begin{algorithmic}[1]
    \Require $\{X_i,i\in [N]\}$, $\btheta^{(0)} = \{\pi_1^{(0)},\mu_1^{(0)},\cdots \pi_K^{(0)},\mu_K^{(0)}\}$
    % \Ensure 
    % \State $y \gets 1$
    % \State $X \gets x$
    \State $j \gets 0$
    \While{not converge}
        \State \textbf{E-step:} $w_{k}^{(j+1)}(X_i) = \frac{\pi_k^{(j)}\phi(X_i; \mu_k^{(j)})}{\sum_{k\in[K]} \pi_k^{(j)}\phi(X_i; \mu_k^{(j)})}$, $i \in [N]$, $k \in [K]$
        \State \textbf{M-step:} $\pi_k^{(j+1)} = \frac{\sum_{i\in[N]}w_{k}^{(j+1)}(X_i)}{N}$, $\mu_k^{(j+1)} = \frac{\sum_{i\in[N]}w_{k}^{(j+1)}(X_i) X_i}{\sum_{i\in[N]}w_{k}^{(j+1)}(X_i)}$, $k \in [K]$
    \State $j \gets j + 1$
    % \If{$N$ is even}
    %     \State $X \gets X \times X$
    %     \State $N \gets \frac{N}{2}$  %\Comment{This is a comment}
    % \ElsIf{$N$ is odd}
    %     \State $y \gets y \times X$
    %     \State $N \gets N - 1$
    % \EndIf
    \EndWhile
    \end{algorithmic}
    \end{algorithm}

    \begin{algorithm}[htbp]
    \caption{Spectral Algorithm for GMM}\label{alg:spectral1}
    \begin{algorithmic}[1]
    \Require $\{X_i,i\in [N]\}$ %, $\btheta^{(0)} = \{\pi_1^{(0)},\mu_1^{(0)},\cdots \pi_K^{(0)},\mu_K^{(0)}\}$
    \State Compute the empirical moments $\hat{M}_2$ and $\hat{M}_3$ by
    \begin{align*}
        \hat{M}_2 &= \frac{1}{N}\sum_{i\in [N]}X_i \otimes X_i-  I_d,~ \\
        \hat{M}_3 &= \frac{1}{N}\sum_{i\in [N]}X_i \otimes X_i \otimes X_i - \frac{1}{N}\sum_{i\in[N], j\in[d]} \paren{X_i \otimes \be_j \otimes \be_j + \be_j \otimes X_i \otimes \be_j + \be_j \otimes \be_j \otimes X_i}
    \end{align*}
    \State Do first $K$-th singular value decomposition(SVD) for $\hat{M}_2$: $\hat{M}_2 \approx UDU^\top$ and let $W = U D^{-1/2}$, $B = U D^{1/2}$ 
    \State Do first $K$-th robust tensor decomposition (Algorithm 1 in \cite{JMLR:v15:anandkumar14b}, see \cref{alg:robust_tensor_power}) for $\Tilde{M}_3 = \hat{M}_3\paren{W,W,W}$:
    \begin{align*}
        \Tilde{M}_3 \approx \sum_{k\in[K]} \lambda_k v_k^{\otimes 3}
    \end{align*}
    \Return $\hat{\pi}_k = \lambda_k^{-2}$, $\hat{\mu}_k = \lambda_k B v_k$, $k \in [K]$. 
    \end{algorithmic}
    \end{algorithm}
    % The robust tensor decomposition(Algorithm 1 in \cite{JMLR:v15:anandkumar14b}) is stated in \cref{alg:robust_tensor_power}.

    \begin{algorithm}[htbp]
    \caption{Robust Tensor Power Method}
    \label{alg:robust_tensor_power}
    \begin{algorithmic}[1]
    \Require symmetric tensor $T \in \mathbb{R}^{d \times d \times d}$, number of iterations $L$, $N$.
    \Ensure the estimated eigenvector/eigenvalue pair; the deflated tensor.
    \For{$\tau = 1$ to $L$}
        \State Draw $v_0^{(\tau)}$ uniformly at random from the unit sphere in $\mathbb{R}^d$.
        \For{$t = 1$ to $N$}
            \State Compute power iteration update:
            \State $v_t^{(\tau)} := \frac{T(I, v_{t - 1}^{(\tau)}, v_{t - 1}^{(\tau)})}{\|T(I, v_{t - 1}^{(\tau)}, v_{t - 1}^{(\tau)})\|}$
        \EndFor
    \EndFor
    \State Let $\tau^* := \arg \max_{\tau \in [L]} \{T(v_N^{(\tau)}, v_N^{(\tau)}, v_N^{(\tau)})\}$.
    \State Do $N$ power iteration updates (line 5) starting from $v_N^{(\tau^*)}$ to obtain $\hat{v}$.
    \State Set $\hat{\lambda} := \tilde{T}(\hat{v}, \hat{v}, \hat{v})$.
    \State \Return the estimated eigenvector/eigenvalue pair $(\hat{v}, \hat{\lambda})$; the deflated tensor $\tilde{T} - \hat{\lambda} \hat{v}^{\otimes 3}$.
    \end{algorithmic}
    \end{algorithm}
% \section{Additional Experimental Results}

\section{Full Notation of Network Architecture}\label{sec:full_notation}

\begin{definition}[Attention layer]
\label{def:attention}
A (self-)attention layer with $M$ heads is denoted as $\Attn_{\bAtt}(\cdot)$ with parameters $\bAtt=\sets{ (\bV_m,\bQ_m,\bK_m)}_{m\in[M]}\subset \R^{D\times D}$. On any input sequence $\bH\in\R^{D\times N}$,
\begin{talign}
% \label{eqn:attention}
    \wt{\bH} = \Attn_{\bAtt}(\bH)\defeq \bH + \sum_{m=1}^M (\bV_m \bH) \operatorname{\mathbf{softmax}}\paren{ (\bK_m\bH)^\top (\bQ_m\bH) } \in \R^{D\times N}, \nonumber
\end{talign}
% where $\sursf: \R^N \to \R^N$ is the softmax function. 
In vector form,
\begin{talign*}
    \wt{\bh}_i = \brac{\Attn_{\bAtt}(\bH)}_i = \bh_i + \sum_{m=1}^M \sum_{j=1}^{N} \brac{\operatorname{\mathbf{softmax}}\paren{ \paren{\paren{\bQ_m\bh_i}^\top \paren{\bK_m\bh_j}}_{j=1}^{N} }}_j  \bV_m\bh_j.
\end{talign*}
Here $\operatorname{\mathbf{softmax}}$ is the activation function defined by
\(    \operatorname{\mathbf{softmax}}\paren{v} = \paren{\frac{\exp(v_1)}{\sum_{i=1}^d \exp(v_i)},\cdots,\frac{\exp(v_d)}{\sum_{i=1}^d \exp(v_i)}}
\)
for $v \in \R^d$.
\end{definition}
The Multilayer Perceptron(MLP) layer is defined as follows.
\begin{definition}[MLP layer]
\label{def:mlp}
A (token-wise) MLP layer with hidden dimension $D'$ is denoted as $\MLP_{\bthetamlp}(\cdot)$ with parameters $\bthetamlp=(\bW_1,\bW_2)\in\R^{D'\times D}\times\R^{D\times D'}$. On any input sequence $\bH\in\R^{D\times N}$, 
\begin{talign*}
    \wt{\bH} = \MLP_{\bthetamlp}(\bH) \defeq \bH + \bW_2\barsig(\bW_1\bH),
\end{talign*}
where $\barsig: \R \to \R$ is the ReLU function. In vector form, we have $\wt{\bh}_i=\bh_i+\bW_2\sigma(\bW_1\bh_i)$. 
\end{definition}
Then we can use the above definitions to define the transformer model.
% We consider a transformer architecture with $L\ge 1$ transformer layers, each consisting of a self-attention layer followed by an MLP layer. 
\begin{definition}[Transformer]
\label{def:tf}
An $L$-layer transformer, denoted as $\TF_\bTF(\cdot)$, is a composition of $L$ self-attention layers each followed by an MLP layer:
\begin{align*}
    \TF_\bTF(\bH) = \MLP_{\bthetamlp^{(L)}}\paren{ \Attn_{\bAtt^{(L)}}\paren{\cdots \MLP_{\bthetamlp^{(1)}}\paren{ \Attn_{\bAtt^{(1)}}\paren{\bH}} }}.
\end{align*}
Here the parameter $\bTF=(\bAtt^{(1:L)},\bthetamlp^{(1:L)})$ consists of the attention layers $\bAtt^{(\ell)}=\sets{ (\bV^{(\ell)}_m,\bQ^{(\ell)}_m,\bK^{(\ell)}_m)}_{m\in[M^{(\ell)}]}\subset \R^{D\times D}$, the MLP layers $\bthetamlp^{(\ell)}=(\bW^{(\ell)}_1,\bW^{(\ell)}_2)\in\R^{D^{(\ell)}\times D}\times \R^{D\times D^{(\ell)}}$.
\end{definition}

\section{On the parameter efficiency of TGMM}\label{sec: parameter_efficiency}
Aside from its backbone, the extra parameters in a TGMM comprises the following:
\begin{description}
    \item[Parameters in the task embedding module] This part has a parameter count of $s \times d_\text{task}$. 
    \item[Parameters in the $\operatorname{Readin}$ layer] This part has a parameter count of $O((d_\text{task} + d) \times D)$.
    \item[Parameters in the $\operatorname{Readout}$ layer] This part has a parameter count of $O(sdD)$, which comprises of parameters from $s$ distinct attention mechanisms.
\end{description}
As $d_\text{task}$ is typically of the order $O(D)$, we conclude that the total extra parameter complexity is of the order $O(sdD)$, which in practice is often way smaller than the parameter complexity of the backbone, i.e., of the order $O(LD^2)$ Meanwhile, a naive implementation of adapting transformer architecture to solve $s$ distinct GMM tasks require a different transformer backbone. As the complexity of backbone often dominate those of extra components, the TGMM implementation can reduce the parameter complexity by an (approximate) factor of $1/s$ in practice.

\section{Limitations and future work directions}
\label{secapp:limitation}
First, while our theoretical analysis focuses on the approximation ability of transformers, the optimization dynamics remain unexplored.
This is a common theoretical challenge in ICL literature; see \cite{bai2023tfstats, lin2024transformers, giannou2025how}.
Second, approximating the full spectral algorithm (\cref{alg:spectral1}; see \cref{secapp:alg}) presents a significant challenge, which we leave for future work. Third, our study is limited to the expressivity of transformers on classical GMM tasks; exploring their performance on other unsupervised learning tasks is an interesting direction that warrants further investigation.

\section{Formal statement of \texorpdfstring{\cref{thm:EM-approx-informal}}{Theorem 1} and proofs}
\label{secapp:EM}

For analytical tractability, we implement $\operatorname{Readin}$ as an identity transformation and define $\operatorname{Readout}$ to extract targeted matrix elements hence they are both fixed functions. 
Actually, we also need "Readout" functions to get the estimated parameters for different tasks, see \cref{rmk:readout}.
% Thus we can use $\bTheta$ to denote the parameters of transformers in the appendix. 
To theoretical convenience, we use the following norm of transformers, which differs slightly from the definition in \cite{bai2023tfstats}.
\begin{align}
% \label{eqn:TF-norm}
    \nrmp{\bTheta}\defeq \max_{\ell\in[L]} \set{  
    \max_{m\in[M]} \set{\norm{\bQ_m^\lth}_2, \norm{\bK_m^\lth}_2, \norm{\bV_m^\lth}_2} +
    \norm{\bW_1^\lth}_2 + \norm{\bW_2^\lth}_2
    } . \nonumber
\end{align}
Then the transformer class can be defined as 
\begin{align}
% \label{eqn:TF-func class}  
\gF \defeq \gF(L, D, D^\prime, M, B_\bTheta) = \set{\TF_\bTheta,  \nrmp{\bTheta} \leq B_\bTheta, D^{(\ell)} \leq D^\prime, M^{(\ell)} \leq M, \ell \in[L]}. \nonumber
\end{align}

\subsection{Formal statement of \texorpdfstring{\cref{thm:EM-approx-informal}}{Theorem 1}}
\label{secapp:state_EM}

First, we introduce some notations. 
We define $\pi_{\min} = \min_i \pi_i^\ast$, $\rho_\pi = \max_i \pi_i^* / \min_i \pi_i^*$. We use $R_{ij} = \|\mu_i^\ast- \mu_j^\ast\|$ to denote the pairwise distance between components and $R_{\min} = \min_{i\neq j} R_{ij}$, $R_{max} = \paren{\max_{i\neq j} R_{ij}} \vee \paren{\max_{i\in[K]}\norm{\mu_i^\ast}} $. 
Without the loss of generality, we assume that $R_{\max} \geq 1$.
For dimension and components adaptation, we assume $d \leq d_0$ and $K \leq K_0$. 
Since in practice the sample size $N$ is much larger than the number of components $K$, we assume that $N$ is divisible by $K$, i.e. $ N/K \in \N$. 
Otherwise, we only consider the first $K \lfloor N/K\rfloor $ samples and drop the others.
We encode $\rmX = \{X_i\}_{i=1}^{N}$ into an input sequence $\bH$ as the following:
\begin{align}
\label{eqn:input-format}
&\bH = \begin{bmatrix}
% 1 & 1 & \dots & 1  \\
\overline{X}_1 & \overline{X}_2 & \dots & \overline{X}_N  \\
\bp_1 & \bp_2 & \dots & \bp_N 
\end{bmatrix} \in \R^{D\times N}, ~
\bp_i = \begin{bmatrix}
    \overline{\btheta}_i\\\br_i
\end{bmatrix},\\
&\overline{\btheta}_i = 
\begin{bmatrix}
   \overline{\bpi}_{\log} \\ \overline{\mu}_{i \% K} \\
   c_{i \% K} \\ \bzero_{3K_0}
   % \overline{\bgamma}_{i} \\ \overline{\bgamma}_{i\log}\\ \overline{\bpi}
\end{bmatrix} \in \R^{d_0 + 4K_0 + 1},~
\br_i = 
\begin{bmatrix}
   \bzero_{\Tilde{D}} \\ 1 \\ \be_{i \% K} 
\end{bmatrix} \in \R^{D-(2d_0+3K_0+1)},
\end{align}
where $\overline{X}_i = [{X}_i^\top, \bzero_{d_0-d}^\top]^\top$, $\overline{\bpi}_{\log} = [ \bpi_{\log}^\top,\bzero_{K_0-K}^\top]^\top$, $\overline{\mu}_{i \% K}=[\mu_{i \% K}^\top$, $\bzero_{d_0-d}^\top]^\top$, 
% $\overline{\bgamma}_{i} = [\bgamma_{i}^\top,\bzero_{K_0-K}^\top]^\top$, $\overline{\bgamma}_{i \log} = [\bgamma_{i\log}^\top,\bzero_{K_0-K}^\top]^\top$, $\overline{\bpi} = [ \bpi^\top,\bzero_{K_0-K}^\top]^\top$, 
$c_{i\%K} \in \R$ and $\be_{i \% K} \in \R^{K_0}$ denotes the ${i \% K}$-th standard unit vector.
To match the dimension, $\Tilde{D} = D-(2d_0+5K_0+2)$.
We choose $D = O(d_0 + K_0)$ to get the encoding above. 
For the initialization, we choose $\bpi_{\log}= \log \bpi^{(0)}$, $\mu_i = \mu_i^{(0)}$, $c_i = \|\mu_i^{(0)}\|_2^2$. 
% and ${\bpi} = \bgamma_{i} = \bgamma_{i \log} = \bzero_K$, $i\in [K]$. 
 
% \begin{align}
% \label{eqn:input-format}
% \bH = \begin{bmatrix}
% % 1 & 1 & \dots & 1  \\
% \overline{X}_1 & \overline{X}_2 & \dots & \overline{X}_N  \\
% \bp_1 & \bp_2 & \dots & \bp_N 
% \end{bmatrix} \in \R^{D\times N}, ~
% \bp_i = \begin{bmatrix}
%     \overline{\btheta}_i\\\br_i
% \end{bmatrix},~
% \overline{\btheta}_i = 
% \begin{bmatrix}
%    \overline{\bpi}_{\log} \\ \overline{\mu}_{i \% K} \\
%    c_{i \% K} \\ \overline{\bgamma}_{i} \\ \overline{\bgamma}_{i\log}\\ \overline{\bpi}
% \end{bmatrix} \in \R^{d_0 + 4K_0 + 1},~
% \br_i = 
% \begin{bmatrix}
%    \bzero_{\Tilde{D}} \\ 1 \\ \be_{i \% K} 
% \end{bmatrix} \in \R^{D-(2d_0+3K_0+1)},
% \end{align}
% where $\overline{X}_i = [{X}_i^\top, \bzero_{d_0-d}^\top]^\top$, $\overline{\bpi}_{\log} = [ \bpi_{\log}^\top,\bzero_{K_0-K}^\top]^\top$, $\overline{\mu}_{i \% K}=[\mu_{i \% K}^\top$, $\bzero_{d_0-d}^\top]^\top$, $\overline{\bgamma}_{i} = [\bgamma_{i}^\top,\bzero_{K_0-K}^\top]^\top$, $\overline{\bgamma}_{i \log} = [\bgamma_{i\log}^\top,\bzero_{K_0-K}^\top]^\top$, $\overline{\bpi} = [ \bpi^\top,\bzero_{K_0-K}^\top]^\top$, $c_{i\%K} \in \R$ and $\be_{i \% K} \in \R^{K_0}$ denotes the ${i \% K}$-th standard unit vector.
% To match the dimension, $\Tilde{D} = D-(2d_0+5K_0+2)$.
% We choose $D = O(d_0 + K_0)$ to get the encoding above. 
% For the initialization, we choose $\bpi_{\log}= \log \bpi^{(0)}$, $\mu_i = \mu_i^{(0)}$, $c_i = \|\mu_i^{(0)}\|_2^2$ and ${\bpi} = \bgamma_{i} = \bgamma_{i \log} = \bzero_K$, $i\in [K]$. 

To guarantee convergence of the EM algorithm, we adopt the following assumption for the initialization parameters, consistent with the approach in \cite{pmlr-v125-kwon20a}.
\begin{enumerate}[label=(A\arabic*),start=1, leftmargin=*]
    \item \label{ass:A1} Suppose the GMM has parameters $\{(\pi_j^\ast, \mu_j^\ast): j \in [K]\}$ such that
    \begin{equation}
        % \label{eq:pop_separation_condition001}
        R_{\min} \ge C  \cdot \sqrt{\log(\rho_\pi K)},\nonumber
    \end{equation}
    and suppose the mean initialization $\mu_1^{(0)}, ..., \mu_K^{(0)}$ satisfies
    \begin{equation}
        % \label{eq:pop_initialization_means001}
        \forall i \in [K], \norm{\mu_i^{(0)} - \mu_i^\ast}\le \frac{R_{\min}}{16}.\nonumber
    \end{equation}
    Also, suppose the mixing weights are initialized such that
    \begin{align}
        % \label{eq:pop_initialization_weights001}
        \forall i \in [K], \left|\pi_i^{(0)} - \pi_i^\ast\right| \le \pi_i / 2. \nonumber
    \end{align} 
\end{enumerate}
We denote the output of the transformer $\TF_\bTheta$ as $\btheta^{\tt TF} \defeq \sets{\pi_1^{\tt TF}, {\mu}_1^{\tt TF}, {\pi}_2^{\tt TF},\mu_2^{\tt TF}, \cdots \pi_K^{\tt TF},\mu_K^{\tt TF}}$ and assume matched indices.
Define 
\begin{equation}
    \label{eq:definte_Dm_main}
    D_\bTheta^{\tt TF} \defeq \max_{i\in [K]} \set{\norm{\mu_i^{\tt TF} - \mu_i^\ast}\vee \left(\left|\pi_i^{\tt TF} - \pi_i^\ast\right| / \pi_i \right)}. \nonumber
\end{equation}
% \begin{equation}
%     \label{eq:definte_Dm+}
%     D_m^+ = \max_{i\in [K]} \left(\|\mu_i^+ - \mu_i^*\|\vee |\pi_i^+ - \pi_i^*| / \pi_i^* \right).\nonumber
% \end{equation}
Now we propose the theorem that transformer can efficient approximate the EM Algorithm (\cref{alg:EM1}), which is  the formal version of \cref{thm:EM-approx-informal}.

\begin{theorem}
\label{thm: EM approx}
    Fix $0< \delta, \beta < 1$ and $1/2<a<1$. Suppose there exists a sufficiently large  universal constant $C \geq 128$ for which assumption \ref{ass:A1} holds. If $N$ is suffcient large and $ \varepsilon \leq 1/\paren{100 K_0}$ sufficient small such that
    \begin{align*}
    % \varepsilon_{unif} := 
    &\frac{\tilde{c}_1}{(1-a)\pi_{\min}}\sqrt{\frac{R_{\max}(R_{\max} \vee d)\log\paren{\frac{24K}{\delta}}}{N}} \\
    &\qquad + \tilde{c}_2 \left(R_{\max} + d\left(1+\sqrt{\frac{2\log(\frac{4N}{\delta})}{d}}\right)\right) N \varepsilon  <\frac{1}{2} \left(a - \frac{1}{2}\right),
    \end{align*}
    and 
    \begin{align*}
    \epsilon(N,\varepsilon,\delta,a) &\defeq 
    \frac{\tilde{c}_3}{\left(1-a\right)\pi_{\min}} \sqrt{\frac{Kd\log(\frac{\tilde{C}N}{\delta})}{N}} \\ 
    &\qquad +\tilde{c}_4 \paren{\frac{1}{\pi_{\min}} + N\paren{R_{\max} + d+\sqrt{2d\log\left(\frac{4N}{\delta}\right)}}}  \varepsilon \\
    &<a(1-\beta),
    \end{align*}
    hold, where $\tilde{c}_1$-$\tilde{c}_4$ are universal constants, $\tilde{C} = 288 K^2(\sqrt{d} + 2R_{\max} + \frac{1}{1-a})^2$. Then there exists a $2(L+1)$-layer transformer $\TF_\bTheta$ such that
    \begin{align}
    \label{eqn:main EM approx}
        D_\bTheta^{\tt TF} \leq a \beta^L + \frac{1}{1-\beta} \epsilon(N,\varepsilon,\delta,a)
    \end{align}
    holds with probability at least $1-\delta$. Moreover, $\TF_\bTheta$ falls within the class $\gF$ with parameters satisfying:
 \begin{align*}
     &D = O(d_0 + K_0), D^\prime \leq \tilde{O}\paren{K_0 R_{\max}\paren{R_{\max} + d_0} \varepsilon^{-1}},\\
     &M = O(1),\\
     &\log B_{\bTheta} \leq \tilde{O}\paren{ K_0 R_{\max}\paren{R_{\max} + d_0}}. 
     % \log B_{\bTheta} \leq O\paren{ \log\paren{\frac{1}{(1-a)\pi_{\min}}} \vee K_0 R_{\max}\paren{R_{\max} + d_0}} 
     % D^\prime \leq O\paren{K_0 R_{\max}\paren{R_{\max} + d\left(1+\sqrt{\frac{2\log(\frac{2n}{\delta})}{d}}\right) }}
 \end{align*}
 Notably, \cref{eqn:main EM approx} holds for all tasks satisfying $d\leq d_0$ and $K\leq K_0$, where the parameters of transformer $\bTheta$ remains fixed across different tasks $\gT$.
\end{theorem}
\begin{remark}
    From \cref{thm: EM approx}, if we take $\varepsilon = \Tilde{O}\paren{N^{-3/2}d^{-1/2}}$ and $L = O\paren{\log N}$, then we have
    \begin{align*}
        D_\bTheta^{\tt TF} \leq \Tilde{O}\paren{\sqrt{\frac{d}{N}}},
    \end{align*}
    which matches the canonical parametric error rate.
\end{remark}
\begin{remark}
    We give some explanations for the notations in \cref{thm: EM approx}. Define 
    \begin{align*}
        D_j^{\rm pEM}\defeq \max_{i\in [K]} \set{\norm{\Tilde{\mu}_i^{(j)} - \mu_i}\vee \left(\left|\Tilde{\pi}_i^{(j)} - \pi_i\right| / \pi_i \right)},
    \end{align*}
    where $\sets{\Tilde{\mu}_i^{(j)}, \Tilde{\pi}_i^{(j)}}_{i\in[K]}$ are the parameters obtained at the $j$-th iteration of the population-EM algorithm (see \cref{secapp:pop-EM} for details).
    In the convergence analysis of the population-EM algorithm \citep{pmlr-v125-kwon20a}, it is shown that after the first iteration, the parameters lie in a small neighborhood of the true parameters with high probability (i.e., $D_1^{\rm pEM} \leq a$ for some $1/2 \leq a < 1$). 
    Furthermore, the authors prove that the algorithm achieves linear convergence (i.e., $D_{j+1}^{\rm pEM} \leq \beta D_j^{\rm pEM}$ for $j \in \N_+$ and some $0 < \beta < 1$) with high probability if $D_1^{\rm pEM} \leq a$ holds.
    Following their notations, here $a$ represents the radius of the neighborhood after the first iteration, while $\beta$ is the linear convergence rate parameter. 
    Finally, $\varepsilon$ controls the approximation error of the transformer.
\end{remark}

\subsection{Construction of transformer architecture and formal version of \texorpdfstring{\cref{fig:thm_fig}}{Figure 7}}
\label{sec:TF_EM}
In this section, we give the transformer architecture construction in \cref{thm: EM approx}. We denote $w_{ij} = w_j(X_i), i\in [N], k\in[K]$ in this subsection for simplicity. Recall that we have assumed that $d \leq d_0$, $K \leq K_0$ and $N$ is divisible by $K$($ N/K \in \N$).
We first restate the encoding formulas in \cref{eqn:input-format}:
\begin{align*}
\label{eqn:refined input-format}
&\bH = \begin{bmatrix}
% 1 & 1 & \dots & 1  \\
\overline{X}_1 & \overline{X}_2 & \dots & \overline{X}_N  \\
\overline{\btheta}_1 & \overline{\btheta}_2 & \dots & \overline{\btheta}_N  \\
\bp_1 & \bp_2 & \dots & \bp_N 
\end{bmatrix} \in \R^{D\times N}, \\
&\overline{\btheta}_i = 
\begin{bmatrix}
   \overline{\bpi}_{\log} \\ \overline{\mu}_{i \% K} \\
   c_{i \% K} \\ \overline{\bw}_{i} \\ \overline{\bw}_{i\log}\\ \overline{\bpi}
\end{bmatrix} \in \R^{d_0 + 4K_0 + 1},~
\bp_i \defeq 
\begin{bmatrix}
   \bzero_{D-(2d_0+5K_0+2)} \\ 1 \\ \be_{i \% K} 
\end{bmatrix} \in \R^{D-(2d_0+3K_0+1)}, 
\end{align*}
where $\overline{X}_i = [{X}_i^\top, \bzero_{d_0-d}^\top]^\top$, $\overline{\bpi}_{\log} = [ \bpi_{\log}^\top,\bzero_{K_0-K}^\top]^\top$, $\overline{\mu}_{i \% K}=[\mu_{i \% K}^\top$, $\bzero_{d_0-d}^\top]^\top$, 
$\overline{\bw}_{i} = [\bw_{i}^\top,\bzero_{K_0-K}^\top]^\top$, $\overline{\bw}_{i \log} = [\bw_{i\log}^\top,\bzero_{K_0-K}^\top]^\top$, $\overline{\bpi} = [ \bpi^\top,\bzero_{K_0-K}^\top]^\top$, 
$c_{i\%K} \in \R$ and $\be_{i \% K} \in \R^{K_0}$ denotes the ${i \% K}$-th standard unit vector.
For the initialization, we choose $\bpi_{\log}= \log \bpi^{(0)}$, $\mu_i = \mu_i^{(0)}$, $c_i = \|\mu_i^{(0)}\|_2^2$. 
and ${\bpi} = \bw_{i} = \bw_{i \log} = \bzero_K$, $i\in [K]$.
Finally, take $\bH^{(0)} = \bH$ which is defined in \cref{eqn:refined input-format}.

Then in E-step, we consider the following attention structures:
we define matrices $\bQ^{(1)}$, $\bK^{(1)}$, $\bV^{(1)}$,  such that
\begin{align*}
    \bQ^{(1)}\bh_i^{(0)} = \begin{bmatrix} \overline{X}_i\\ \overline{\bpi}_{\log} \\ 1\\ \bzero \end{bmatrix}, \quad
    \bK^{(1)}\bh_j^{(0)} = \begin{bmatrix} -\overline{\mu}_{j\%K} \\ \be_{j \% K} \\ \frac{1}{2} c_{j\%K} \\ \bzero \end{bmatrix}, \quad
    \bV^{(1)}\bh_j^{(0)} = \begin{bmatrix} \bzero_{d_0} \\ \bzero_{K_0}\\\bzero_{d_0 + 1} \\ \be_{j \% K} \\ \bzero_{D-(2d_0+2K_0+1)} \end{bmatrix},
\end{align*}
and use the standard softmax attention, thus 
\begin{align*}
    \wt{\bh}_i^{(1)} &= \brac{\Attn_{\bAtt^{(1)}}(\bH^{(0)})}_{:,i} \\
    &= \bh_i^{(0)} +  \sum_{j=1}^{N} \brac{\operatorname{\mathbf{softmax}}\paren{ \paren{\paren{\bQ^{(1)}\bh_i^{(0)}}^\top \paren{\bK^{(1)}\bh_j^{(0)}}}_{j=1}^{N} }}_j \cdot \bV^{(1)}\bh_j^{(0)} \\
    & = \bh_i^{(0)} + \sum_{j=1}^{N} \frac{\alpha_{j\%K}^{(0)}\exp\paren{-X_i^\top\mu_{j\%K} + \frac{1}{2} \mu_{j\%K}^\top\mu_{j\%K}}}{B \sum_{k=1}^{K} \alpha_k^{(0)}\exp\paren{-X_i^\top\mu_{k} + \frac{1}{2} \mu_{k}^\top\mu_{k}}} \cdot \bV^{(1)}\bh_j^{(0)} \\
    & = \bh_i^{(0)} + \frac{1}{B}\sum_{j=1}^{N} \hat{w}_{i\,j\%K}^{(1)}\bV^{(1)} \bh_j^{(0)} \\
    & = \bh_i^{(0)} + \sum_{j=1}^{K} \hat{w}_{ij}^{(1)}\bV^{(1)} \bh_j^{(0)}  \\
    & = \bh_i^{(0)} + \begin{bmatrix} \bzero_{d_0} \\ \bzero_{K_0}\\\bzero_{d_0+1} \\  \overline{\hat{\bw}}_i^{(1)}
    \\ \bzero_{D-(2d_0+2K_0+1)} \end{bmatrix}, ~i \in [N].
\end{align*}
where $\overline{\hat{\bw}}_i^{(1)} = \paren{\hat{w}_{i1}^{(1)},\hat{w}_{i2}^{(1)}, \cdots, \hat{w}_{iK}^{(1)},0,\cdots,0}^{\top} \in \R^{K_0}$.

% To improve the MLP layer \cref{mlp_E} 
Then we use a two-layer MLP to approximate $\log x$ and clean all  $\overline{\bpi}_{\log}$ , $\overline{\mu}_{i\%K}$ and $c_{i\%K}$, which is 
\begin{align}
    \bh_i^{(1)} = \MLP_{\bthetamlp^{(1)}}\paren{\wt{\bh}_i^{(1)}} = \begin{bmatrix} \overline{X}_i \\ \bzero_{K_0}\\\bzero_{d_0+1} \\  {\overline{\hat{\bw}}}_i^{(1)} \\{\overline{\hat{\bw}}}_{i \log}^{(1)} \\ \bzero_{K_0}\\\bzero_{D-(2d_0+5K_0+2)} \\ 1 \\ \be_{i \% K}  \end{bmatrix}~, i \in [N],\nonumber
    % \label{refined mlp_E}
\end{align}
where ${\overline{\hat{\bw}}}_{i \log}^{(1)} = \widehat{\log}{\overline{\hat{\bw}}}_i^{(1)}$.
Notice that although $\log x$ is not defined at $0$, the MLP approximation is well defined with some value which we do not care because we will not use it in the M-step.
Similarly, for any $\ell \%2 = 1$, $\ell \in \mathbb{N}_+$ we have
\begin{align*}
     \bh_i^{(\ell)} = \MLP_{\bthetamlp^{(\ell\%2)}}\paren{\brac{\Attn_{\bAtt^{(\ell\%2)}}(\bH^{(\ell - 1)})}_{:,i}} = \begin{bmatrix} \overline{X}_i \\ \bzero_{K_0}\\\bzero_{d_0+1} \\ {\overline{\hat{\bw}}}_i^{((\ell+1)/2)} \\{\overline{\hat{\bw}}}_{i \log}^{((\ell+1)/2)} \\ \bzero_{K_0}\\\bzero_{D-(2d_0+5K_0+2)} \\ 1 \\ \be_{i \% K}  \end{bmatrix}, ~ i \in [N],
\end{align*}
where ${\overline{\hat{\bw}}}_{i \log}^{((\ell + 1)/2)} = \widehat{\log}{\overline{\hat{\bw}}}_i^{((\ell + 1)/2)}$.

In M-step, we consider the following attention structures: we similarly define matrices $\bQ_m^{(2)}$, $\bK_m^{(2)}$, $\bV_m^{(2)}$, $m = 1,2$  such that 
\begin{align*}
    \bQ_1^{(2)}\bh_j^{(1)} = \begin{bmatrix} \be_{j \% K} \\  \bzero \end{bmatrix}, \quad
    \bK_1^{(2)}\bh_i^{(1)} = \begin{bmatrix} {\overline{\hat{\bw}}}_{i \log}^{(1)} \\ \bzero \end{bmatrix}, \quad
    \bV_1^{(2)}\bh_i^{(1)} = \begin{bmatrix} \bzero_{d_0} \\ \bzero_{K_0}\\ \overline{X}_i\\0 \\ \bzero_{K_0} \\\bzero_{K_0} \\ \bzero_{D-(2d_0+3K_0+1)} \end{bmatrix},
\end{align*}
and
\begin{align*}
    \bQ_2^{(2)}\bh_j^{(1)} = \bzero, \quad
    \bK_2^{(2)}\bh_i^{(1)} = \bzero, \quad
    \bV_2^{(2)}\bh_i^{(1)} = \begin{bmatrix} \bzero_{d_0} \\ \bzero_{K_0}\\ \bzero_{d_0+1} \\ \bzero_{K_0} \\ \bzero_{K_0}\\{\overline{\hat{\bw}}}_{i}^{(1)}\\ \bzero_{D-(2d_0+4K_0+1)} \end{bmatrix},
\end{align*}

Then we get
{\allowdisplaybreaks
\begin{align*}
    \wt{\bh}_j^{(2)} &= \brac{\Attn_{\bAtt^{(2)}}(\bH^{(1)})}_{:,j} \\
    &= \bh_j^{(1)} +  \sum_{m=1}^{2} \sum_{i=1}^{N} \brac{\operatorname{\mathbf{softmax}}\paren{ \paren{\paren{\bQ_m^{(2)}\bh_i^{(1)}}^\top \paren{\bK_m^{(2)}\bh_j^{(1)}}}_{j=1}^{N} }}_j \cdot \bV_m^{(2)}\bh_i^{(1)} \nonumber\\
    & = \bh_j^{(1)} + \sum_{i=1}^{N} \frac{\hat{w}_{i\, j\%K}^{(1)}}{\sum_{i=1}^{N} \hat{w}_{i\, j\%K}^{(1)}} \cdot \bV_1^{(2)}\bh_i^{(1)} + \sum_{i=1}^{N} \frac{1}{N} \cdot \bV_2^{(2)}\bh_i^{(1)}\nonumber\\
    & = \bh_j^{(1)} +  \begin{bmatrix} \bzero_{d_0} \\ \bzero_{K_0}\\ \overline{\hat{\mu}}_{j\%K}^{(1)} \\0\\ \bzero_{K_0} \\ \bzero_{K_0} \\ \bzero_{D-(2d_0+3K_0+1)} \end{bmatrix} + \begin{bmatrix} \bzero_{d_0} \\ \bzero_{K_0}\\\bzero_{d_0+1} \\ \bzero_{K_0} \\ \bzero_{K_0} \\\overline{\hat{\bpi}}^{(1)}\\ \bzero_{D-(2d_0+4K_0+1)} \end{bmatrix},\nonumber \\
    & = \bh_j^{(1)} + \begin{bmatrix} \bzero_{d_0} \\ \bzero_{K_0}\\\overline{\hat{\mu}}_{j\%K}^{(1)}\\0  \\ \bzero_{K_0} \\ \bzero_{K_0} \\ \overline{\hat{\bpi}}^{(1)} \\\bzero_{D-(2d_0+4K_0+1)} \end{bmatrix}, ~j \in [N] . \nonumber
\end{align*}
}

Similarly, we use a two-layer MLP to approximate $\log x$, $x^2$ and clean all $\overline{\bw}_i$, $\overline{\bw}_{i\log}$ and $\overline{\bpi}_i$, which is 
\begin{align}
    \bh_j^{(2)} = \MLP_{\bthetamlp^{(2)}}\paren{\wt{\bh}_i^{(2)}} = \begin{bmatrix} \overline{X}_j \\ \overline{\hat{\bpi}}_{\log}^{(1)}\\\overline{\hat{\mu}}_{j\%K}^{(1)} \\ \hat{c}_{j\%K}^{(1)}  \\ \bzero_{K_0} \\ \bzero_{K_0} \\\bzero_{K_0} \\ \bzero_{D-(2d_0+5K_0+2)} \\ 1 \\ \be_{j \% K}  \end{bmatrix}~, j \in [N], \nonumber
    % \label{refined mlp_M}
\end{align}
where $\overline{\hat{\bpi}}_{\log}^{(1)} = \widehat{\log} \overline{\hat{\bpi}}^{(1)}$, $\hat{c}_{j\%K}^{(1)} = \widehat{\|\hat{\mu}_{j\%K}^{(1)}\|_2^2}$. 

Similarly, for any $\ell \%2 = 0$, $\ell \in \mathbb{N}_+$ we have
\begin{align*}
     \bh_j^{(\ell)} = \MLP_{\bthetamlp^{(\ell\%2)}}\paren{\brac{\Attn_{\bAtt^{(\ell\%2)}}(\bH^{(\ell - 1)})}_{:,j}} = \begin{bmatrix} \overline{X}_j \\ \overline{\hat{\bpi}}_{\log}^{(\ell/2)}\\\overline{\hat{\mu}}_{j\%K}^{(\ell /2)} \\ \hat{c}_{j\%K}^{(\ell/2)} \\ \bzero_{K_0} \\ \bzero_{K_0} \\\bzero_{K_0} \\ \bzero_{D-(2d_0+5K_0+2)} \\ 1 \\ \be_{j \% K}  \end{bmatrix}~, j \in [N].
\end{align*}
where $\overline{\hat{\bpi}}_{\log}^{(\ell/2)} = \widehat{\log} \overline{\hat{\bpi}}^{(\ell/2)}$, $\hat{c}_{j\%K}^{(\ell/2)} = \widehat{\|\hat{\mu}_{j\%K}^{(\ell/2)}\|_2^2}$.

Thus, we can get $\hat{\bpi}^{(\ell)}$ and $\hat{\mu}_{j}^{(\ell)}$, $j \in [K]$ after $2\ell$ layers of transformer constructed above. (The last-layer MLP block retains $\bpi$ as an output parameter without cleaning it.) Our transformer construction is summarized in \cref{fig:thm_EM_formal}, which is the formal version of \cref{fig:thm_fig} in \cref{sec:proof idea}.

\begin{remark}
\label{rmk:readout}
    The output of transformer $\bH^{(2L)}$ is a large matrix containing lots of elements. To get the estimated parameters, we need to extract specific elements. In details, $\bH^{(2L)}=\brac{\bh_1^{(2L)},\cdots,\bh_N^{(2L)}}\in \R^{D\times N}$, where 
    \begin{align*}
        \bh_i^{(2L)} = \begin{bmatrix} \overline{X}_i \\ \overline{\hat{\bpi}}_{\log}^{(L)}\\\overline{\hat{\mu}}_{i\%K}^{(L)} \\ \hat{c}_{i\%K}^{(L)} \\ \bzero_{K_0} \\ \bzero_{K_0} \\\overline{\hat{\bpi}}^{(L)} \\ \bzero_{D-(2d_0+5K_0+2)} \\ 1 \\ \be_{j \% K}  \end{bmatrix}~, i \in [N].
    \end{align*}
    We use the following linear attentive pooling to get the parameters:
    \begin{align}
    \mathbf{O} = \frac{1}{N}(\bV_o \bH) \paren{ (\bK_o\bH)^\top \bQ_o } \in \R^{(d+K)\times K}, \nonumber
    \end{align}
    where $\bQ_o = \brac{\bq_{o 1},\cdots, \bq_{o K}} \in \R^{(d+K)\times K}$, $\bK_o, \bV_o \in \R^{(d+K)\times N}$ satisfying
    \begin{equation*}
        \bq_{o i} = \begin{bmatrix}
            K\be_i\\\bzero_{d}
        \end{bmatrix},~
        \bK_o \bh_j^{(2L)} = \begin{bmatrix}
            \be_{j\%K}\\\bzero_{d}
        \end{bmatrix},~
        \bV_o \bh_j^{(2L)} = \begin{bmatrix}
            \hat{\bpi}^{(L)}\\\hat{\mu}_{j\%K}
        \end{bmatrix}.
    \end{equation*}
    Thus by $N/K \in \N$, we have 
    \begin{align*}
        \mathbf{o}_i = \frac{1}{N}\sum_{j\in[N]} \bq_{o i}^{\top} \paren{\bK_o \bh_j} \bV_o \bh_j = \frac{K}{N}  \frac{N}{K} 
        \begin{bmatrix}
            \hat{\bpi}^{(L)}\\\hat{\mu}_{i}
        \end{bmatrix}
        =\begin{bmatrix}
            \hat{\bpi}^{(L)}\\\hat{\mu}_{i}
        \end{bmatrix} \in \R^{(d+K)}, i\in[K].
    \end{align*}
    Finally, we get
    \begin{align*}
        \mathbf{O}= \brac{\bq_{o 1},\cdots,\bq_{o N}} =
        \begin{bmatrix}
            \hat{\bpi}^{(L)} &\hat{\bpi}^{(L)} &\cdots &\hat{\bpi}^{(L)}\\
            \hat{\mu}_{1}& \hat{\mu}_{2} &\cdots &\hat{\mu}_{K}
        \end{bmatrix}.
    \end{align*}
\end{remark}

\begin{figure}[H]
    \centering
    \includegraphics[width=0.8\linewidth]{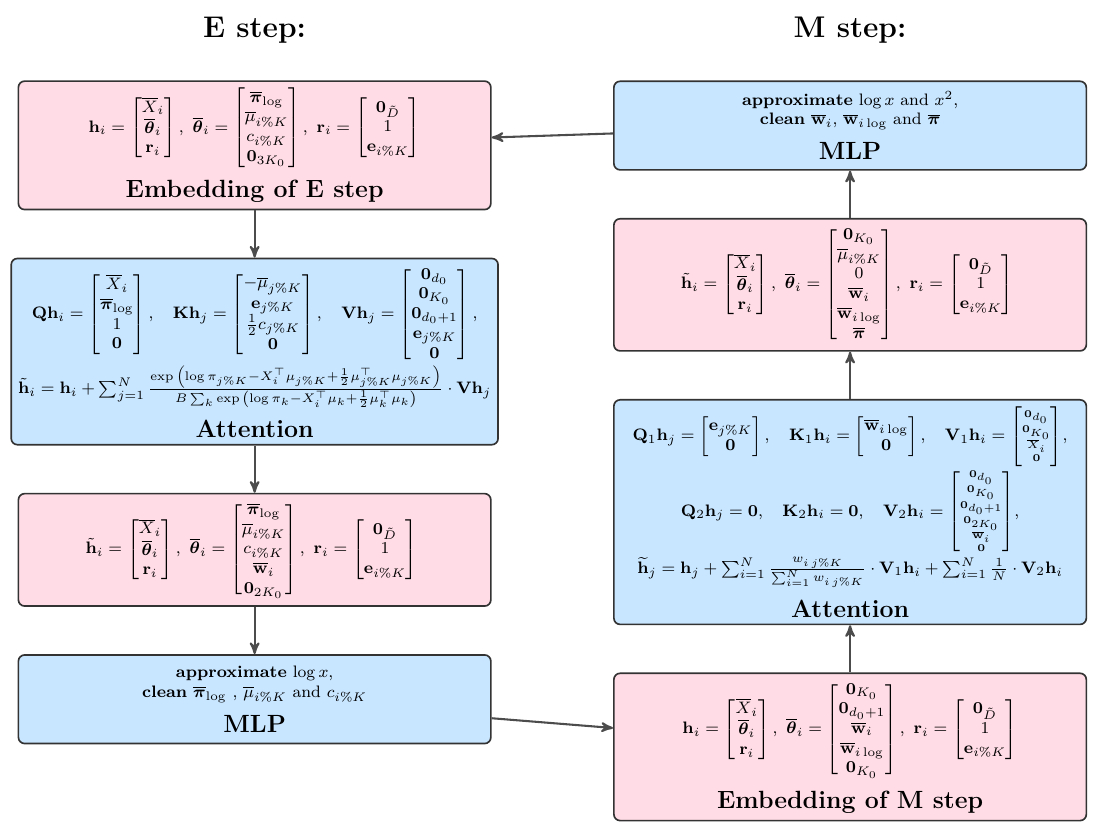}
    \caption{Transformer Construction for Approximating EM Algorithm Iterations.
    The {\it pink box} represents the state of tokens, while the {\it blue box} represents the structure of different parts of the network. The term "{\bf clean}" means setting all positions of the corresponding vector to zero.}
    \label{fig:thm_EM_formal}
\end{figure}

\subsection{Convergence results for EM algorithm}
\label{secapp:pop-EM}
\subsubsection{Convergence results for population-EM algorithm}
First, we review some notations. Recall that $\pi_{min} = \min_i \pi_i^*$, $\rho_\pi = \max_i \pi_i^* / \min_i \pi_i^*$, $R_{ij} = \|\mu_i^* - \mu_j^*\|$, $R_{min} = \min_{i\neq j} R_{ij}$ and $R_{max} = \paren{\max_{i\neq j} R_{ij}} \vee \paren{\max_{i\in[K]}\norm{\mu_i^\ast}} $.
Without the loss of generality, we assume that $R_{max} \geq 1$. 
For clarity, we restate assumption (\ref{ass:A1_copy}), which is consistent with \cite{pmlr-v125-kwon20a}.

\begin{enumerate}[label=(A\arabic*),start=1, leftmargin=*]
    \item \label{ass:A1_copy} Suppose the GMM has parameters $\{(\pi_j^\ast, \mu_j^\ast): j \in [K]\}$ such that
    \begin{equation}
        \label{eq:pop_separation_condition}
        R_{\min} \ge C  \cdot \sqrt{\log(\rho_\pi K)},
    \end{equation}
    and suppose the mean initialization $\mu_1^{(0)}, ..., \mu_K^{(0)}$ satisfies
    \begin{equation}
        \label{eq:pop_initialization_means}
        \forall i \in [K], \norm{\mu_i^{(0)} - \mu_i^\ast}\le \frac{R_{\min}}{16}.
    \end{equation}
    Also, suppose the mixing weights are initialized such that
    \begin{align}
        \label{eq:pop_initialization_weights}
        \forall i \in [K], \left|\pi_i^{(0)} - \pi_i^\ast\right| \le \pi_i / 2.
    \end{align} 
\end{enumerate}

For population-EM, the algorithm can be presented as 
\begin{align*}
    \text{(E-step):} && w_i(X) &= \frac{\pi_i \exp(-\|X - \mu_i\|^2/2 )}{\sum_{j=1}^K \pi_j \exp(-\|X - \mu_j\|^2/2)}, \\
    \text{(M-step):} && \pi_i^+ &= \E[w_i], \quad \mu_i^+ = \E[w_i X] / \E[w_i]. 
    % && {\sigma_i^+}^2 &= \E[w_i \|X - \mu_i^+\|^2] / (d \E[w_i]),
\end{align*}

The following results gives linear convergenve guarantees of population-EM, which comes from \cite{pmlr-v125-kwon20a}.

\begin{theorem}[\cite{pmlr-v125-kwon20a}, Theorem 1, part i]
    \label{theorem:gmm_pop_main_1}
    % There exists a universal constant $C \ge 128$ such that the following holds. 
    Let $C \ge 64$ be a universal constant for which assumption (\ref{ass:A1_copy}) holds.
    % Suppose a mixture of $K$ Gaussians has parameters $\{(\pi_j^*, \mu_j^*): j \in [K]\}$ such that
    % \begin{equation}
    %     \label{eq:pop_separation_condition}
    %     R_{min} \ge C  \cdot \sqrt{\log(\rho_\pi K)},
    % \end{equation}
    % and suppose the mean initialization $\mu_1^{(0)}, ..., \mu_K^{(0)}$ satisfies
    % \begin{equation}
    %     \label{eq:pop_initialization_means}
    %     \forall i \in [K], \|\mu_i^{(0)} - \mu_i^*\| \le \frac{R_{ij}}{16}.
    % \end{equation}
    % Also, suppose the mixing weights are initialized such that
    % \begin{align}
    %     \label{eq:pop_initialization_weights}
    %     \forall i \in [K], |\pi_i^{(0)} - \pi_i^*| \le \pi_i^* / 2.
    % \end{align}
    Then, after one-step population-EM update, we have
    \begin{align}
        \label{eq:pop_one-step_result}
        \forall i \in [K], |\pi_i^{+} - \pi_i^*| \le \pi_i^* / 2, ~ \|\mu_i^{+} - \mu_i^*\| \le 1/2.
    \end{align}
    % converges in $T = O(\log(1/\epsilon))$ iterations to the true solution such that for all $i\in [k]$, we have $\|\mu_i^{(T)} - \mu_i^*\| \le \sigma_i^* \epsilon$, $|\pi_i^{(T)} - \pi_i^*| / \pi_i^* \le \epsilon$.
\end{theorem}
Now we define 
\begin{equation}
    \label{eq:definte_Dm}
    D_m = \max_{i\in [K]} \left(\|\mu_i - \mu_i^*\|\vee |\pi_i - \pi_i^*| / \pi_i^* \right), \nonumber
\end{equation}
and
\begin{equation}
    \label{eq:definte_Dm+}
    D_m^+ = \max_{i\in [K]} \left(\|\mu_i^+ - \mu_i^*\|\vee |\pi_i^+ - \pi_i^*| / \pi_i^* \right).\nonumber
\end{equation}
The linear convergence of population-EM is stated by the following theorem.

\begin{theorem} [\cite{pmlr-v125-kwon20a}, Theorem 1, part ii]
\label{theorem:gmm_pop_main_2}
     Let $C \ge 128$ be a large enough universal constant. Fix $0< a < 1$. Suppose the separation condition \cref{eq:pop_separation_condition} holds and suppose the initialization parameter satisfies $D_m \leq a$, then $D_m^+ \leq \beta D_m$ for some $0< \beta < 1$.
\end{theorem}
\begin{remark}
\label{rmk:gmm_pop_main_2}
    Here the contraction parameter $\beta$ is only dependent with $C$ and $a$. In other words, if we fix $a \in (0,1)$, then for any $\beta \in (0,1)$, there exists a large enough $C$ such that \cref{theorem:gmm_pop_main_2} holds. For details, see Appendix E in \cite{pmlr-v125-kwon20a}.
\end{remark}

Combing \cref{theorem:gmm_pop_main_1} and \cref{theorem:gmm_pop_main_2}, we can get the linear convergence of population-EM algorithm.

\subsubsection{Convergence results for empirical-EM algorithm}
Now we consider the empirical-EM, i.e., \cref{alg:EM1}. For convenience, the algorithm can be presented as  
\begin{align*}
    \text{(E-step):}   && w_i(X_\ell) &= w_{\ell i} =  \frac{\pi_i \exp(-\|X_\ell - \mu_i\|^2/2 )}{\sum_{j=1}^K \pi_j \exp(-\|X_\ell - \mu_j\|^2/2)} \\
    \text{(M-step):}   &&\pi_i^+ &= \frac{1}{n} \sum_{l=1}^{n} w_i(X_\ell), ~ \mu_i^+ = \frac{\sum_{l=1}^{n} w_i(X_\ell) X_\ell}{\sum_{l=1}^{n} w_i(X_\ell)} = \frac{1}{n\pi_i^+}\sum_{l=1}^{n} w_i(X_\ell) X_\ell.
\end{align*}
Similarly, we can define $D_m$ and $D_m^+$ in empirical sense.

For the linear convergence of empirical-EM, we have the following theorem.
\begin{theorem}
\label{theorem:gmm_empi_main_2}   
Fix $0< \delta, \beta < 1$ and $0<a<1$. Let $C \ge 128$ be a large enough universal constant. Suppose the separation condition \cref{eq:pop_separation_condition} holds and suppose the initialization parameter satisfies $D_m \leq a$. If $n$ is suffcient large such that
\begin{align*}
    \varepsilon_{unif} := \frac{\tilde{c}}{\left(1-a\right)\pi_{min}} \sqrt{\frac{Kd\log(\frac{\tilde{C}n}{\delta})}{n}} < a(1-\beta) 
    % a\left((1-\beta) \wedge \frac{1}{3} \right),
\end{align*}
where $\tilde{C} =72K^2(\sqrt{d} + 2R_{max} + \frac{1}{1-a})^2$ and $\tilde{c}$ is a universal constant. Then 
$$
D_m^+ \leq \beta D_m + \varepsilon_{unif} \leq a
$$ 
uniformly holds with probability at least $1-\delta$.
\end{theorem}
\begin{proof}
    First, we have
    \begin{align*}
        \frac{|\pi_i^+ - \pi_i^*|}{\pi_i^*} &= \frac{1}{\pi_i^*}\left|\frac{1}{n} \sum_{l=1}^{n} w_i(X_\ell) - \pi_i^*\right|\\
        &\leq \frac{1}{\pi_i^*}\left(\left|\frac{1}{n} \sum_{l=1}^{n} w_i(X_\ell) - \E\left[w_i(X)\right] \right| + \left|\E\left[w_i(X)\right] - \pi_i^*\right|\right)\\
        &:= (I) + (II).
    \end{align*}
    By \cref{theorem:gmm_pop_main_2}, we get
    \begin{align*}
        (II) = \frac{1}{\pi_i^*} \left|\E\left[w_i(X)\right] - \pi_i^*\right| \leq \tilde{\beta} D_m.
    \end{align*}
    And by \cref{lem::unifrom bound of pi}, we have
    \begin{align*}
        (I) =  \frac{1}{\pi_i^*} \left|\frac{1}{n} \sum_{l=1}^{n} w_i(X_\ell) - \E\left[w_i(X)\right] \right| \leq \frac{\tilde{c}_1}{\pi_{min}} \sqrt{\frac{Kd\log(\frac{\tilde{C}_1 n}{\delta_1})}{n}},
    \end{align*}
    where $\tilde{C}_1 =18K^2(\sqrt{d} + 2R_{max} + \frac{1}{1-a})$ and $\tilde{c}_1$ is a suitable universal constant. Thus, by taking $\tilde{\beta} =  \beta$. $\delta_1 = \delta/2$ and suitable $\tilde{c}$, ${|\pi_i^+ - \pi_i^*|}/{\pi_i^*}\leq \beta D_m + \varepsilon_{unif} \leq a$, $\forall i \in [K]$.
    
    For the second term, we have
    \begin{align*}
        \|\mu_i^+ - \mu_i^*\| &=\left\|\frac{1}{n\pi_i^+}\sum_{l=1}^{n} w_i(X_\ell) (X_\ell - \mu_i^{\ast}) \right\| \\
        &\leq\frac{1}{\pi_i^+}\paren{\left\|\frac{1}{n}\sum_{l=1}^{n} w_i(X_\ell) (X_\ell - \mu_i^{\ast}) - \E\brac{w_i(X)(X - \mu_i^{\ast})} \right\|+\left\|\E\brac{w_i(X)(X - \mu_i^{\ast})} \right\|}\\
        % &\overset{(i)}{\leq}\frac{1}{(1-a)\pi_i^{\ast}}\paren{\left\|\frac{1}{n}\sum_{l=1}^{n} w_i(X_\ell) (X_\ell - \mu_i^{\ast}) - \E\brac{w_i(X)(X - \mu_i^{\ast})} \right\|+\left\|\E\brac{w_i(X)(X - \mu_i^{\ast})} \right\|}\\
        &:=(III) + (IV),
    \end{align*}
     By \cref{theorem:gmm_pop_main_2} and \cref{rmk:gmm_pop_main_2} we get,
    \begin{align*}
        (IV) &= \frac{1}{\pi_i^+} \left\|\E\brac{w_i(X)(X - \mu_i^{\ast})} \right\|\\
        &\overset{(i)}{\leq}\frac{1}{(1-a)\pi_i^{\ast}} \left\|\E\brac{w_i(X)(X - \mu_i^{\ast})} \right\|\\
        &=\frac{\E\brac{w_i(X)}}{(1-a)\pi_i^{\ast}} \left\|\frac{\E\brac{w_i(X)X}}{\E\brac{w_i(X)}} - \mu_i^{\ast}\right\|\\
        &\leq \frac{1+a}{1-a} \tilde{\beta} D_m.
    \end{align*}
    where $(i)$ follows from ${|\pi_i^+ - \pi_i^*|}/{\pi_i^*} \leq a$.
    And by \cref{lem::uniform bound of mu}, we have
    \begin{align*}
        (III) &= \frac{1}{\pi_i^{+}}\paren{\left\|\frac{1}{n}\sum_{l=1}^{n} w_i(X_\ell) (X_\ell - \mu_i^{\ast}) - \E\brac{w_i(X)(X - \mu_i^{\ast})} \right\|}\\
        &\leq \frac{1}{(1-a)\pi_i^{\ast}}\paren{\left\|\frac{1}{n}\sum_{l=1}^{n} w_i(X_\ell) (X_\ell - \mu_i^{\ast}) - \E\brac{w_i(X)(X - \mu_i^{\ast})} \right\|}\\
        &\leq\frac{\tilde{c}_2}{(1-a)\pi_{min}} \sqrt{\frac{Kd\log(\frac{\tilde{C}_2 n}{\delta_2})}{n}},
    \end{align*}
    where $\tilde{C} =18K^2(\sqrt{d} + 2R_{max} + \frac{1}{1-a})^2$ and $\tilde{c}$ is a suitable universal constant.
    Thus, by taking $\tilde{\beta} = (1-a)/(1+a) \beta$, $\delta_2 = \delta/2$ and suitable $\tilde{c}$, $\|\mu_i^+ - \mu_i^*\| \leq \beta D_m + \varepsilon_{unif} \leq a$, $\forall i \in [K]$.

    In conclusion, if we take $\tilde{\beta} = (1-a)/(1+a) \beta$, $\tilde{c} = \tilde{c}_1 \vee \tilde{c}_2$, $C = C(\beta, a) \geq 128$ large enough such that \cref{theorem:gmm_pop_main_2} holds, and take $\delta_1 = \delta_2 = \delta /2$ and use union bound argument, then we get $D_m^+ \leq \beta D_m + \varepsilon_{unif} \leq a$.

\end{proof}

We need the following technical lemma.
\begin{lemma}[\cite{10.1214/21-EJS1905}, Lemma B.2.]
    \label{uniform_convergence_lemma}
    Fix $0<\delta<1$. Let $B_{1},\ldots,B_{K}\subset\mathbb{R}^{d}$ be Euclidean balls of radii $r_{1},\ldots,r_{K}$. Define $\mathcal{B}=\otimes_{k=1}^{K}B_{k} \subset \mathbb{R}^{Kd}$ and $r=\max_{k\in[K]}r_{k}$. Let $X$ be a random vector in $\mathbb{R}^d$ and $W:\mathbb{R}^{d}\times{\cal B}\to \mathbb{R}^k$ where $k\leq d$. 
    Assume the following  hold: 
	
	1. There exists a constant $L\ge 1$ such that for any $\theta\in\mathcal{B},\varepsilon>0$, and $\theta^{\varepsilon}\in \mathcal B$
	which satisfies $\max_{i\in[K]}\|\theta_{i}-\theta_{i}^{\varepsilon}\|\le\varepsilon$,
	then 
	$
	\mathbb{E}_X\left[\sup_{\mu \in\mathcal B}\|W(X,\theta)-W(X,\theta^{\varepsilon})\|\right]\le L\varepsilon
	$.

	2. There exists a constant $R$ such that for any $\theta\in {\cal{B}}$, $\|W(X,\theta)\|_{\psi_2}\le R$.
	
	Let $X_{1},\ldots,X_{n}$ be i.i.d.
	random vectors with the same distribution as $X$. Then there exists a universal constant $\tilde{c}$ such that with probability at least $1-\delta$,
	\begin{equation}\label{etaeq}
	\sup_{\theta\in\mathcal{B}}\left\Vert \frac{1}{n}\sum_{\ell=1}^{n}W\left(X_{\ell},\theta\right)-\mathbb{E}\left[W\left(X,\theta\right)\right]\right\Vert \le R\sqrt{\tilde{c}\frac{Kd\log\left(1+\frac{12nLr}{\delta}\right)}{n}}.
	\end{equation}
\end{lemma}
\begin{remark}
    There is one difference between \cref{uniform_convergence_lemma} and LemmaB.2. in \cite{10.1214/21-EJS1905}: in \cref{uniform_convergence_lemma}, we use $1+\frac{12nLr}{\delta}$ to replace $\frac{18nLr}{\delta}$, thus we avoid the condition $r_1,\cdots,r_K \geq 1$.
\end{remark}
Hence we can get the uniform convergence of $w_i(X,\theta)$ and $w_i(X,\theta)(X-\mu_i^*)$, $i \in [K]$. Our proof is similar to \cite{10.1214/21-EJS1905}, except that we consider the variation of both $\pi$ and $\mu$. From now on, we denote $\theta_i = \{\pi_i, \mu_i\}$, $\theta = \{\theta_i\}_{i=1}^{n}$.
\begin{lemma}
\label{lem::unifrom bound of pi}
    Fix $0 < \delta <1$ and $0< a <1$. Consider the parameter region $\gD_a := \{D_m \leq a\}$. Let $X_1,\cdots, X_n \overset{\text{i.i.d.}}{\sim}  \text{GMM}(\pi^*, \mu^*)$, then with probability at least $1-\delta$,
    \begin{align}
        \sup_{\theta \in \gD_a} \left|\frac{1}{n}\sum_{\ell=1}^n w_i(X_\ell,\theta)- \E[w_i(X,\theta)] \right| \leq \tilde{c} \sqrt{\frac{Kd\log(\frac{\tilde{C}n}{\delta})}{n}},~\forall i\in [K],
    \end{align}
    where $\tilde{C} =18K^2(\sqrt{d} + 2R_{max} + \frac{1}{1-a})$ and $\tilde{c}$ is a suitable universal constant.
\end{lemma}
\begin{proof}
    The proof is similar to the proof of Lemma 5.1 in \cite{10.1214/21-EJS1905}. For simplicity, we omit it.
\end{proof}
\begin{lemma}
\label{lem::uniform bound of mu}
    Fix $0 < \delta <1$ and $0< a <1$. Consider the parameter region $\gD_a := \{D_m \leq a\}$. Let $X_1,\cdots, X_n \overset{\text{i.i.d.}}{\sim}  \text{GMM}(\pi^*, \mu^*)$ with $R_{min}$ satisfying \cref{eq:pop_separation_condition}, then with probability at least $1-\delta$,
    \begin{align}
        \sup_{\theta \in \gD_a} \left|\frac{1}{n}\sum_{\ell=1}^n w_i(X_\ell,\theta)(X_\ell-\mu_i^*)- \E[w_i(X,\theta)(X_\ell-\mu_i^*)] \right| \leq \tilde{c} \sqrt{\frac{Kd\log(\frac{\tilde{C}n}{\delta})}{n}},~\forall i\in [K],
    \end{align}
    where $\tilde{C} =36K^2(\sqrt{d} + 2R_{max} + \frac{1}{1-a})^2$ and $\tilde{c}$ is a suitable universal constant.
\end{lemma}
\begin{proof}
    The proof is similar to the proof of Lemma 5.4 in \cite{10.1214/21-EJS1905}(Notice that the condition (36) in \cite{10.1214/21-EJS1905} is trivial in our case). For simplicity, we omit it.
\end{proof}
% \begin{lemma}
%     Fix $0 < \delta <1$ and $0 < a < 3/4$. Consider the parameter region $\gD_a := \{D_m \leq a\}$. Let $X_1,\cdots, X_n \overset{\text{i.i.d.}}{\sim}  \text{GMM}(\pi^*, \mu^*)$ with $R_{min}$ satisfying \cref{eq:pop_separation_condition}. Assume sample size $n$ large enough such that 
%     \begin{align*}
%         \tilde{c} \sqrt{\frac{Kd\log(\frac{\tilde{C}n}{\delta})}{n}} \le \frac{1}{3}a\pi_{min},
%     \end{align*}
%     where $\tilde{C} =18K^2(\sqrt{d} + 2R_{max} + \frac{1}{1-a})$ and $\tilde{c}$ is a  universal constant. Then with probability at least $1-\delta$, 
%     \begin{align}
%         \inf_{\theta \in \gD_a} \frac{1}{n}\sum_{\ell=1}^n w_i(X_\ell,\theta) \ge (1-\frac{4}{3}a)\pi_i^*, ~ \forall i \in [K].
%     \end{align}
% \end{lemma}

For the first step empirical-EM, we have the following results.
\begin{theorem}
\label{theorem:gmm_empi_main_1}   
Fix $0< \delta < 1$ and $1/2<a<1$. Let $C \ge 128$ be a large enough universal constant for which assumption (\ref{ass:A1_copy}) holds. 
% Suppose the separation condition \cref{eq:pop_separation_condition} holds and suppose the initialization parameter satisfies \cref{eq:pop_initialization_means} and \cref{eq:pop_initialization_weights}. 
If $n$ is suffcient large such that
\begin{align*}
    \varepsilon_{step1} := 
    \frac{\tilde{c}}{(1-a)\pi_{min}}\sqrt{\frac{R_{max}(R_{max} \vee d)\log\paren{\frac{6K}{\delta}}}{n}} < \left(a - \frac{1}{2}\right),
\end{align*}
 where $\tilde{c}$ is a  universal constant. Then 
 \[
 D_m^+ \le \frac{1}{2} + \varepsilon_{step1} \le a
 \]
 holds with probability at least $1-\delta$.
\end{theorem}
\begin{proof}
    Notice that we only need simple concentration not uniform concentration in this theorem. We use the same definition of term $(I)$, $(II)$ as in the proof of \cref{theorem:gmm_empi_main_2}. First, by \cref{theorem:gmm_pop_main_1}, we have $(II) \leq 1/2$. Since $0\leq w_i(X) \leq 1$, by a standard concentration of bounded variables, we can get 
    \begin{align}
    (I)\leq \frac{\tilde{c}_1}{\pi_{min}} \sqrt{\frac{\log(\frac{K}{\delta_1})}{n}}, \forall i \in [K], \label{eqn: I}
    \end{align}
    where $\tilde{c}_1$ is a universal constant. Taking $\tilde{c} \geq \tilde{c}_1$ and $\delta_1 = \delta/2$, we have 
    $$
    \frac{|\pi_i^+ - \pi_i^*|}{\pi_i^*}\leq \frac{1}{2} + \frac{\tilde{c}}{\pi_{min}}\sqrt{\frac{\log\paren{\frac{2K}{\delta}}}{n}}\leq a, \forall i \in [K].
    $$
    For the second term, we have
    \begin{align*}
        \|\mu_i^+ - \mu_i^*\| &=\left\|\frac{1}{n\pi_i^+}\sum_{l=1}^{n} w_i(X_\ell) X_\ell - \mu_i^{\ast} \right\| \\
        &\leq \left\|\frac{1}{n\pi_i^+}\sum_{l=1}^{n} w_i(X_\ell) X_\ell - \frac{\E[w_i(X_\ell)X_\ell]}{\E[w_i(X_\ell)]} \right\| + \left\|\frac{\E[w_i(X_\ell)X_\ell]}{\E[w_i(X_\ell)]} - \mu_i^{\ast} \right\|\\
        &:= (V) + (VI).
    \end{align*}
    By \cref{theorem:gmm_pop_main_1}, we have $(VI) \leq 1/2$. For $(V)$, by triangle inequality,
    \begin{align}
        (V) &\leq \left\|\frac{1}{n\pi_i^+}\sum_{l=1}^{n} w_i(X_\ell) X_\ell - \frac{1}{\pi_i^+} \E[w_i(X_\ell)X_\ell] \right\| + \left\|\frac{1}{\pi_i^+}\E[w_i(X_\ell)X_\ell] - \frac{\E[w_i(X_\ell)X_\ell]}{\E[w_i(X_\ell)]} \right\| \nonumber\\
        &= \frac{1}{\pi_i^+} \left\|\frac{1}{n}\sum_{l=1}^{n} w_i(X_\ell) X_\ell -\E[w_i(X_\ell)X_\ell] \right\| + \frac{\left\|\E[w_i(X_\ell)X_\ell]\right\|}{\pi_i^+ \E[w_i(X_\ell)]} \left|{\pi_i^+} - {\E[w_i(X_\ell)]} \right|.\label{eqn: V}
    \end{align}
    Using Lemma B.1 and Lemma B.2 in \cite{10.1214/19-EJS1660}, we can get $\|w_i(X_\ell) X_\ell\|_{\psi_2} \leq \|X_\ell \|_{\psi_2} \leq \tilde{c}_3 R_{max}$, $\forall i \in [K]$. Hence by Lemma B.1 in \cite{10.1214/21-EJS1905}, with probability at least $1-\delta_2$, 
    \begin{align*}
        \left\|\frac{1}{n}\sum_{l=1}^{n} w_i(X_\ell) X_\ell - \E[w_i(X_\ell)X_\ell] \right\| \leq \tilde{c}_4 \sqrt{\frac{R_{max}d\log\paren{\frac{3K}{\delta_2}}}{n}},~ \forall i \in [K],
    \end{align*}
    where $\tilde{c}_4$ is an universal constant.
    And by \cref{theorem:gmm_pop_main_1}, we have
    \begin{align*}
        \frac{\left\|\E[w_i(X_\ell)X_\ell]\right\|}{\E[w_i(X_\ell)]} \leq R_{max} + \frac{1}{2} \leq 2 R_{max}, \forall i \in [K].
    \end{align*}
    Finally, by \cref{eqn: I},
    \begin{align*}
     \left|{\pi_i^+} - {\E[w_i(X_\ell)]} \right| \leq {\tilde{c}_1} \sqrt{\frac{\log(\frac{K}{\delta_1})}{n}}.
    \end{align*}
    Combining all terms together and taking $\delta_1 = \delta_2 = \delta/2$ we can bound \cref{eqn: V} by
    \begin{align*}
        (V)&\leq \frac{1}{\pi_i^+}\paren{\tilde{c}_4 \sqrt{\frac{R_{max}d\log\paren{\frac{3K}{\delta_2}}}{n}} + 2\tilde{c}_1 R_{max} \sqrt{\frac{\log(\frac{K}{\delta_1})}{n}}}\\
        &\leq \frac{\tilde{c}_6}{(1-a)\pi_{min}}\sqrt{\frac{R_{max}(R_{max} \vee d)\log\paren{\frac{6K}{\delta}}}{n}}.
    \end{align*}
    Taking $\tilde{c}\geq \tilde{c}_6$, we get 
    \begin{align*}
        \|\mu_i^+ - \mu_i^*\| \leq \frac{1}{2} + \frac{\tilde{c}}{(1-a)\pi_{min}}\sqrt{\frac{R_{max}(R_{max} \vee d)\log\paren{\frac{6K}{\delta}}}{n}} \leq a.
    \end{align*}
    
    % And by \cite{van2023weak}, section 2.2, we have $\left\|\E[w_i(X_\ell)X_\ell]\right\| \leq \tilde{c}_5 \|w_i(X_\ell) X_\ell\|_{\psi_2} \leq \tilde{c}_6 R_{max}$. Finally, by \cref{eqn: I},
    % \begin{align*}
    %     \left|\frac{1}{\pi_i^+} - \frac{1}{\E[w_i(X_\ell)]} \right| \leq \frac{1}{\pi_i^+ \E[w_i(X_\ell)]} \left|{\pi_i^+} - {\E[w_i(X_\ell)]} \right| \leq 
    % \end{align*}

\end{proof}

\subsubsection{Convergence results for transformer-based EM in \texorpdfstring{\cref{sec:TF_EM}}{Section D.2}}
We first state some useful approximation lemmas.
\begin{lemma}[Lemma 9 in \cite{mei2024unets}]
    \label{lem:ReLU_logarithm}
For any $A > 0$, $\delta > 0$, take $M = \lceil 2 {\log A}/\delta \rceil + 1\in \N$. Then there exists $\{ (a_j, w_j, b_j) \}_{j \in [M]}$ with 
\begin{equation}\label{eqn:ReLU_logarithm_weight_constraint}
\sup_{j} | a_j | \le 2 A,~~~\sup_{j} | w_j | \le 1,~~~ \sup_{j} | b_j | \le A,
\end{equation}
such that defining $\log_\delta : \R \to \R$ by
\[
 \log_\delta(x) = \sum_{j = 1}^M a_j \cdot \operatorname{ReLU}(w_j x + b_j ), 
\]
we have $\log_\delta$ is non-decreasing on $[1/A, A]$, and
\[
\sup_{x \in [1/A, A]} | \log(x) - \log_\delta(x) | \le \delta.
\]
\end{lemma}
\begin{remark}
    There is a small improvement $M = \lceil 2 \log A/\delta \rceil + 1$ compared to $M = \lceil 2 A/\delta \rceil + 1$ in \cite{mei2024unets}. Further more, it is easy to check that $\log_\delta(x) \leq -\log A$ for $x \in [0, 1/A]$.
\end{remark}
% \begin{lemma}[Lemma 5.1. in \cite{doi:10.1137/20M134695X}]
% \label{lem:approxSquare}
%     For any $N,L\in \N^+$, there exists a function  $\phi$ implemented by a ReLU FNN with width $3N$ and depth $L$ such that
%     \begin{equation*}
%     |\phi(x)-x^2|\le N^{-L}\quad \forall x\in [0,1].
%     \end{equation*}
% \end{lemma}
% \begin{remark}
%     In this work, we only need the case that $L = 1$. By the proof of \cref{lem:approxSquare} in \cite{doi:10.1137/20M134695X}, we can also get the range of parameters in the network. We restate it with a different scale using the similar statement of \cref{lem:ReLU_logarithm}.  
% \end{remark}

\begin{lemma}
% [restatement of \cref{lem:approxSquare}]
\label{lem:ReLU_square}
    For any $A > 0$, $\delta > 0$, take $M = \lceil 2A^2/\delta \rceil + 1\in \N$. Then there exists $\{ (a_j, w_j, b_j) \}_{j \in [M]}$ with 
    \begin{equation}\label{eqn:ReLU_square_weight_constraint}
    \sup_{j} | a_j | \le 2A,~~~\sup_{j} | w_j | \le 1,~~~ \sup_{j} | b_j | \le A,
    \end{equation}
    such that defining $\phi_\delta : \R \to \R$ by
    \[
     \phi_\delta(x) = \sum_{j = 1}^M a_j \cdot \operatorname{ReLU}(w_j x + b_j ), 
    \]
    we have $\phi_\delta$ is non-decreasing on $[-A, A]$, and
    \[
    \sup_{x \in [-A, A]} | \phi_\delta(x) - x^2 | \le \delta.
    \]
\end{lemma}

\begin{proof}
    Similar to \cref{lem:ReLU_logarithm}. Omitted.
\end{proof}
\begin{lemma}[Lemma A.1 in \cite{bai2023tfstats}]
\label{lem:gaussian concentration}
    Let $\beta\sim\mathcal{N}(0,I_d)$. Then we have
    \[
    \mathbb{P}\Big(\|\beta\|^2\geq d(1+\delta)^2\Big)\leq e^{-d\delta^2/2}.
    \]
\end{lemma}

\begin{lemma}[Lemma 18 in \cite{lin2024transformers}]
\label{lem:lip of log-softmax}
    For any $\mathbf{u},\mathbf{v}\in\mathbb{R}^d$, we have
    \[
    \left\|\log\left(\frac{e^\mathbf{u}}{\|e^\mathbf{u}\|_1}\right)-\log\left(\frac{e^\mathbf{v}}{\|e^\mathbf{v}\|_1}\right)\right\|_\infty \leq 2\left\|\mathbf{u}-\mathbf{v}\right\|_\infty . 
    \]
\end{lemma}
\begin{corollary}
\label{cor:pseudo-lip of softmax}
    For any $\mathbf{u},\mathbf{v}\in\mathbb{R}^d$, we have
    \[
        \left\|\frac{e^\mathbf{u}}{\|e^\mathbf{u}\|_1}-\frac{e^\mathbf{v}}{\|e^\mathbf{v}\|_1}\right\|_\infty \leq \exp \paren{2\left\|\mathbf{u}-\mathbf{v}\right\|_\infty} - 1
    \]
\end{corollary}
\begin{proof}
    This follows directly from \cref{lem:lip of log-softmax} and simple calculations.
\end{proof}
Now we propose the results for transformer-based EM. Similar to \cref{sec:TF_EM}, we use notations with superscript “\^{~}” to represent the output of the transformer-based EM.

\begin{theorem}
\label{theorem:gmm_tf_main_1}   
Fix $0< \delta < 1$ and $1/2<a<1$. Let $C \ge 128$ be a large enough universal constant for which assumption (\ref{ass:A1_copy}) holds. 
% Suppose the separation condition \cref{eq:pop_separation_condition} holds and suppose the initialization parameter satisfies \cref{eq:pop_initialization_means} and \cref{eq:pop_initialization_weights}. 
If $n$ is sufficient large and $\varepsilon \leq 1/100$ sufficient small such that
\begin{align*}
    % \varepsilon_{unif} := 
    &\frac{\tilde{c}_1}{(1-a)\pi_{min}}\sqrt{\frac{R_{max}(R_{max} \vee d)\log\paren{\frac{12K}{\delta}}}{n}} \\
    &\qquad + \tilde{c}_2 \left(R_{\max} + d\left(1+\sqrt{\frac{2\log(\frac{2n}{\delta})}{d}}\right)\right) n \varepsilon  <\frac{1}{2} \left(a - \frac{1}{2}\right),
\end{align*}
 where $\tilde{c}_1$, $\tilde{c}_2$ are universal constants. Then there exists a 2-layer transformer $\TF_\bTheta$ such that $\hat{D}_m^+ \le a$ holds with probability at least $1-\delta$. Moreover, $\TF_\bTheta$ falls within the class $\gF$ with parameters satisfying:
 \begin{align*}
     &D = O(d_0 + K_0), D^\prime \leq \tilde{O}\paren{K_0 R_{\max}\paren{R_{\max} + d_0} \varepsilon^{-1}},\\
     &M = O(1),\\
     &\log B_{\bTheta} \leq O\paren{K_0 R_{\max}\paren{R_{\max} + d_0}}. 
     % D^\prime \leq O\paren{K_0 R_{\max}\paren{R_{\max} + d\left(1+\sqrt{\frac{2\log(\frac{2n}{\delta})}{d}}\right) }}
 \end{align*}
\end{theorem}
\begin{proof}
    Recall that $\hat{D}_m^+ = \max_{i\in [K]} \left(\|\hat{\mu}_i^+ - \mu_i^*\|\vee |\hat{\pi}_i^+ - \pi_i^*| / \pi_i^* \right) $. Thus 
    \begin{align*}
        \hat{D}_m^+ &\leq \max_{i\in [K]} \left(\|\mu_i^+ - \mu_i^*\|\vee |\pi_i^+ - \pi_i^*| / \pi_i^* \right) + \max_{i\in [K]} \left(\|\hat{\mu}_i^+ - \mu_i^+\|\vee |\hat{\pi}_i^+ - \pi_i^+| / \pi_i^* \right) \\
        & = D_m^+ + \max_{i\in [K]} \left(\|\hat{\mu}_i^+ - \mu_i^+\|\vee |\hat{\pi}_i^+ - \pi_i^+| / \pi_i^* \right)
    \end{align*}
    We first claim that with probability at least $1-\delta/2$,
    \begin{align}
    \label{eqn:approx bound in first step}
        \max_{i\in [K]} \left(\|\hat{\mu}_i^+ - \mu_i^+\|\vee |\hat{\pi}_i^+ - \pi_i^+| / \pi_i^* \right) \leq \tilde{c}_2 \left(R_{\max} + d\left(1+\sqrt{\frac{2\log(\frac{2n}{\delta})}{d}}\right)\right) n \varepsilon.
    \end{align}  
    Then by \cref{theorem:gmm_empi_main_1}, with probability at least $1-\delta$, we have
    \begin{align*}
        \hat{D}_m^+ &\leq D_m^+ + \max_{i\in [K]} \left(\|\hat{\mu}_i^+ - \mu_i^+\|\vee |\hat{\pi}_i^+ - \pi_i^+| / \pi_i^* \right)\\
        &\leq \frac{1}{2} + \frac{\tilde{c}}{(1-a)\pi_{min}}\sqrt{\frac{R_{max}(R_{max} \vee d)\log\paren{\frac{12K}{\delta}}}{n}} + \tilde{c}_2 \left(R_{\max} + d\left(1+\sqrt{\frac{2\log(\frac{2n}{\delta})}{d}}\right)\right) n \varepsilon \\
        &\leq a.
    \end{align*}
    Now we only need to prove \cref{eqn:approx bound in first step}.
    By the construction in \cref{sec:TF_EM}, we can see that $w_{\ell i}$ in first step can be well calculated, thus $|\hat{\pi}_i^+ - \pi_i^+| = 0$ and the error comes only from the calculation of $\sets{\hat{\mu}_i^+}$. Recall that $\mu_i^+ = \frac{\sum_{\ell=1}^{n} w_{\ell i} X_\ell}{\sum_{\ell=1}^{n} w_{\ell i}}$
    and 
    \[
    \hat{\mu}_i^+ = \frac{\sum_{\ell=1}^{n} \exp\paren{\widehat{\log} ({w}_{\ell i})} X_\ell}{\sum_{l=1}^{n} \exp\paren{\widehat{\log}({w}_{\ell i})}}.
    \]
    Recall that 
    \[
        w_{\ell i} = \frac{\pi_i \exp(-\|X_\ell - \mu_i\|^2/2 )}{\sum_{j=1}^K \pi_j \exp(-\|X_\ell - \mu_j\|^2/2)} = \frac{1}{1 + \sum_{j\neq i} \frac{\pi_j}{\pi_i} \exp\paren{\paren{\mu_j - \mu_i}^\top X_l-\|\mu_j\|^2/2 + \|\mu_i\|^2/2}}.
    \]
    By the initial condition \cref{eq:pop_initialization_weights} and \cref{eq:pop_initialization_means}, we have
    \[
    \|\mu_j - \mu_i\|\leq R_{\max} + 2*\frac{1}{16}R_{\min} = O(R_{\max}),~ \|\mu_j\|^2 = O(R_{\max}^2).
    \]
    Since $X_\ell \overset{\text{i.i.d.}}{\sim}  \text{GMM}(\pi^*, \mu^*) $, using \cref{lem:gaussian concentration}, with probability at least $1-\delta/2$, we have
    \[
    \sup_{\ell\in[n]}\|X_\ell\| \leq R_{\max} + d\paren{1+\sqrt{\frac{2\log(\frac{2n}{\delta})}{d}}} = \tilde{O}(R_{\max} + d).
    \]
    Combine all things together, we get that with probability at least $1-\delta/2$,
    \[
        w_{\ell i}^{-1} \leq  \exp\paren{\tilde{O}(K_0R_{\max}(R_{\max} + d_0))}, ~\forall \ell \in [n] \text{ and } i \in [K].
    \]
    Thus taking $A = \exp\paren{\tilde{O}(K_0R_{\max}(R_{\max} + d_0))}$ and and $\delta = \varepsilon$ in \cref{lem:ReLU_logarithm}, we can get $|\log - \widehat{\log} | \big|_{[1/A, A]} \leq \varepsilon$. Then by \cref{lem:lip of log-softmax}, we have
    {\allowdisplaybreaks
    \begin{align*}
        \|\hat{\mu}_i^+ - \mu_i^+\|  &= \left\|\frac{\sum_{\ell=1}^{n} \exp\paren{\widehat{\log} ({w}_{\ell i})} X_\ell}{\sum_{l=1}^{n} \exp\paren{\widehat{\log}({w}_{\ell i})}} - \frac{\sum_{\ell=1}^{n} \exp\paren{\log w_{\ell i}} X_\ell}{\sum_{\ell=1}^{n} \exp\paren{\log w_{\ell i}}}\right\| \\
        &\leq \sum_{\ell = 1}^n  \left\|\frac{ \exp\paren{\widehat{\log} ({w}_{\ell i})} X_\ell}{\sum_{l=1}^{n} \exp\paren{\widehat{\log}({w}_{\ell i})}} - \frac{\exp\paren{\log w_{\ell i}} X_\ell}{\sum_{\ell=1}^{n} \exp\paren{\log w_{\ell i}}}\right\| \\
        &\leq \sup_{\ell \in [n]} \|X_\ell\| \paren{ \sum_{\ell = 1}^{n} \left|\frac{ \exp\paren{\widehat{\log} ({w}_{\ell i})}}{\sum_{l=1}^{n} \exp\paren{\widehat{\log}({w}_{\ell i})}} - \frac{\exp\paren{\log w_{\ell i}} }{\sum_{\ell=1}^{n} \exp\paren{\log w_{\ell i}}}\right| } \\
        &\leq n \paren{R_{\max} + d +\sqrt{{2d\log\left(\frac{2n}{\delta}\right)}}} \paren{\exp\paren{ 2 \left\|\paren{\widehat{\log}({w}_{\ell i})}_{\ell} - \paren{{\log}({w}_{\ell i})}_{\ell}\right\|_{\infty}} -1} \\
        &\leq 4 n \paren{R_{\max} + d +\sqrt{{2d\log\left(\frac{2n}{\delta}\right)}}} \varepsilon, ~\forall i \in [K].
    \end{align*}    
    }
Thus \cref{eqn:approx bound in first step} is proved. The parameter bounds can be directly computed by the construction in \cref{sec:TF_EM} and \cref{lem:ReLU_logarithm}.

    % \[
    %     w_{\ell i} = \frac{\pi_i \exp(-\|X_\ell - \mu_i\|^2/2 )}{\sum_{j=1}^K \pi_j \exp(-\|X_\ell - \mu_j\|^2/2)} = \frac{\exp(\log(\pi_i) + \mu_i^\top X_l-\|\mu_i\|^2/2)}{\sum_{j=1}^K \exp(\log(\pi_j) + \mu_j^\top X_l-\|\mu_j\|^2/2)}
    % \]
    
\end{proof}

\begin{theorem}
\label{theorem:gmm_tf_main_2}   
Fix $0< \delta, \beta < 1$ and $1/2<a<1$. Let $C \ge 128$ be a large enough universal constant. Suppose the separation condition \cref{eq:pop_separation_condition} holds and suppose the initialization parameter input to transformer satisfies ${D}_m \leq a$. If $n$ is suffcient large and $K_0 \varepsilon \leq 1/100$ sufficient small such that
\begin{align*}
    \epsilon(n,\varepsilon,\delta,a) &:= 
    \frac{\tilde{c}_1}{\left(1-a\right)\pi_{\min}} \sqrt{\frac{Kd\log(\frac{\tilde{C}n}{\delta})}{n}} + \tilde{c}_2 \paren{\frac{1}{\pi_{\min}} + n\paren{R_{\max} + d+\sqrt{2d\log\left(\frac{2n}{\delta}\right)}}}  \varepsilon  \\
    &<a(1-\beta),
\end{align*}
 where $\tilde{c}_1$, $\tilde{c}_2$ are universal constants, $\tilde{C} = 144 K^2(\sqrt{d} + 2R_{max} + \frac{1}{1-a})^2$. Then there exists a 2-layer transformer $\TF_\bTheta$ such that 
 $$
    \hat{D}_m^+ \le \beta {D}_m + \epsilon(n,\varepsilon,\delta,a) \le a
 $$ 
uniformly holds with probability at least $1-\delta$. Moreover, $\TF_\bTheta$ falls within the class $\gF$ with parameters satisfying:
 \begin{align*}
     &D = O(d_0 + K_0), D^\prime \leq \tilde{O}\paren{K_0 R_{\max}\paren{R_{\max} + d_0} \varepsilon^{-1}}, \\
     &M = O(1),\\
     &\log B_{\bTheta} \leq \tilde{O}\paren{ K_0 R_{\max}\paren{R_{\max} + d_0}}. 
     % \log B_{\bTheta} \leq O\paren{ \log\paren{\frac{1}{(1-a)\pi_{\min}}} \vee K_0 R_{\max}\paren{R_{\max} + d_0}} 
     % D^\prime \leq O\paren{K_0 R_{\max}\paren{R_{\max} + d\left(1+\sqrt{\frac{2\log(\frac{2n}{\delta})}{d}}\right) }}
 \end{align*}
\end{theorem}
\begin{proof}
    Similar to the proof of \cref{theorem:gmm_tf_main_1}, using \cref{theorem:gmm_empi_main_2}, we only need to prove that with probability at least $1-\delta/2$, 
    \begin{align}
    \label{eqn:approx bound in second step}
        \max_{i\in [K]} \left(\|\hat{\mu}_i^+ - \mu_i^+\|\vee |\hat{\pi}_i^+ - \pi_i^+| / \pi_i^* \right) \leq \tilde{c}_2 \paren{\frac{1}{\pi_{\min}} + n\paren{R_{\max} + d+\sqrt{2d\log\left(\frac{2n}{\delta}\right)}}}  \varepsilon.
    \end{align}  
    Define $u_\ell = (u_{\ell,1}, \cdots, u_{\ell,K})^{\top}$, $\hat{u}_\ell = (\hat{u}_{\ell,1}, \cdots, \hat{u}_{\ell,K})^{\top}$, where $u_{\ell,i} = \log \pi_i + \mu_i^\top X_\ell - 1/2\|\mu_i\|^2$ and $\hat{u}_{\ell,i} = \widehat{\log} \pi_i + \mu_i^\top X_\ell - 1/2\widehat{\|\mu_i\|^2}$ By the construction in \cref{sec:TF_EM} and \cref{cor:pseudo-lip of softmax}, we have
    \begin{align*}
        \|\hat{\bw}_{\ell} - \bw_{\ell} \|_\infty  = \left\| \frac{e^{\hat{u}_\ell}}{\|e^{\hat{u}_\ell}\|_1} - \frac{e^{u_\ell}}{\|e^{u_\ell}\|_1}\right\|_\infty \leq \exp\paren{2\|\hat{u}_\ell - u_\ell\|_\infty} - 1 ,~\forall \ell \in [n].
        % \leq \exp((K+2)\varepsilon) - 1 \leq 2(K+2)\varepsilon.
    \end{align*}
    Now taking $\delta = \varepsilon$, $A = \paren{(1-a)\pi_{\min}}^{-1}$ in \cref{lem:ReLU_logarithm} and $\delta = \varepsilon / K$, $A = (R_{\max} + a)^2$  in \cref{lem:ReLU_square}, we have $\|\hat{u}_\ell - u_\ell\|_\infty \leq 3\varepsilon/2$, hence
    \begin{align*}
         \|\hat{\bw}_{\ell} - \bw_{\ell} \|_\infty \leq \exp\paren{2\|\hat{u}_\ell - u_\ell\|_\infty} - 1 \leq \exp(3\varepsilon) - 1 \leq 6\varepsilon, ~\forall \ell \in [n].
    \end{align*}
    Then by the construction in \cref{sec:TF_EM}, we have 
    \begin{align}
    \label{eqn:approx of pi}
    |\hat{\pi}_i^+ - \pi_i^+| \leq 6\varepsilon, ~\forall i \in [K].
    \end{align}
    For the term $\|\hat{\mu}_i^+ - \mu_i^+\|$, we can calculate it similar to the proof of \cref{theorem:gmm_tf_main_1}. First, we recall that with probability at least $1-\delta/2$,
    \[
        w_{\ell i}^{-1} \leq  \exp\paren{\tilde{O}(K_0R_{\max}(R_{\max} + d_0))},~ \forall \ell \in [n] \text{ and } i \in [K].
    \]
    Similarly, for $\hat{w}_{\ell, i}$, we can also get(just calculate again) that with probability at least $1-\delta/2$,
     \[
        \hat{w}_{\ell i}^{-1} \leq  \exp\paren{\tilde{O}(K_0R_{\max}(R_{\max} + d_0))},~ \forall \ell \in [n] \text{ and } i \in [K].
    \]
    Then following the same argument in \cref{theorem:gmm_tf_main_1}, taking $A = \exp\paren{\tilde{O}(K_0R_{\max}(R_{\max} + d_0))}$ and and $\delta = \varepsilon$ in \cref{lem:ReLU_logarithm}, we have also
    \begin{align}
    \label{eqn:approx of mu}
     \|\hat{\mu}_i^+ - \mu_i^+\| \leq 4 n \paren{R_{\max} + d +\sqrt{{2d\log\left(\frac{2n}{\delta}\right)}}} \varepsilon, ~\forall i \in [K].   
    \end{align}
    Combining \cref{eqn:approx of pi} and \cref{eqn:approx of mu}, \cref{eqn:approx bound in second step} is proved. The parameter bounds can be directly computed by the construction in \cref{sec:TF_EM}, \cref{lem:ReLU_logarithm}, \cref{lem:ReLU_square} and the parameter $A$, $\delta$ taken in the proof.
    
\end{proof}

\subsection{Proof of \texorpdfstring{\cref{thm: EM approx}}{Theorem D.1}}
First, by \cref{theorem:gmm_tf_main_1} and the first condition in \cref{thm: EM approx} , there exist a $2$-layer transformer $\TF_{\bTheta_1}$ such that
\begin{align}\label{eqn:proof-initial-1}
    D_{\bTheta_1}^{\tt TF} \leq a,
\end{align}
holds with probability at least $1-\delta/2$. Then using \cref{theorem:gmm_pop_main_2}, \cref{eqn:proof-initial-1} and the second condition in \cref{thm: EM approx}, there $2$-layer transformer $\TF_{\bTheta_2}$ such that
\begin{align}
% \label{eqn:proof-2}
    D_{\bTheta_1 \cup \bTheta_2}^{\tt TF} \le \beta D_{\bTheta_1}^{\tt TF} + \epsilon(n,\varepsilon,\delta/2,a) \le a,\nonumber
\end{align}
uniformly holds with probability at least $1-\delta/2$. 
Denote as $\bTheta_2^L = \cup_{\ell\in[L]}\bTheta_2$. Thus, for any $L\in\N$, by reduction, we have
\begin{align}
% \label{eqn:proof-3}
    D_{\bTheta_1 \cup \bTheta_2^L}^{\tt TF} \le \beta^L D_{\bTheta_1}^{\tt TF} + \paren{1+\beta+\cdots+\beta^{L-1}}\epsilon(n,\varepsilon,\delta/2,a) , \nonumber
\end{align}
uniformly holds with probability at least $1-\delta/2$.
Combine all things together, we have, for any $L\in\N$,
\begin{align*}
    D_{\bTheta_1 \cup \bTheta_2^L}^{\tt TF} &\le \beta^L D_{\bTheta_1}^{\tt TF} + \paren{1+\beta+\cdots+\beta^{L-1}}\epsilon(n,\varepsilon,\delta/2,a)\\
    &\le \beta^L a + \frac{1}{1-\beta} \epsilon(n,\varepsilon,\delta/2,a)
\end{align*}
holds with probability at least $1-\delta$ (Note that the definitions of $\epsilon\paren{\cdot}$ in \cref{theorem:gmm_tf_main_2} and \cref{thm: EM approx} differ slightly). The parameter bounds can be directly computed by \cref{theorem:gmm_tf_main_1} and \cref{theorem:gmm_tf_main_2}. The theorem is proved.

\section{Formal statement of \texorpdfstring{\cref{thm:tensor approx}}{Theorem 2}  and proofs}
\label{secapp:tensor}
Following \cref{secapp:EM}, we implement $\operatorname{Readin}$ as an identity transformation and define $\operatorname{Readout}$ to extract targeted matrix elements hence they are both fixed functions.

\subsection{Formal statement of \texorpdfstring{\cref{thm:tensor approx}}{Theorem 2}}
\label{secapp:state_tensor}
In this section, we give the formal statement of \cref{thm:tensor approx}. First, we need to introduce the embeddings of the transformer. Let $\rmT$ be the matrix representation of the cubic tensor $T$, which is 
\begin{align*}
    \rmT \defeq \begin{bmatrix}
        \bt_1,\bt_2,\cdots,\bt_d
    \end{bmatrix}
    \defeq
    \begin{bmatrix}
        T_{:,1,1} & T_{:,2,1} & \cdots & T_{:,d,1}\\
        T_{:,1,2} & T_{:,2,2} & \cdots & T_{:,d,2}\\
        \vdots &\vdots & \ddots &\vdots\\
        T_{:,1,d} & T_{:,2,d} & \cdots & T_{:,d,d}
    \end{bmatrix} \in \R^{d^2\times d},
\end{align*}
where $T_{:,i,j} = \paren{T_{1,i,j},T_{2,i,j},\cdots,T_{d,i,j}} \in \R^d$, $i,j\in[d]$. For dimension adaptation, we assume $d\leq d_0$. The augment version of $\rmT$ is defined as
\begin{align}
\label{eqn:augment-T}
    \overline{\rmT} \defeq \begin{bmatrix}
        \overline{\bt}_1,\overline{\bt}_2,\cdots,\overline{\bt}_{d_0} 
    \end{bmatrix}
    \defeq
    \begin{bmatrix}
        \overline{T}_{:,1,1} & \overline{T}_{:,2,1} & \cdots & \overline{T}_{:,d_0,1}\\
        \overline{T}_{:,1,2} & \overline{T}_{:,2,2} & \cdots & \overline{T}_{:,d_0,2}\\
        \vdots &\vdots & \ddots &\vdots\\
        \overline{T}_{:,1,d_0} & \overline{T}_{:,2,d_0} & \cdots & \overline{T}_{:,d_0,d_0}
    \end{bmatrix}
    \in \R^{d_0^2\times d_0},
\end{align}
where $\overline{T}_{:,i,j} \in \R^{d_0}$. If $i \leq d$ and $j\leq d$, $\overline{T}_{:,i,j} = \brac{T_{:,i,j}^\top, \bzero_{d_0-d}^\top}^\top$; Else $\overline{T}_{:,i,j} =\bzero_{d_0}$.
We construct the following input sequence:
\begin{align}
\label{eqn:tensor-input-format}
    \bH = \begin{bmatrix}
    % 1 & 1 & \dots & 1  \\
    \overline{\bt}_1 & \overline{\bt}_2 & \dots & \overline{\bt}_{d}  \\
    \bp_1 & \bp_2 & \dots & \bp_{d} 
    \end{bmatrix} \in \R^{D\times d}, ~
    \bp_i = \begin{bmatrix}
        \overline{v}^{(0)}\\ \be_i \\ 1\\ d \\ \bzero_{\tilde{D}}
    \end{bmatrix},
\end{align}
where $\overline{\bt}_i\in \R^{d_0^2}$ is defined as \cref{eqn:augment-T}, $\overline{v}^{(0)} = \brac{v^{(0)\top},\bzero_{d_0-d}^\top}^\top \in \R^{d_0}$, $\be_i \in \R^{d_0}$ denotes the $i$-th standard unit vector in $\R^{d_0}$. We choose $D = O(d_0^2)$ and $\tilde{D} = D - d_0^2 - 2d_0-2$ to get the encoding above. 
Then we give a rigorous definition of ReLU-activated transformer following \cite{bai2023tfstats}.
\begin{definition}[ReLU-attention layer]
\label{def:attention-ReLU}
A (self-)attention layer activated by ReLU function with $M$ heads is denoted as $\Attn_{\bAtt}(\cdot)$ with parameters $\bAtt=\sets{ (\bV_m,\bQ_m,\bK_m)}_{m\in[M]}\subset \R^{D\times D}$. On any input sequence $\bH\in\R^{D\times N}$,
\begin{talign}
% \label{eqn:attention}
    \wt{\bH} = \Attn_{\bAtt}^R(\bH)\defeq \bH + \frac{1}{N}\sum_{m=1}^M (\bV_m \bH) \sigma \paren{ (\bK_m\bH)^\top (\bQ_m\bH) } \in \R^{D\times N}, \nonumber
\end{talign}
% where $\sursf: \R^N \to \R^N$ is the softmax function. 
In vector form,
\begin{talign*}
    \wt{\bh}_i = \brac{\Attn_{\bAtt}^R(\bH)}_i = \bh_i + \sum_{m=1}^M \frac{1}{N}\sum_{j=1}^{N} \sigma \paren{\paren{\bQ_m \bh_i}^\top \paren{\bK_m \bh_j}} \bV_m \bh_j.
\end{talign*}
Here $\sigma(x) = x \vee 0$ denotes the ReLU function.
\end{definition}
The MLP layer is the same as \cref{def:mlp}. The ReLU-activated transformer is defined as follows.
\begin{definition}[ReLU-activated transformer]
    An $L$-layer transformer, denoted as $\TF_{\bTheta}^R(\cdot)$, is a composition of $L$ ReLU-attention layers each followed by an MLP layer:
\begin{align*}
    \TF_{\bTheta}^R(\bH) = \MLP_{\bthetamlp^{(L)}}\paren{ \Attn_{\bAtt^{(L)}}^R\paren{\cdots \MLP_{\bthetamlp^{(1)}}\paren{ \Attn_{\bAtt^{(1)}}^R\paren{\bH}} }}.
\end{align*}
Above, the parameter $\bTheta=(\bAtt^{(1:L)},\bthetamlp^{(1:L)})$ consists of the attention layers $\bAtt^{(\ell)}=\sets{ (\bV^{(\ell)}_m,\bQ^{(\ell)}_m,\bK^{(\ell)}_m)}_{m\in[M^{(\ell)}]}\subset \R^{D\times D}$ and the MLP layers $\bthetamlp^{(\ell)}=(\bW^{(\ell)}_1,\bW^{(\ell)}_2)\in\R^{D^{(\ell)}\times D}\times \R^{D\times D^{(\ell)}}$. 
% and the parameters in the $\operatorname{Readin}$ and the $\operatorname{Readout}$ functions $\bTheta_{\tt in}$, $\bTheta_{\tt out}$.
\end{definition}
Similar to \cref{sec:thm statement}, We consider the following function class of transformer. 
\begin{align*}
    \gF \defeq \gF(L, D, D^\prime, M, B_\bTheta) = \set{\TF_{\bTheta}^R,  \nrmp{\bTheta} \leq B_\bTheta, D^{(\ell)} \leq D^\prime, M^{\ell} \leq M,  \ell \in[L]}. \nonumber
\end{align*}

Now we can give the formal statement of \cref{thm:tensor approx}.
\begin{theorem} [Formal version of \cref{thm:tensor approx}]
    \label{thm:tensor-approx-formal}
    There exists a transformer $\TF_\bTheta$ with  ReLU activation such that for any $d\leq d_0$, $T\in \R^{d\times d\times d}$ and $v^{(0)} \in \R^d$,  given the encoding \cref{eqn:tensor-input-format}, $\TF_\bTheta$ implements $L$ steps of \cref{eqn:tensor iteration} exactly. Moreover, $\TF_\bTheta$ falls within the class $\gF$ with parameters satisfying:
    \begin{align*}
        D = D^\prime = O(d_0^2), M = O(d_0),\log B_{\bTheta} \leq O(1). 
    \end{align*}
\end{theorem}
\begin{remark}
    In fact, \cref{thm:tensor-approx-formal} is also hold for attention-only transformers since the MLP layer do not use in the proof. To do that, we only need to add another head in every odd attention layer to clean the terms $\sets{d \overline{v}_i}$. For details, see the proof.
\end{remark}
\begin{remark}\label{rmk:thm2}
    Readers might question why the normalization step is omitted in our theorem. The key challenge is that we have absolutely no knowledge of a lower bound for  $\norm{T\paren{I,v^{(j)},v^{(j)}}}$. Without this bound, approximating the normalization step becomes infeasible.
\end{remark}
\begin{remark}
    The use of the ReLU activation function here is primarily for technical convenience and does not alter the fundamental nature of the attention mechanism. Several studies have demonstrated that transformers with ReLU-based attention perform comparably to those using softmax attention\citep{shen2023study,bai2023tfstats,he2025learningspectralmethodstransformers}.
    % The choice of ReLU activation function here is only for technical convenience, and it do not change the intrinsic of the attention structure. Various works have shown that transformer with ReLU attention performs comparable with softmax attention\cite{shen2023study,bai2023tfstats,he2025learningspectralmethodstransformers}.
\end{remark}

\subsection{Proof of \texorpdfstring{\cref{thm:tensor-approx-formal}}{Theorem E.8}}
\begin{proof}
    For simplicity, we only proof the case that $\sigma(x) = x$ in the attention layer. For ReLU activated transformer, the result can be similarly proved by $\operatorname{ReLU}(x) - \operatorname{ReLU}(-x) = x$ and the $\sigma(x) = x$ case. Hence we omit the notation $\sigma$ in the following proof. We take $\bH^{(0)} = \bH$.
    In the first attention layer, consider the following attention structures:
    \begin{align*}
        \bQ^{(1)} \bh_i^{(0)} = \begin{bmatrix}
            \be_i\\ \bzero
        \end{bmatrix},~
        \bK^{(1)} \bh_j^{(0)} = \begin{bmatrix}
            \overline{v}^{(0)}\\ \bzero
        \end{bmatrix},~
        \bV^{(1)} \bh_j^{(0)} = \begin{bmatrix}
            \bzero_{d_0^2}\\\bzero_{d_0}\\0\\0\\d\\\bzero
        \end{bmatrix}.
    \end{align*}
    After the attention operation, we have
    \begin{align*}
        \wt{\bh}_i^{(1)} &= \brac{\Attn_{\bAtt^{(1)}}^R\paren{\bH^{0}}}_{:,i} = \bh_i^{(0)} +  \frac{1}{d}\sum_{j=1}^{d} \paren{\paren{\bQ^{(1)} \bh_i^{0}}^\top \paren{\bK^{(1)} \bh_j^{0}}} \bV^{(1)}\bh_j^{0} \\
        &= \bh_i^{(0)} + \begin{bmatrix}
            \bzero_{d_0^2}\\\bzero_{d_0}\\0\\0\\d \overline{v}^{(0)}_i\\\bzero
        \end{bmatrix} = \begin{bmatrix}
            \overline{\bt}_i\\ \overline{v}^{(0)} \\1\\d\\d \overline{v}^{(0)}_i\\\bzero
        \end{bmatrix},~ i \in [d].
    \end{align*}
    Then we use a two-layer MLP to implement identity operation, which is 
    \begin{align*}
        \bh_i^{(1)} = \MLP_{\bthetamlp^{(1)}}\paren{\wt{\bh}_i^{(1)}} = \begin{bmatrix}
         \overline{\bt}_i\\ \overline{v}^{(0)} \\1\\d\\d \overline{v}^{(0)}_i\\\bzero
        \end{bmatrix},~ i \in [d].
    \end{align*}
    Now we use an attention layer with $d_0+1$ heads to implement the power iteration step of the cubic tensor. Consider the following attention structure:
    \begin{align*}
        \bQ_m^{(2)} \bh_i^{(1)} = \begin{bmatrix}
            \overline{v}^{(0)}_m\\ \bzero
        \end{bmatrix},~
        \bK_m^{(2)} \bh_j^{(1)} = \begin{bmatrix}
            d \overline{v}^{(0)}_j\\ \bzero
        \end{bmatrix},~
        \bV_m^{(2)} \bh_j^{(1)} = \begin{bmatrix}
            \bzero_{d_0^2}\\\overline{T}_{:,j,m}\\\bzero
        \end{bmatrix},~m \in [d_0],
    \end{align*}
    and 
    \begin{align*}
        \bQ_{d_0+1}^{(2)} \bh_i^{(1)} = \begin{bmatrix}
            1\\ \bzero
        \end{bmatrix},~
        \bK_{d_0+1}^{(2)} \bh_j^{(1)} = \begin{bmatrix}
            d\\ \bzero
        \end{bmatrix},~
        \bV_{d_0+1}^{(2)} \bh_j^{(1)} = \begin{bmatrix}
            \bzero_{d_0^2}\\-\overline{v}^{(0)}\\ \bzero
        \end{bmatrix}.
    \end{align*}
    After the attention operation, we have
    \begin{align*}
        \wt{\bh}_i^{(2)} &= \brac{\Attn_{\bAtt^{(2)}}^R\paren{\bH^{(1)}}}_{:,i}\\ 
        &= \bh_i^{(1)} + \sum_{m=1}^{d_0} \frac{1}{d}\sum_{j=1}^{d} \paren{\paren{\bQ^{(2)}_m \bh_i^{(1)}}^\top \paren{\bK^{(2)}_m \bh_j^{(1)}}} \bV^{(2)}_m \bh_j^{(1)} \\
        &\qquad + \frac{1}{d}\sum_{j=1}^{d} \paren{\paren{\bQ^{(2)}_{d_0+1} \bh_i^{(1)}}^\top \paren{\bK^{(2)}_{d_0+1} \bh_j^{(1)}}} \bV^{(2)}_{d_0+1} \bh_j^{(1)}\\
        &= \bh_i^{(1)} +\sum_{m=1}^{d_0} \frac{1}{d}\sum_{j=1}^{d} \paren{d\overline{v}^{(0)}_m \overline{v}^{(0)}_j}\begin{bmatrix}
            \bzero_{d_0^2}\\\overline{T}_{:,j,m}\\\bzero
        \end{bmatrix} +  \frac{1}{d}\sum_{j=1}^{d} d \begin{bmatrix}
            \bzero_{d_0^2}\\-\overline{v}^{(0)}\\ \bzero
        \end{bmatrix}\\
        &=\bh_i^{(1)} +\begin{bmatrix}
            \bzero_{d_0^2}\\\overline{v}^{(1)}\\ \bzero
        \end{bmatrix} +\begin{bmatrix}
            \bzero_{d_0^2}\\-\overline{v}^{(0)}\\ \bzero
        \end{bmatrix}\\
        &=\begin{bmatrix}
         \overline{\bt}_i\\ \overline{v}^{(1)} \\1\\d\\d \overline{v}^{(0)}_i\\\bzero
        \end{bmatrix},~ i \in [d],
    \end{align*}
    where $\overline{v}^{(1)} = \brac{{v}^{(1)\top},\bzero_{d_0-d}^{\top}}^\top$ and ${v}^{(1)} = \sum_{j,m\in[d]}v_m^{(0)}v_j^{(0)}T_{:,j,m}$.
    
    Then we use a two-layer MLP to clean the term $d \overline{v}^{(0)}_i$, which is 
    \begin{align*}
        \bh_i^{(2)} = \MLP_{\bthetamlp^{(2)}}\paren{\wt{\bh}_i^{(2)}} = \begin{bmatrix}
         \overline{\bt}_i\\ \overline{v}^{(1)} \\1\\d\\\bzero
        \end{bmatrix},~ i \in [d].
    \end{align*}
    Similarly, for any $\ell \in \N_+$, we have
    \begin{align*}
        \bh_i^{(2\ell)} = \brac{\MLP_{\bthetamlp^{(2)}}\paren{\Attn_{\bAtt^{(2)}}\paren{\MLP_{\bthetamlp^{(1)}}\paren{\Attn_{\bAtt^{(1)}}\paren{\bH^{(2\ell - 2)}}}}}}_{:,i} =
        \begin{bmatrix}
         \overline{\bt}_i\\ \overline{v}^{(\ell)} \\1\\d\\\bzero
        \end{bmatrix},~ i \in [d].
    \end{align*}
    The parameter bounds can be directly computed by the construction above. The theorem is proved.

\end{proof}

\section{More on empirical studies}\label{sec: more_exps}
\label{secapp:experiment-app}
\subsection{More on experimental setups}\label{sec: more_exp_setups}
\paragraph{Anisotropic adjustments} We consider anisotropic Gaussian mixtures that takes the following form: A $K$-component anisotropic Gaussian mixture distribution is defined with parameters $\btheta = \bpi \cup \bmu \cup \bsigma$, where $\bpi \defeq \sets{\pi_1^\ast, \pi_2^\ast, \cdots, \pi_K^\ast}$, $\pi_k^\ast \in \R$,   $\bmu = \sets{\mu_1^\ast, \mu_2^\ast, \cdots,\mu_K^\ast},\mu_k^\ast\in \R^d$, $k \in [K]$ and $\bsigma = \sets{\sigma_1^\ast, \sigma_2^\ast, \cdots,\sigma_K^\ast},\sigma_k^\ast\in \R_+^d$, $k \in [K]$. A sample $X_i$ from the aforementioned anisotropic GMM is expressed as:
\begin{align}\label{eqn: anisotropic_gmm}
    X_i = \mu_{y_i}^\ast + \sigma_{y_i}^\ast Z_i,
\end{align}
where $\sets{y_i}_{i\in[N]}$ are iid discrete random variables with $\sP\paren{y=k} = \pi_k^\ast$ for $k\in[K]$ and $\sets{Z_i}_{i\in[N]}$ are iid standard Gaussian random vector in $\R^d$. Analogous to that in the isotropic case and overload some notations, we define an anisotropic GMM task to be $\mathcal{T} = (\rmX, \btheta, K)$.\par
To adapt the TGMM framework to be compatible to anisotropic problems, we expand the output dimension of the attentive pooling module from $(d + K) \times K$ to $(d + 2K) \times K$, with the additional $K$ rows reserved for the estimate $\widehat{\bsigma}$ of $\bsigma$, with the corresponding estimation loss function augmented with a scale estimation part:
\begin{align}
\label{eqn:loss_anisotropic}
    \widehat{L}_n\paren{\bTheta} = \frac{1}{n} \sum_{i=1}^n \ell_\mu(\widehat{\bmu}_i,\bmu_i) + \ell_\pi(\widehat{\bpi}_i,\bpi_i) + \ell_\sigma(\widehat{\bsigma}_i,\bsigma_i),
\end{align}
where the loss function $\ell_\sigma$ is chosen as the mean-square loss. During the experiments, we inherit configurations from those of isotropic counterparts, except for the calculation of the $\ell_2$-error metric, where we additionally considered contributions from the estimation error of scales.
\paragraph{Configurations related to Mamba2 architecture} We adopt a Mamba2 \cite{dao2024transformers} model comprising $12$-layers and $128$-dimensional hidden states, with the rest hyper-parameters chosen so as to approximately match the number of a $12$-layer transformer with $128$-dimensional hidden states. As the Mamba series of models are essentially recurrent neural networks (RNNs), we tested two different kinds of $\operatorname{Readout}$ design with either (i). the attentive pooling module as used in the case of transformer backbone and (ii). a more natural choice of using simply the last hidden state to decode all the estimates, as RNNs compress input information in an ordered fashion. We observe from our empirical investigations that using attentive pooling yields better performance even with a Mamba2 backbone. The other training configurations are cloned from those in TGMM experiments with transformer backbones.
\paragraph{Software and hardware infrastructures}
Our framework is built upon PyTorch \cite{pytorch} and \texttt{transformers} \cite{wolf-etal-2020-transformers} libraries, which are open-source software released under BSD-style \footnote{\url{https://github.com/pytorch/pytorch/blob/master/LICENSE}} and Apache license \footnote{\url{https://github.com/huggingface/transformers/blob/main/LICENSE}}. The code implementations will be open-sourced after the reviewing process of this paper.  All the experiments are conducted using $8$ NVIDIA A100 GPUs with $80$ GB memory each.

\subsection{A complete report regarding different evaluation metrics}\label{sec: complete_exp_metrics}
In this section, we present complete reports of empirical performance regarding the evaluation problems mentioned in section \ref{sec:experiment}. Aside from the $\ell_2$-error metric that was reported in section \ref{sec:results-and-findings}, we additionally calculated all the experimental performance under the following metrics:
\begin{description}
    \item[Clustering accuracy] We compare estimated cluster membership with the true component assignment, after adjusting for permutation invariance as mentioned in section \ref{sec: meta_training_procedure}.
    \item[Log-likelihood] We compute average log-likelihood as a standard metric in unsupervised statistical estimation. 
\end{description}
The results are reported in figure \ref{fig: baseline_full}, \ref{fig: length_generalization_full}, \ref{fig: ood_full}, \ref{fig: architecture_full} and \ref{fig: anisotropic_full}, respectively. According to the evaluations, the learned TGMM models show comparable clustering accuracy against the spectral algorithm and outperform EM algorithm when $K > 2$ across all comparisons. Regarding the log-likelihood metric, TGMM demonstrates comparable performance with the other two classical algorithms in comparatively lower dimensional cases. i.e., $d \in \sets{2, 8}$, but underperforms both baselines in larger dimensional problems. We conjecture that is might be due to the fact that EM algorithm is essentially a maximum-likelihood algorithm \cite{EM1977}, while the TGMM estimation objective \eqref{eqn:loss} is not explicitly related to likelihood-based training. 
\begin{figure}[]
    \centering
    \begin{subfigure}[b]{\textwidth}
        \includegraphics[width=\textwidth]{figures/v2/baseline_comparison_v2_err.pdf}
        \caption{$\ell_2$-error}
    \end{subfigure}
    \begin{subfigure}[b]{\textwidth}
        \includegraphics[width=\textwidth]{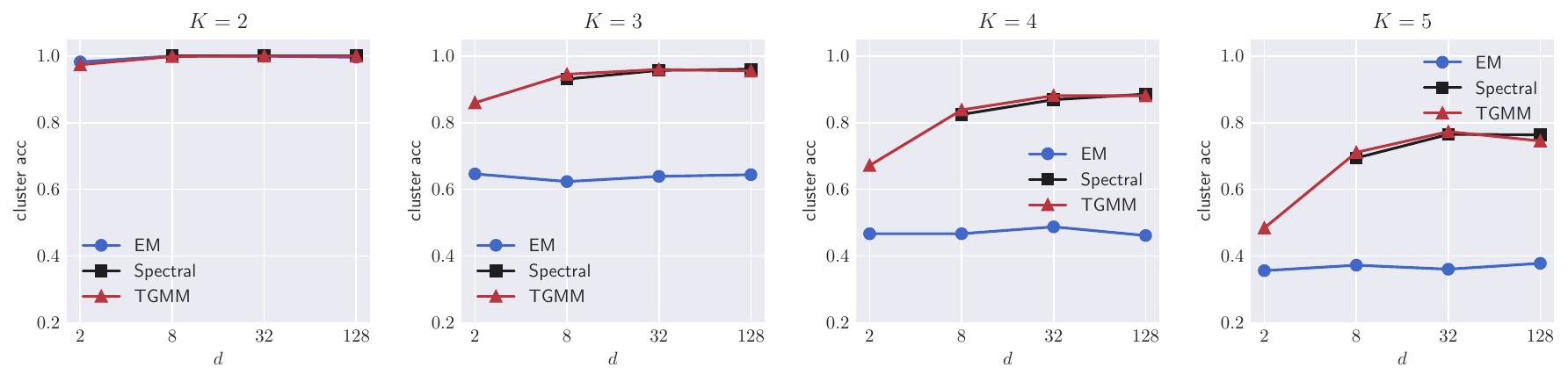}
        \caption{clustering accuracy}
    \end{subfigure}
    \begin{subfigure}[b]{\textwidth}
        \includegraphics[width=\textwidth]{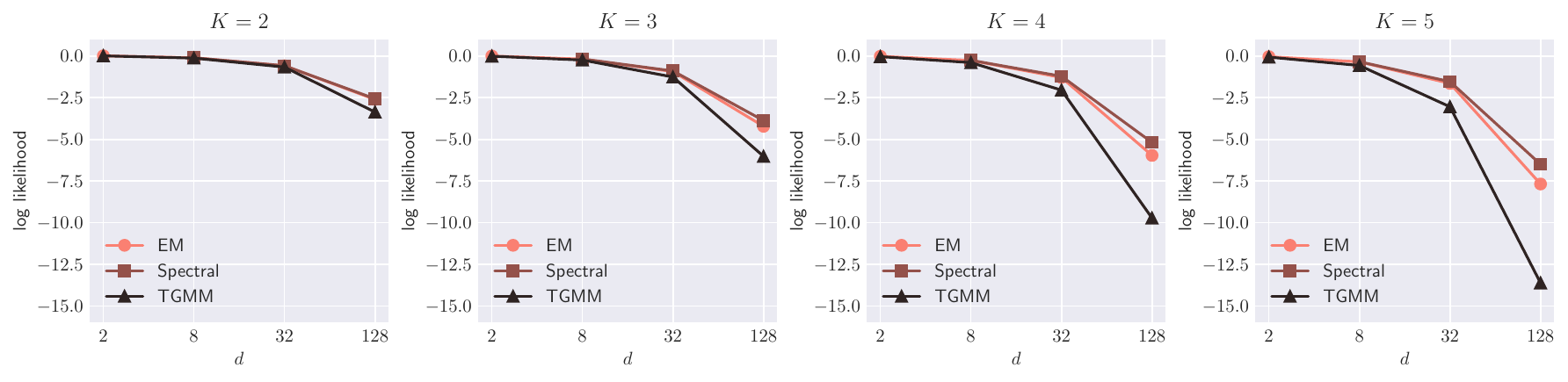}
        \caption{log-likelihood}
    \end{subfigure}
    \caption{Performance comparison between TGMM and two classical algorithms, reported in three metrics.}
    \label{fig: baseline_full}
\end{figure}

\begin{figure}[]
    \centering
    \begin{subfigure}[b]{\textwidth}
        \includegraphics[width=\textwidth]{figures/v2/length_generalization_v2_err.pdf}
        \caption{$\ell_2$-error}
    \end{subfigure}
    \begin{subfigure}[b]{\textwidth}
        \includegraphics[width=\textwidth]{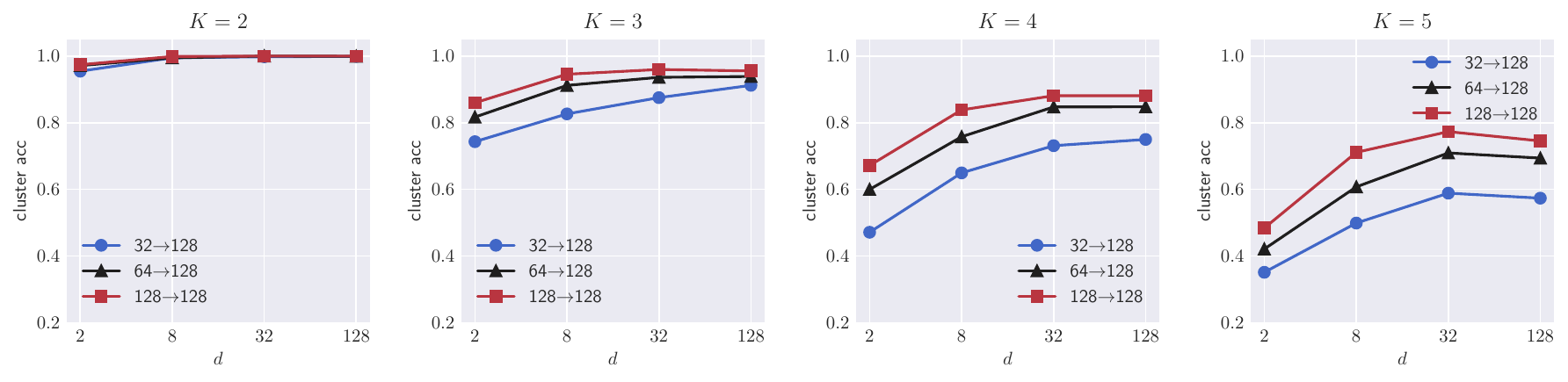}
        \caption{clustering accuracy}
    \end{subfigure}
    \begin{subfigure}[b]{\textwidth}
        \includegraphics[width=\textwidth]{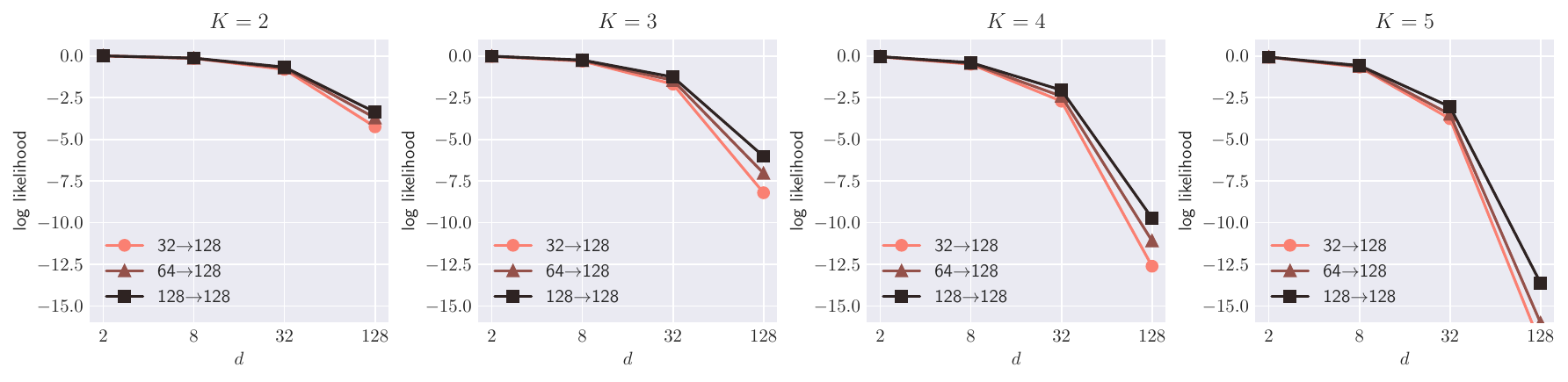}
        \caption{log-likelihood}
    \end{subfigure}
    \caption{Assessments of TGMM under test-time task distribution shifts I: A line with $N_0^\text{train} \rightarrow N^\text{test}$ draws the performance of a TGMM model trained over tasks with sample size randomly sampled in $[N_0^\text{train} / 2, N_0^\text{train}]$ and evaluated over tasks with sample size $N^\text{test}$. We can view the configuration $128\rightarrow 128$ as an in-distribution test and rest as out-of-distribution tests.}
    \label{fig: length_generalization_full}
\end{figure}

\begin{figure}[]
    \centering
    \begin{subfigure}[b]{\textwidth}
        \includegraphics[width=\textwidth]{figures/v2/ood_v2_err.pdf}
        \caption{$\ell_2$-error}
    \end{subfigure}
    \begin{subfigure}[b]{\textwidth}
        \includegraphics[width=\textwidth]{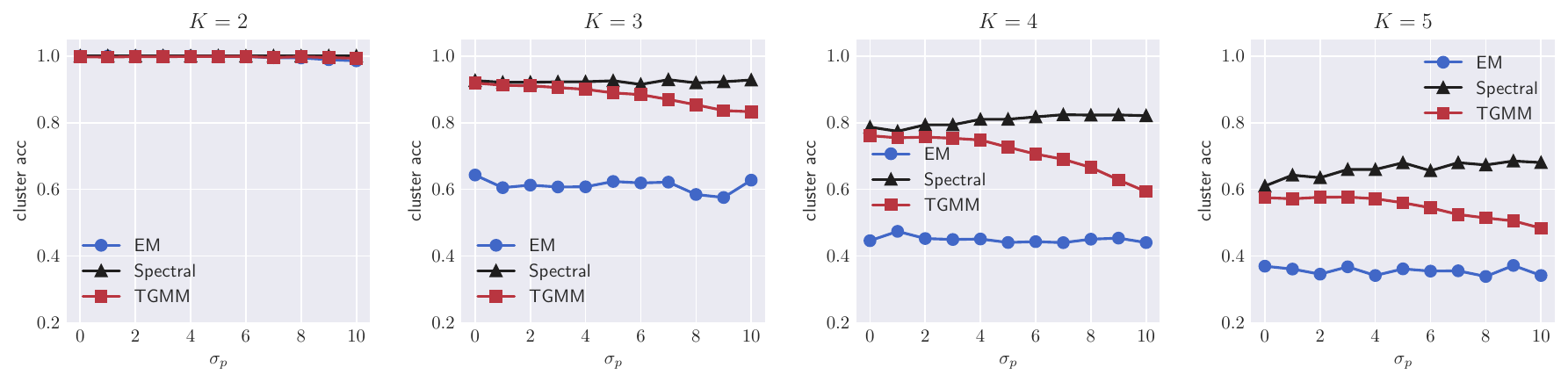}
        \caption{clustering accuracy}
    \end{subfigure}
    \begin{subfigure}[b]{\textwidth}
        \includegraphics[width=\textwidth]{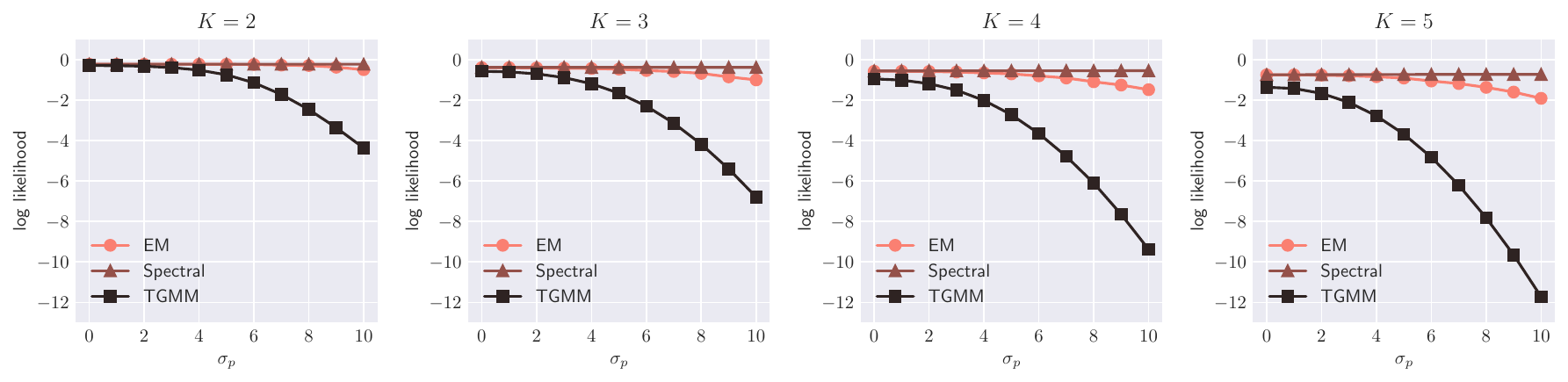}
        \caption{log-likelihood}
    \end{subfigure}
    \caption{Assessments of TGMM under test-time task distribution shifts II: $\ell_2$-error of estimation when the test-time tasks $\mathcal{T}^\text{test}$ are sampled using a mean vector sampling distribution $p_\mu^\text{test}$ different from the one used during training.}
    \label{fig: ood_full}
\end{figure}

\begin{figure}[]
    \centering
    \begin{subfigure}[b]{\textwidth}
        \includegraphics[width=\textwidth]{figures/v2/architecture_comparison_v2_err.pdf}
        \caption{$\ell_2$-error}
    \end{subfigure}
    \begin{subfigure}[b]{\textwidth}
        \includegraphics[width=\textwidth]{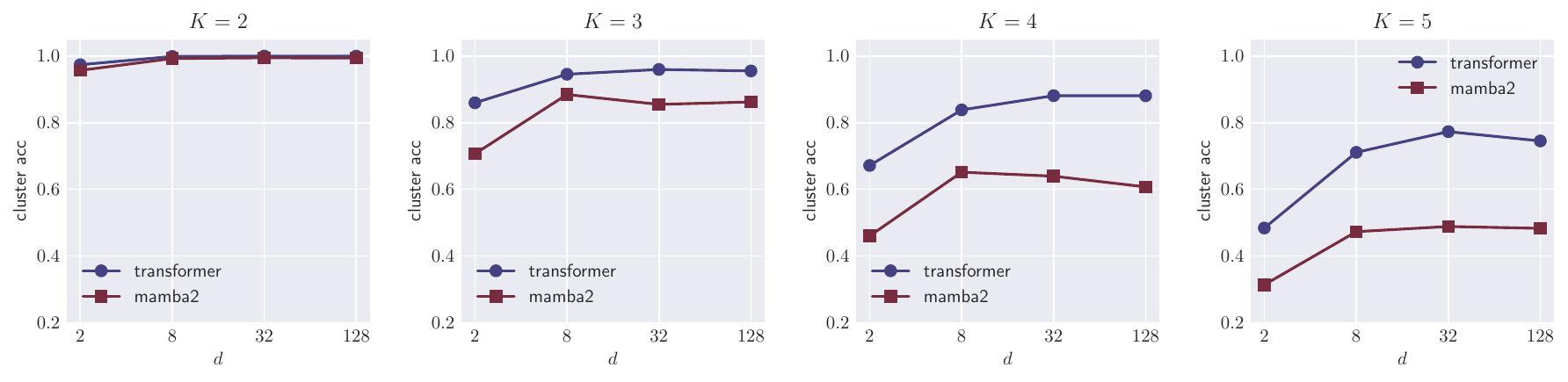}
        \caption{clustering accuracy}
    \end{subfigure}
    \begin{subfigure}[b]{\textwidth}
        \includegraphics[width=\textwidth]{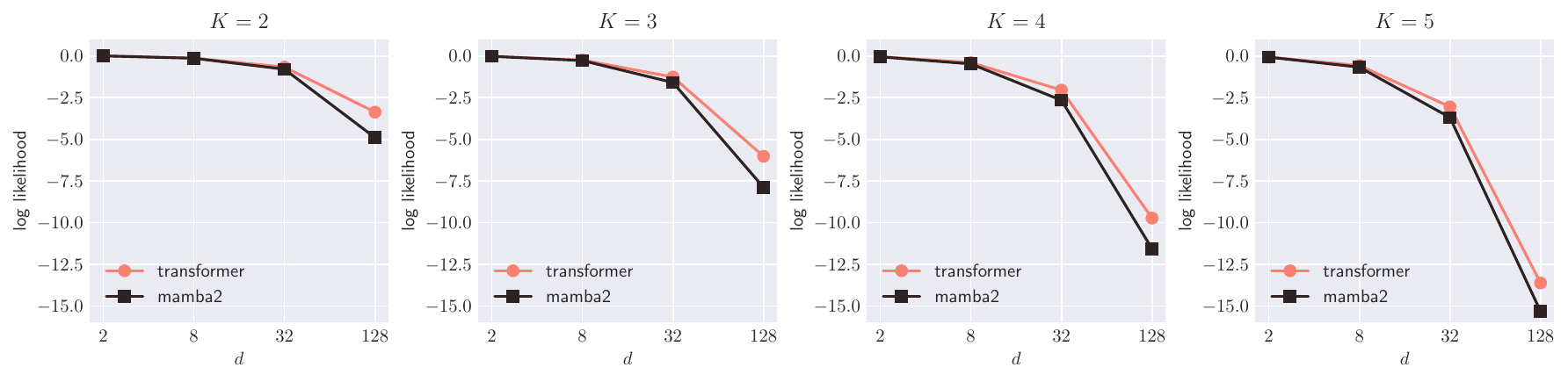}
        \caption{log-likelihood}
    \end{subfigure}
    \caption{Performance comparisons between TGMM using transformer and Mamba2 as backbone, reported in three metrics.}
    \label{fig: architecture_full}
\end{figure}

\begin{figure}[]
    \centering
    \begin{subfigure}[b]{\textwidth}
        \includegraphics[width=\textwidth]{figures/v2/baseline_comparison_anisotropic_v2_err.pdf}
        \caption{$\ell_2$-error}
    \end{subfigure}
    \begin{subfigure}[b]{\textwidth}
        \includegraphics[width=\textwidth]{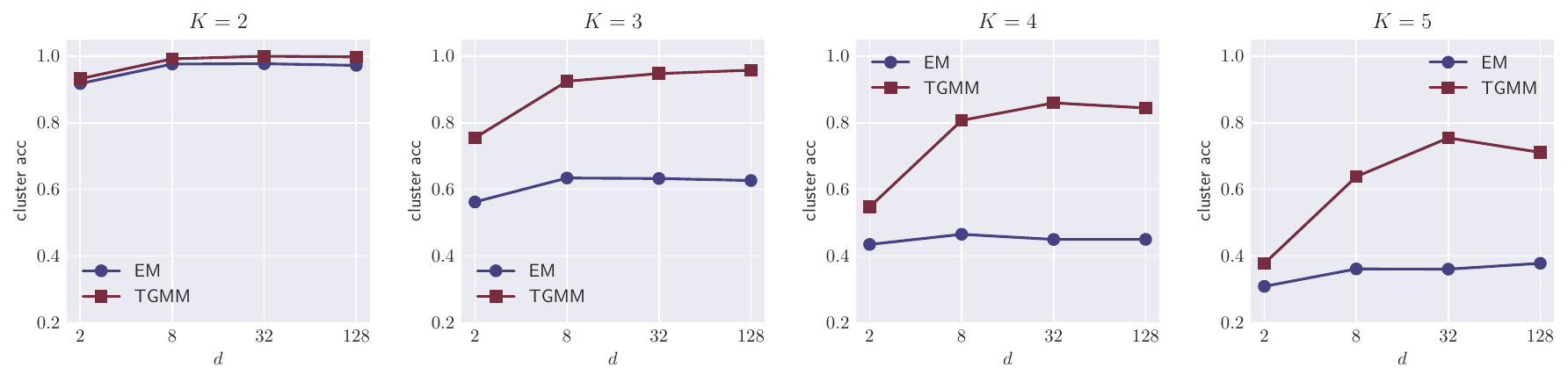}
        \caption{clustering accuracy}
    \end{subfigure}
    \begin{subfigure}[b]{\textwidth}
        \includegraphics[width=\textwidth]{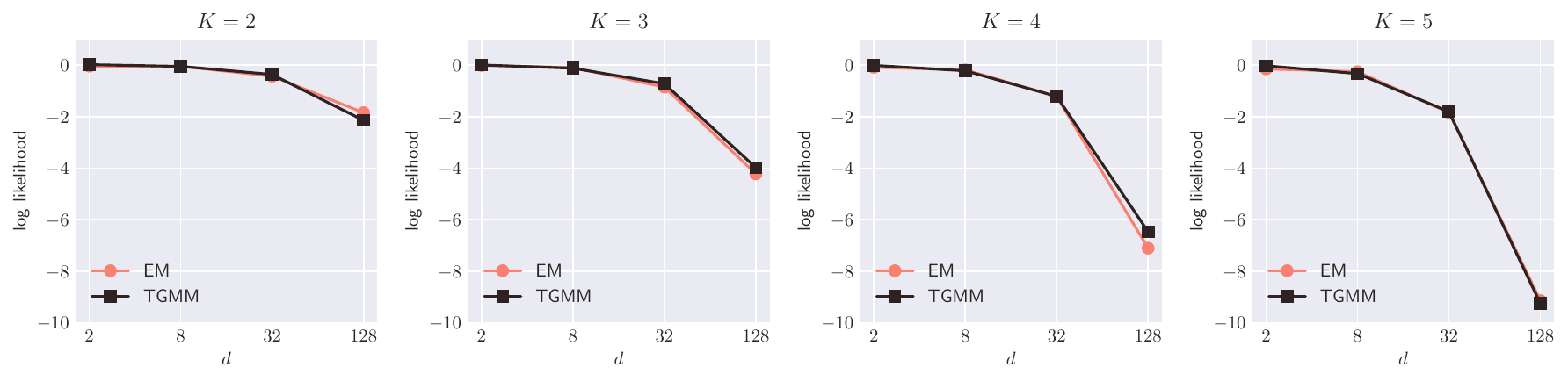}
        \caption{log-likelihood}
    \end{subfigure}
    \caption{Performance comparison between TGMM and the EM algorithm on anisotropic GMM tasks, reported in three metrics}
    \label{fig: anisotropic_full}
\end{figure}

\subsection{On the impact of inference-time sample size \texorpdfstring{$N$}{N}}\label{sec: sample_complexity}
Motivated by the classical statistical phenomenon that estimation quality tends to improve with sample size, we test whether TGMM's estimation performance increases as $N$ goes up. We run corresponding experiments by varying the sample size to be $N \in \sets{32, 64, 128}$ during both train and inference, while controlling other experimental configurations same as those in section \ref{sec: exp_setup}. The results are reported in $\ell_2$-error, clustering accuracy as log-likelihood and summarized in figure \ref{fig: sample_complexity_full}. The results exhibit a clear trend that aligns with our hypothesis, justifying the TGMM learning process as learning a statistically meaningful algorithm for solving GMMs.

\begin{figure}[]
    \centering
    \begin{subfigure}[b]{\textwidth}
        \includegraphics[width=\textwidth]{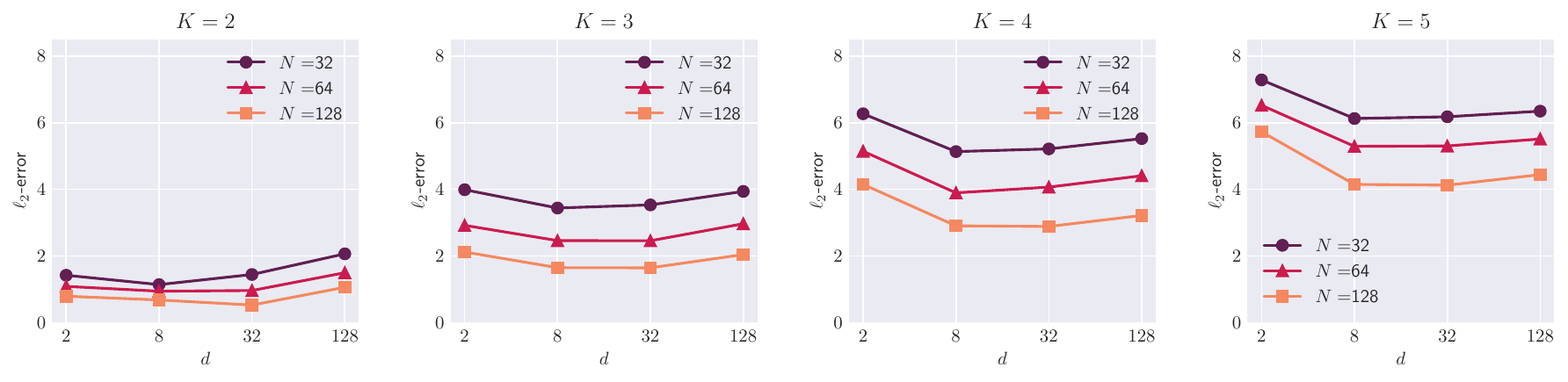}
        \caption{$\ell_2$-error}
    \end{subfigure}
    \begin{subfigure}[b]{\textwidth}
        \includegraphics[width=\textwidth]{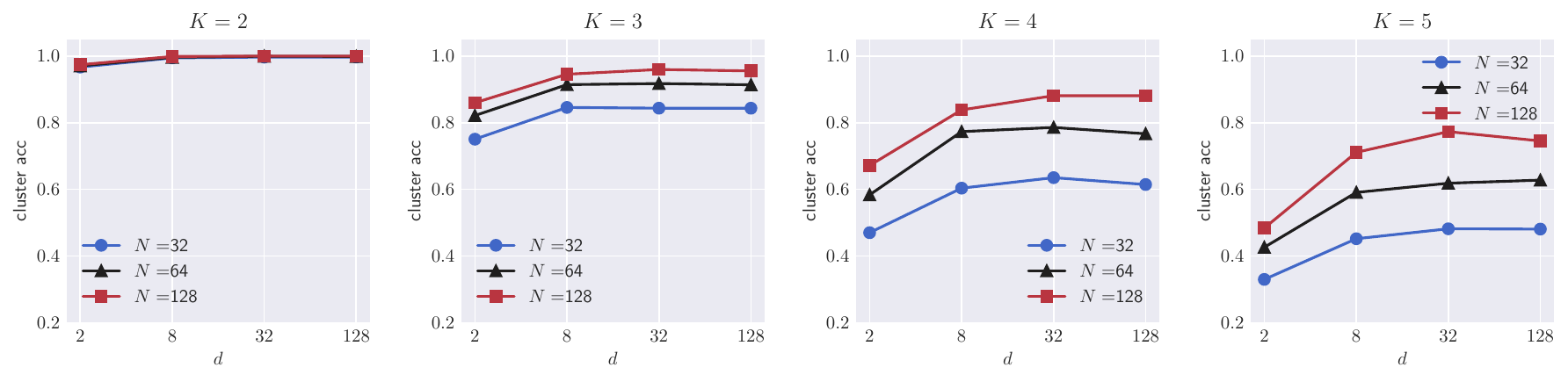}
        \caption{clustering accuracy}
    \end{subfigure}
    \begin{subfigure}[b]{\textwidth}
        \includegraphics[width=\textwidth]{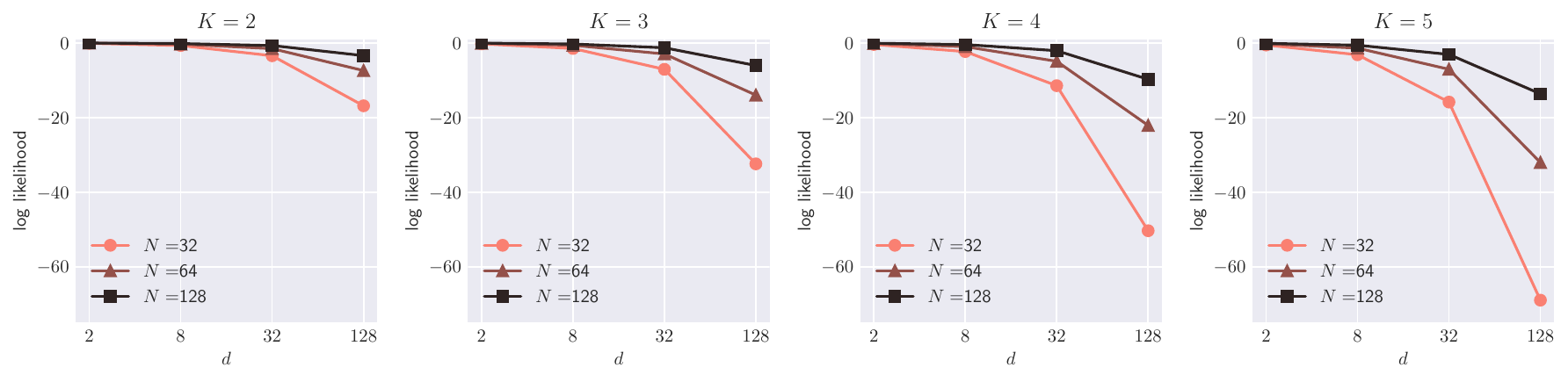}
        \caption{log-likelihood}
    \end{subfigure}
    \caption{Performance comparison between TGMM models trained under varying configurations of sample-size. For example, $N=64$ means that the model is trained over GMM tasks with (randomly chosen) sample sizes within the range $[32, 64]$ and tested on tasks with sample size $64$.}
    \label{fig: sample_complexity_full}
\end{figure}

\subsection{On the impact of backbone scale}\label{sec: impact_scale}
The scaling phenomenon is among the mostly discussed topics in modern AI, as choosing a suitable scale is often critical to the performance of transformer-based architectures like LLMs. In this section we investigate the scaling  properties of TGMM via comparing performances produced by varying sizes of backbones that differ either in per-layer width (i.e., the dimension of attention embeddings) or in the total number of layers $L$. With the rest hyper-parameters controlled to be the same as those in section \ref{sec: exp_setup}. The results are reported in three metrics and summarized in figure \ref{fig: width_full} and figure \ref{fig: depth_full}, respectively. According to these investigations, while in general a larger-sized backbone yields slightly better performance as compared to smaller ones. The performance gaps remain mild especially for tasks with relative lower complexity, i.e., $K=2$. Consequently, even a $3$-layer transformer backbone is able to achieve non-trivial learning performance for solving isotropic GMMs, a phenomenon that was also observed in a recent work \cite{he2025transformersversusemalgorithm}.
\begin{figure}[htbp]
    \centering
    \begin{subfigure}[b]{\textwidth}
        \includegraphics[width=\textwidth]{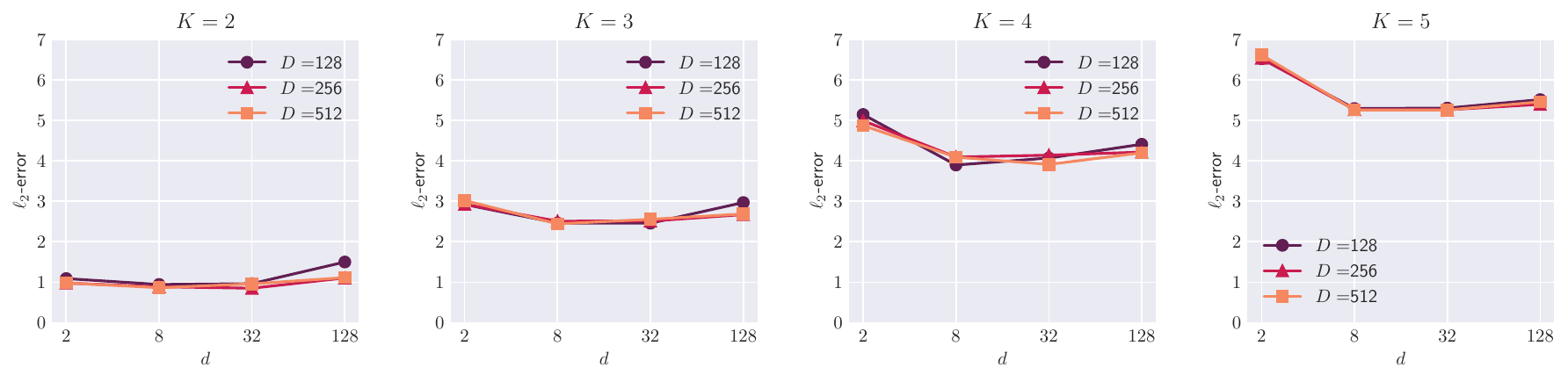}
        \caption{$\ell_2$-error}
    \end{subfigure}
    \begin{subfigure}[b]{\textwidth}
        \includegraphics[width=\textwidth]{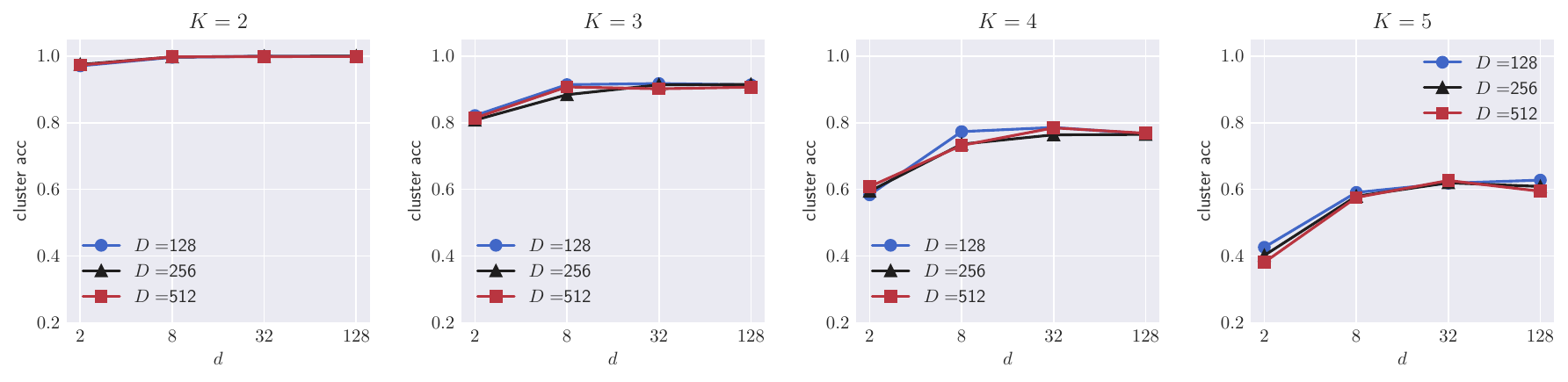}
        \caption{clustering accuracy}
    \end{subfigure}
    \begin{subfigure}[b]{\textwidth}
        \includegraphics[width=\textwidth]{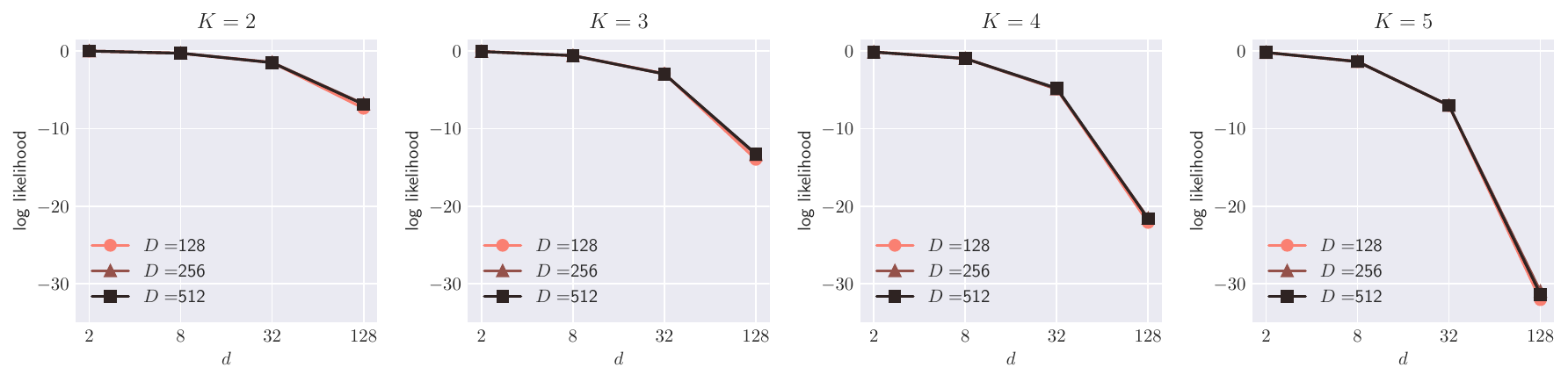}
        \caption{log-likelihood}
    \end{subfigure}
    \caption{Performance comparison between TGMM under backbones of varying scales I: We fix embedding size at $d=128$ and tested over different number of transformer layers $L \in \{3, 6, 12\}$. Results are reported in three metrics.}
    \label{fig: width_full}
\end{figure}

\begin{figure}[htbp]
    \centering
    \begin{subfigure}[b]{\textwidth}
        \includegraphics[width=\textwidth]{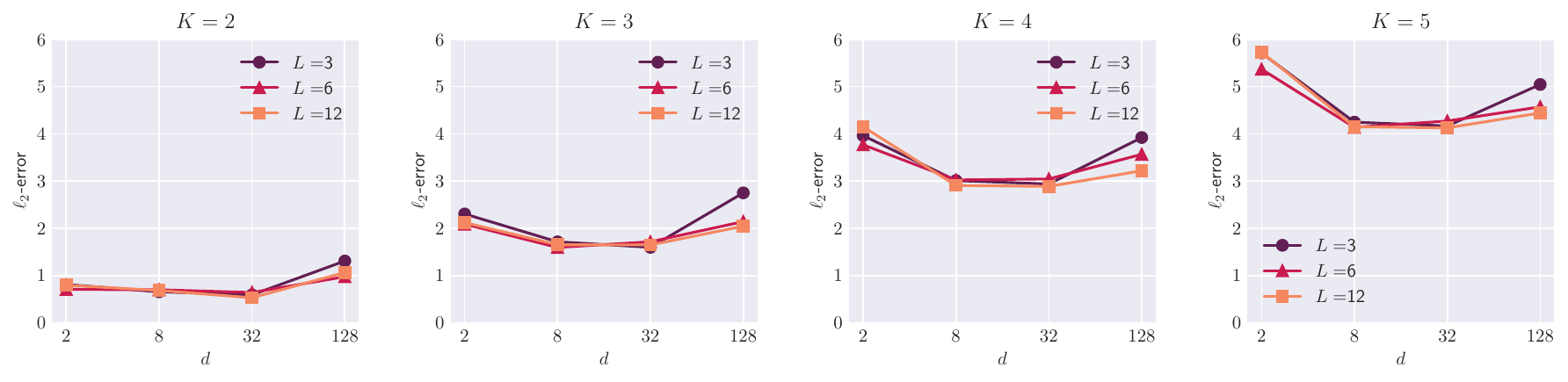}
        \caption{$\ell_2$-error}
    \end{subfigure}
    \begin{subfigure}[b]{\textwidth}
        \includegraphics[width=\textwidth]{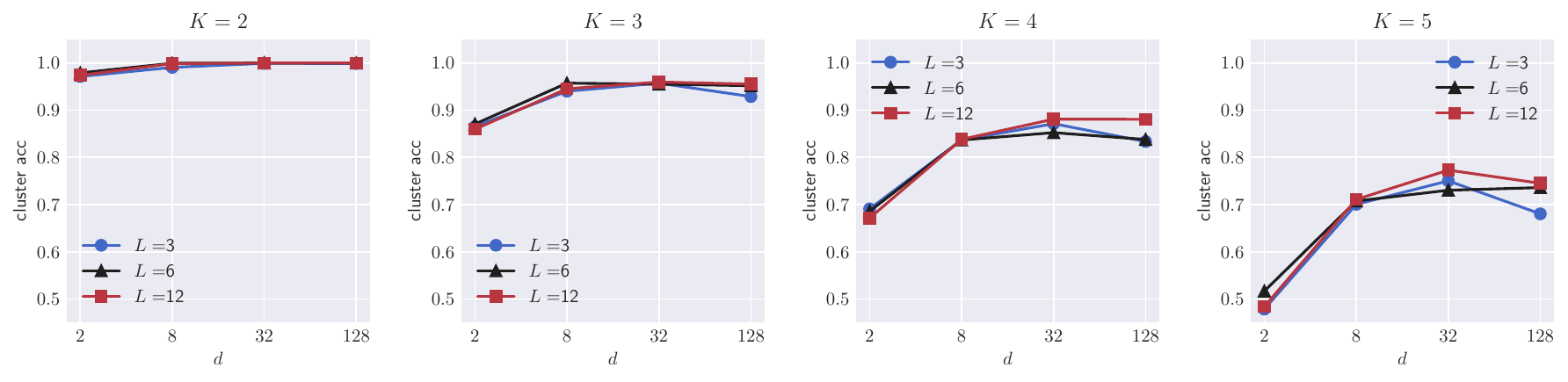}
        \caption{clustering accuracy}
    \end{subfigure}
    \begin{subfigure}[b]{\textwidth}
        \includegraphics[width=\textwidth]{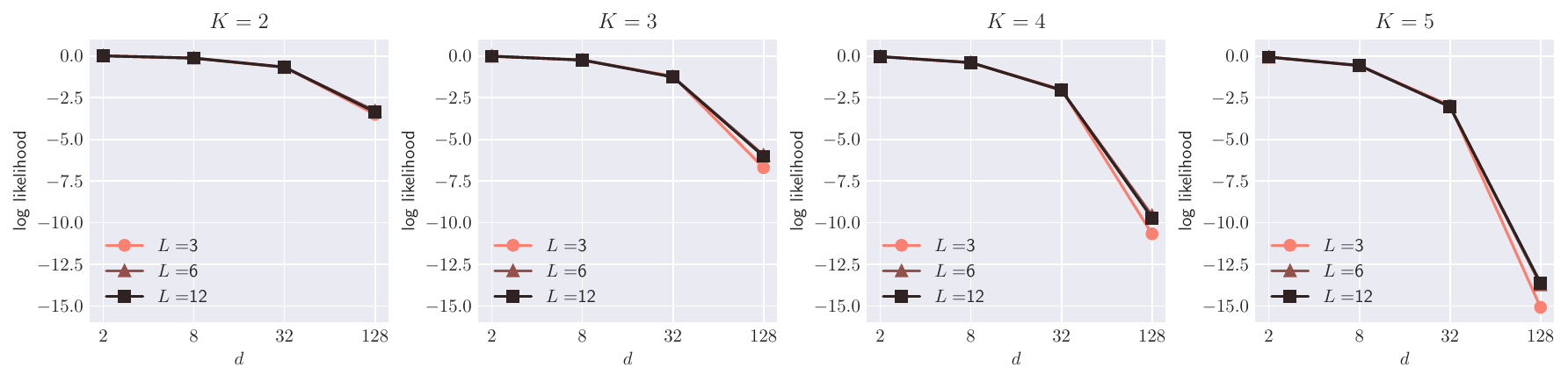}
        \caption{log-likelihood}
    \end{subfigure}
    \caption{Performance comparison between TGMM under backbones of varying scales II: We fix the number of transformer layers at $L=12$ and tested over different number of hidden states $d \in \{128, 256, 512\}$. Results are reported in three metrics.}
    \label{fig: depth_full}
\end{figure}

\end{document}